\documentclass{article}

\usepackage{arxiv}
\usepackage[numbers,sort&compress]{natbib}

\usepackage[utf8]{inputenc} 
\usepackage[T1]{fontenc}    
\usepackage{url}            
\usepackage{booktabs}       
\usepackage{amsfonts}       
\usepackage{nicefrac}       
\usepackage{microtype}      
\usepackage{lipsum}		
\usepackage{graphicx}
\usepackage{doi}
\usepackage{tikz}
\usetikzlibrary{arrows.meta,positioning,fit,backgrounds,calc}

\usepackage[toc,page]{appendix}
\usepackage{titlesec}

\usepackage{algorithm}
\usepackage{algpseudocode}
\usepackage{xcolor}
\usepackage{placeins}
\usepackage{multirow}

\algrenewcommand{\algorithmicrequire}{\textbf{Input:}}
\algrenewcommand{\algorithmicensure}{\textbf{Output:}}
\algrenewcommand{\algorithmiccomment}[1]{\hfill{\color{gray}$\triangleright$~#1}}

\definecolor{merw}{RGB}{31,119,180}    
\definecolor{kernel}{RGB}{44,160,44}   
\definecolor{dist}{RGB}{214,39,40}     
\definecolor{gram}{RGB}{148,103,189}   

\definecolor{cpurple}{RGB}{148,103,189}
\definecolor{cteal}{RGB}{23,190,207}
\definecolor{cblue}{RGB}{31,119,180}
\definecolor{ccoral}{RGB}{214,39,40}
\definecolor{camber}{RGB}{188,120,30}
\definecolor{cgray}{RGB}{200,200,200}
\definecolor{cbg}{RGB}{250,250,250}

\tikzset{
  cell/.style={circle, draw=#1, fill=#1!20, font=\small\bfseries,
               inner sep=3pt, minimum size=20pt},
  mybox/.style={draw=#1!70, fill=#1!10, rounded corners=5pt,
                inner sep=7pt, align=left, font=\small},
  panel/.style={font=\small\bfseries, gray!70},
  divider/.style={gray!30, dashed, thin},
  arr/.style={-{Stealth[length=4pt]}, #1, thick},
  dasharr/.style={-{Stealth[length=4pt]}, gray!55, dashed, thick},
}

\usepackage{amsmath, amssymb}
\usepackage{subcaption}
\usepackage{bm}
\usepackage{geometry}

\usepackage{pgfplots}
\pgfplotsset{compat=1.18}
\usepackage{siunitx}
\usepackage{amsthm}
\usepackage{hyperref}
\usepackage{cleveref}

\definecolor{merw}{RGB}{31,119,180}    
\definecolor{kernel}{RGB}{44,160,44}   
\definecolor{dist}{RGB}{214,39,40}     
\definecolor{gram}{RGB}{148,103,189}   

\theoremstyle{plain}
\newtheorem{theorem}{Theorem}[section]

\usepackage{aliascnt}

\newaliascnt{proposition}{theorem}
\newtheorem{proposition}[proposition]{Proposition}
\aliascntresetthe{proposition}

\newaliascnt{lemma}{theorem}
\newtheorem{lemma}[lemma]{Lemma}
\aliascntresetthe{lemma}

\newaliascnt{corollary}{theorem}

\aliascntresetthe{corollary}

\theoremstyle{definition}

\newaliascnt{definition}{theorem}
\newtheorem{definition}[definition]{Definition}
\aliascntresetthe{definition}

\theoremstyle{remark}

\newaliascnt{remark}{theorem}
\newtheorem{remark}[remark]{Remark}
\aliascntresetthe{remark}

\newaliascnt{assumption}{theorem}
\newtheorem{assumption}[assumption]{Assumption}
\aliascntresetthe{assumption}

\crefname{theorem}{Theorem}{Theorems}
\Crefname{theorem}{Theorem}{Theorems}
\crefname{proposition}{Proposition}{Propositions}
\Crefname{proposition}{Proposition}{Propositions}
\crefname{lemma}{Lemma}{Lemmas}
\Crefname{lemma}{Lemma}{Lemmas}
\crefname{corollary}{Corollary}{Corollaries}
\Crefname{corollary}{Corollary}{Corollaries}
\crefname{definition}{Definition}{Definitions}
\Crefname{definition}{Definition}{Definitions}
\crefname{remark}{Remark}{Remarks}
\Crefname{remark}{Remark}{Remarks}

\newcommand{\lmax}{\lambda_{\max}}
\newcommand{\dM}{d_{\mathcal{M}}}
\newcommand{\diag}{\operatorname{diag}}
\newcommand{\Psimat}{\varPsi}
\newcommand{\R}{\mathbb{R}}

\newcommand{\calM}{\mathcal{M}}

\usepackage{xspace}
\newcommand{\entropath}{EntroPath\xspace}

\title{EntroPath: Maximum Entropy Path Ensemble Embedding for Manifold Learning}

\author{
Przemys\l aw Rola \\
Department of Mathematics, Krakow University of Economics \\
\texttt{przemyslaw.rola@outlook.com} \\
}

\renewcommand{\shorttitle}{\textit{arXiv} Template}

\hypersetup{
pdftitle={A template for the arxiv style},
pdfsubject={q-bio.NC, q-bio.QM},
pdfauthor={David S.~Hippocampus, Elias D.~Striatum},
pdfkeywords={First keyword, Second keyword, More},
}

\begin{document}
\maketitle

\begin{abstract}
We introduce \emph{\entropath}, a manifold learning method that recovers
geodesic geometry from data graphs through ensembles of diffusion paths.
Many existing graph-based embeddings rely either on locally normalised
random walks or on shortest-path distances. The former can concentrate
diffusion in densely sampled regions, while the latter are sensitive to
spurious shortcut edges in the graph. \emph{\entropath} instead builds its
dissimilarities from the maximum entropy random walk (MERW), which
aggregates the full ensemble of $k$-step paths between points rather than
relying on any single trajectory. We show that the resulting free-energy
dissimilarity converges to squared geodesic distance in the short-time
limit, via Varadhan's heat-kernel formula. The diffusion depth $k$
interpolates smoothly between local neighbourhood structure and global
manifold geometry, and the symmetrised kernel admits an exact Gram
factorisation connecting \emph{\entropath} to kernel methods. We further
provide scalable extensions via landmark projection and diffusion-potential
pseudotime. Across synthetic manifolds and single-cell benchmarks,
\emph{\entropath} consistently matches or outperforms diffusion- and
shortest-path-based methods,
while remaining competitive with neighbourhood-preserving embeddings (UMAP, $t$-SNE) on local-structure metrics.
Its gains are most pronounced on manifolds with
non-uniform sampling density and well-separated branching trajectories,
where path-ensemble diffusion more faithfully preserves the underlying
geodesic geometry.
\end{abstract}


\section{Introduction}
\label{sec:intro}
\paragraph{Motivation.}
Modern high-dimensional datasets---single-cell transcriptomics, molecular
dynamics trajectories, and image manifolds---typically concentrate near a
low-dimensional Riemannian manifold embedded in ambient space. The intrinsic
geometry of such a manifold is characterised by its geodesic distances, and
preserving this geometry in a low-dimensional representation is a central goal
of manifold learning.

Existing methods approach this problem in different ways. Graph shortest-path
methods, exemplified by Isomap~\citep{Tenenbaum2000}, estimate geodesic
distances as shortest paths through a neighbourhood graph. This single-path
estimate is fragile: a single short-circuit edge between geodesically distant
regions can corrupt the resulting distances and distort the global
embedding~\citep{Balasubramanian2002}. Moreover, theoretical recovery guarantees
require restrictive assumptions, such as convex parameter domains, that are
violated by branching or holed manifolds, producing strong distortions around
the undersampled regions~\citep{Donoho2003,Bernstein2000}. Other approaches,
including t-SNE~\citep{vanderMaaten2008} and UMAP~\citep{McInnes2018}, focus
primarily on preserving local neighbourhood relationships through probabilistic
objectives rather than explicitly approximating manifold geodesics. While often
effective for visualisation, these methods do not directly encode global
geodesic structure.

Diffusion and spectral methods instead aggregate information over an
\emph{ensemble} of graph paths through a diffusion operator. By integrating many
paths rather than relying on a single shortest path, they are less sensitive to
individual graph errors and provide a natural connection to the manifold's
intrinsic geometry. Within this family, the key design choice is the
normalisation of the graph operator, since the conversion of affinities into
transition probabilities determines which geometric features are preserved.
Standard random walks (SRW) normalise row-wise by node degree, causing diffusion
to drift toward high-degree, densely sampled regions. In data graphs this has
two concrete consequences. First, sparse branches and bottlenecks are
compressed: a thin bridge between two clusters has low degree, so the SRW leaks
through it readily and underestimates the true separation (\Cref{fig:bottleneck}). Second, in
single-cell trajectory datasets, branch-point and transitional cell states are
often underrepresented relative to mature terminal populations; random-walk
embeddings can be sensitive to such sampling-density imbalances, potentially
biasing representations toward densely populated terminal regions and reducing
the accuracy of pseudotime estimation around lineage bifurcations unless
density-correction mechanisms are applied.

The maximum entropy random walk (MERW) addresses both effects by maximising the
entropy of trajectories \emph{globally} rather than locally. Distributing
probability as uniformly as possible over the ensemble of paths leaves little
probability mass on low-degree bottlenecks, preserving the separation that degree
normalisation collapses.

Because the resulting dissimilarities are computed over path ensembles rather than single shortest paths, this global correction retains the noise-robustness of diffusion geometry (\Cref{app:ablation-noise}) while targeting the geodesic distances
that shortest-path methods aim for but cannot estimate reliably on non-convex,
branching data.

\paragraph{Contributions.}
\begin{itemize}
  \item We introduce \entropath, a manifold-learning method based on the
    maximum entropy random walk (MERW) and a free-energy dissimilarity
    derived from $k$-step path ensembles.

  \item We give a free-energy (path-ensemble) formulation of the
    dissimilarity as a log-sum-exp over $k$-step path costs---a smooth
    soft-min relaxation of shortest-path geometry that aggregates the whole
    $k$-step ensemble rather than a single trajectory. The entropy term in
    this free energy (the source of the name \entropath) is what separates
    it from shortest-path distance, rewarding pairs joined by many
    comparable paths as well as by short ones.

  \item We prove a Varadhan-type short-time geodesic-recovery theorem: via
    the discrete Schr\"odinger-operator structure of the MERW Laplacian,
    the symmetrised free-energy dissimilarity recovers the squared geodesic
    distance on the underlying manifold. We further show the symmetrised
    MERW kernel is a diffusion-map Gram kernel (the Perron-vector ratios
    cancel exactly), connecting \entropath\ to kernel methods.

  \item We develop scalable extensions: a landmark Nystr\"om-type kernel
    projection for out-of-sample embedding and a MERW diffusion-potential
    pseudotime for trajectory inference.

  \item We introduce a two-stage evaluation protocol that separates
    distance-level (Level-1) from embedding-level (Level-2) performance, and
    analyse a bias in shortest-path geodesic evaluation that favours methods
    whose distances are tied to the underlying $k$-NN graph.

  \item Across synthetic manifolds and biological single-cell datasets,
    \entropath\ matches or outperforms PHATE, HeatGeo, DTNE, diffusion
    maps, and Isomap, with its clearest advantages at the distance level
    and on non-uniformly sampled manifolds and well-separated branching
    trajectories.
\end{itemize}

\paragraph{Paper organisation.}
Section~\ref{sec:merw} introduces the notation, defines the maximum entropy
random walk, and develops its Schr\"odinger-operator formulation. Section~\ref{sec:fedd} introduces the free-energy dissimilarity and its geometric regimes, while
Section~\ref{sec:kernel} establishes its Gram-matrix representation and its
relation to DTNE. Section~\ref{sec:geodesic} states the geodesic approximation
theorem, with the complete four-step proof deferred to
Appendix~\ref{app:proof}. Section~\ref{sec:algorithm} presents the algorithm,
and Section~\ref{sec:scalable} extends it to scalable landmark embeddings and
pseudotime inference. Section~\ref{sec:experiments} evaluates \entropath on
synthetic manifolds and single-cell datasets. Section~\ref{sec:related}
discusses related work, and Section~\ref{sec:conclusion} concludes with
limitations and directions for future work. The appendices contain the complete
experimental tables, ablation studies, and the analytic-versus-shortest-path
protocol validation (Appendix~\ref{app:additional-experiments}); details of the
spectral basis for diffusion-time selection
(Appendix~\ref{app:appendix-spectral-basis}); the correspondence with DTNE
(Appendix~\ref{app:dtne}); and the full proof of
\Cref{thm:geodesic} (Appendix~\ref{app:proof}).

\section{Maximum Entropy Random Walk}
\label{sec:merw}

\subsection{Setup and notation}
\label{sec:setup}
Let $G=(V,E)$ be a connected undirected graph with $n$ nodes and symmetric
nonnegative affinity matrix $A\in\R^{n\times n}$.
By the Perron--Frobenius theorem, $A$ admits a unique positive principal
eigenvector $\psi$ associated with its largest eigenvalue $\lmax$. Let $(\lambda_m,\phi^{(m)})_{m=1}^{n}$ be the eigenpairs of $A$,
\[
  A\phi^{(m)}=\lambda_m\phi^{(m)},
  \qquad
  \lambda_1\ge\lambda_2\ge\cdots\ge\lambda_n,
\]
where
\[
  \phi^{(1)}\equiv\psi>0,
  \qquad
  \|\psi\|_2=1,
  \qquad
  \lambda_1=\lmax.
\]
Define
\[
  \Psimat=\diag(\psi).
\]
Let
\[
D=\diag(d_1,\dots,d_n),
\qquad
d_i=\sum_j A_{ij},
\]
and define the combinatorial graph Laplacian
\[
L=D-A.
\]
Equivalently, using the physics convention,
\[
\Delta=A-D=-L.
\]


\subsection{Definition and basic properties}
The transition matrix that maximises the entropy rate among all Markov chains
compatible with $G$ is
\begin{equation}\label{eq:T}
  T_{ij}
  =
  \frac{A_{ij}}{\lmax}\cdot\frac{\psi_j}{\psi_i}
  \;=\;
  \frac{1}{\lmax}
  \bigl[\Psimat^{-1}A\Psimat\bigr]_{ij}.
\end{equation}
The chain is reversible with stationary distribution
\[
  \pi_i=\psi_i^2,
\]
since
\[
  \pi_iT_{ij}
  =
  \pi_jT_{ji}.
\]
This quadratic dependence on the Perron vector resembles the Born rule in
quantum mechanics, where probabilities are squared amplitudes.
Unlike the standard random walk
$T^{\mathrm{SRW}}_{ij}=A_{ij}/d_i$,
MERW maximises entropy over entire paths of a given length, making all
admissible paths of equal length between two nodes equiprobable. The resulting diffusion process is globally reweighted by the
Perron--Frobenius eigenvector $\psi$, while its transient dynamics remain
determined by the full spectrum of $A$.

The $k$-step transition matrix admits the exact spectral representation
\begin{equation}\label{eq:Tk}
  T^k
  =
  \Psimat^{-1}
  \left(\frac{A}{\lmax}\right)^k
  \Psimat,
  \qquad
  T^k_{ij}
  =
  \frac{\psi_j}{\psi_i}
  \sum_{m=1}^{n}
  \mu_m^k\,
  \phi_i^{(m)}
  \phi_j^{(m)},
\end{equation}
where
\[
\mu_m=\frac{\lambda_m}{\lmax}\in[-1,1],
\qquad
\mu_1=1.
\]

\paragraph{Implications for embedding.}
Unlike the standard random walk, which is normalised solely by local degrees,
MERW incorporates global graph structure through the Perron--Frobenius
eigenvector $\psi$. From \eqref{eq:T}, $T=\Psimat^{-1}(A/\lmax)\Psimat$, so $T$
and $A/\lmax$ are similar and share the same spectrum. Defining the \emph{discrete Hamiltonian}
\begin{equation}\label{eq:H}
  H = \lmax I - A,
\end{equation}
the MERW Laplacian $I-T$ is similar to $H/\lmax$:
\[
  I-T \;=\; \tfrac{1}{\lmax}\,\Psimat^{-1}H\,\Psimat,
  \qquad H = L + V, \quad V_i=\lmax-d_i,
\]
a discrete Schr\"odinger operator with degree-induced potential $V$, analysed in
\Cref{sec:schrodinger}. Embeddings derived from $T^k$ therefore reflect
diffusion in a geometry shaped by both local connectivity and global spectral
structure, rather than by local degree normalisation alone.

\begin{remark}[Validity for weighted Gaussian affinities]
\label{rem:gaussian-valid}
Burda et al.~\citep{Burda2009} derive MERW for the $\{0,1\}$ adjacency matrix of
an unweighted graph, but the construction~\eqref{eq:T} extends to any
nonnegative symmetric $A$, requiring only:
\begin{enumerate}
  \item \textbf{Nonnegativity}: $A_{ij}\ge0$, ensuring nonnegative transition
        probabilities;
  \item \textbf{Symmetry}: $A=A^\top$, ensuring detailed balance,
        reversibility, and a real orthogonal eigendecomposition;
  \item \textbf{Irreducibility}: the graph is connected, ensuring the
        Perron--Frobenius eigenvector $\psi$ is strictly positive and unique.
\end{enumerate}
The Gaussian affinity $A_{ij}=\exp(-\|x_i-x_j\|^2/\sigma_i\sigma_j)$ satisfies all three on a connected $k$NN graph after the standard symmetrisation $A\leftarrow\max(A,A^\top)$, so the
results below apply to it without modification. For non-symmetric $A$ a MERW can still be defined via the left and right Perron eigenvectors, but reversibility is lost and the operator is no longer similar to a self-adjoint Schr\"odinger operator; since our analysis relies on reversible diffusion and spectral geometry, we restrict attention to symmetric affinities throughout.
\end{remark}

\subsection{MERW Laplacian as a discrete Schr\"odinger operator}
\label{sec:schrodinger}
Write $L_{\mathrm{MERW}} = I - T$ for the MERW Laplacian, which by
\eqref{eq:H} is similar to $H/\lmax$. We now establish the spectral properties
of the Schr\"odinger operator $H = L + V = -\Delta + V$ (with degree-induced
potential $V_i = \lmax - d_i$) on which the geodesic analysis depends.

The potential $V$ is sign-indefinite. By Perron--Frobenius the spectral radius is bracketed by the extreme degrees,
$d_{\min} \le \lmax \le d_{\max}$, so
$V_i = \lmax - d_i$ is positive where $d_i < \lmax$, vanishes on regular graphs
(where $\lmax = d$), and is negative at the maximum-degree node(s) whenever the
graph is irregular. Rescaling by $\lmax$ does not change this: the normalised
potential $\hat V_i = 1 - d_i/\lmax$ lies in
$[\,1 - d_{\max}/\lmax,\ 1 - d_{\min}/\lmax\,]$, an interval that straddles zero
($\hat V_i \le 1$, but not in general $\ge 0$). Sign-indefiniteness of the potential is
not problematic: what matters is that
$H$ is positive semidefinite \emph{as an operator}, a spectral rather than a
pointwise property.

Three properties of $H$ are important for what follows.
First, $H$ is positive semidefinite, with eigenvalues
$\mu_\ell^H = \lmax - \lambda_\ell \ge 0$ (since $\lambda_\ell \le \lmax$ for
all $\ell$), so $e^{-tH}$ is a well-defined contraction semigroup for all
$t \ge 0$.
Second, $H\psi = 0$: the Perron--Frobenius eigenvector $\psi$ of $A$
is the zero-energy eigenfunction of $H$, playing the role of the ground state.
Third, and most importantly for manifold learning, $L_{\mathrm{MERW}}$ is
similar to the rescaled Schr\"odinger operator $H/\lmax$, via the $\Psimat$
conjugation introduced with \eqref{eq:H}. Consequently the heat semigroup
$e^{-tH}$ corresponds---up to that similarity---to multi-step MERW diffusion.
Concretely, the normalised affinity used in our dissimilarity
(\Cref{sec:dissimilarity}) is $\tilde A = A/\lmax = I - H/\lmax$, so
$\tilde A^{\,k} = (I - H/\lmax)^k \approx e^{-tH}$ with $t = k/\lmax$: the
diffusion-step count $k$ corresponds to heat-kernel time $t = k/\lmax$,
the approximation becoming exact in the short-time limit.
This connection is what allows Varadhan's heat-kernel formula
(\Cref{sec:geodesic}) to link MERW transition probabilities to
geodesic distances on the underlying data manifold.

Unlike the normalised graph Laplacian $L_{\mathrm{norm}} = I - D^{-1/2}AD^{-1/2}$,
which normalises degree away locally, $H = \lmax I - A$ normalises only by the
global spectral radius $\lmax$ and retains the potential $V_i=\lmax-d_i$. The
MERW geometry is therefore governed by the spectrum of $A$ (equivalently $H$)
directly, rather than by a degree-normalised operator.


\section{Free-Energy Dissimilarity} 
\label{sec:fedd}

\paragraph{Geodesic dissimilarity from MERW.}
\label{sec:dissimilarity}
We define the dissimilarity
\[
  D_{ij} \;=\; -\log\bigl(\tilde{A}^k\bigr)_{ij},
  \qquad \tilde{A} = A/\lmax,
\]
on the $k$-step rescaled maximum entropy random walk (MERW) kernel. This is the
form used throughout, and it is exact in the symmetric case (our primary setting;
\Cref{rem:symm-simplification}). For general $A$ the dissimilarity is defined by
the geometric-mean symmetrisation of the MERW transition probabilities
(\Cref{def:D}), in which the Perron-vector ratios cancel and which reduces
exactly to the expression above when $A$ is symmetric.
Our main theoretical result (\Cref{thm:geodesic}) shows that, under the
time scaling $t = k/\lmax$,
\[
  -4t \log\bigl(\tilde{A}^k\bigr)_{ij} \;\longrightarrow\; \dM(i,j)^2
  \qquad \text{as } t \to 0,
\]
so $D_{ij}$ recovers the squared geodesic distance on the underlying
manifold in the short-time limit, with larger $k$ giving a smooth
multi-path relaxation, via Varadhan's short-time heat-kernel
formula~\citep{Varadhan1967} (proof in \Cref{sec:geodesic}).
A statistical-physics reading of $D_{ij}$ as a path-ensemble free energy
is developed in \Cref{sec:thermo}.

\subsection{Path costs and the log-sum-exp dissimilarity}
\label{sec:thermo}
Assign each edge a \emph{cost} equal to its negative log-affinity,
\begin{equation}\label{eq:edge-energy}
  E(u,v) = -\log A_{uv},
\end{equation}
so for the Gaussian kernel with bandwidth $\varepsilon = \sigma_u\sigma_v$ the cost
$E(u,v)=\|x_u-x_v\|^2/\varepsilon$ is the squared normalised distance: short edges
are cheap, long edges costly. Path cost is additive,
$E(\gamma)=\sum_{(u,v)\in\gamma}E(u,v)=-\sum\log A_{uv}$, and each path carries
weight $e^{-E(\gamma)}=\prod_{(u,v)\in\gamma}A_{uv}$. This choice of cost is the
natural one: the logarithm turns the path product into a sum of edge costs, and
the minimum-cost path is exactly the maximum-weight (most probable) path, so cost
minimisation and probability maximisation coincide.
Summing these weights over all $k$-step paths between $i$ and $j$ is, by the
definition of matrix multiplication, the $(i,j)$ entry of the $k$-th power:
\begin{equation}\label{eq:Z-is-Ak}
  \sum_{\gamma:\,i\to j,\,|\gamma|=k} e^{-E(\gamma)} = (A^k)_{ij}.
\end{equation}
Our dissimilarity is the negative logarithm of this sum, after the normalisation
$\tilde A = A/\lmax$:
\begin{equation}\label{eq:free-energy}
  D_{ij}
  = -\log (\tilde A^k)_{ij}
  = -\log\!\!\sum_{\gamma:\,i\to j,\,|\gamma|=k}\;\prod_{(u,v)\in\gamma}\tilde A_{uv}.
\end{equation}
Since $(\tilde A^k)_{ij}$ is a partition function over all $k$-step paths between
$i$ and $j$, $D_{ij}$ is the corresponding free energy: a \emph{log-sum-exp} over
path costs, the smooth minimum that aggregates the whole $k$-step ensemble rather
than any single path. This statistical-mechanical interpretation of adjacency
walk sums follows \citet{Estrada2007}; a related path-integral weighting of MERW
underlies PAN~\citep{Ma2020} (see \Cref{sec:related}).
\paragraph{Why the normalisation $\tilde A=A/\lmax$.}
Without it, $(A^k)_{ij}$ grows like $\lmax^k$ and exceeds $1$, making
$-\log(A^k)_{ij}$ negative. Rescaling each of the $k$ edges by $1/\lmax$ gives
$(\tilde A^k)_{ij}=\lmax^{-k}(A^k)_{ij}\in(0,1]$, hence $D_{ij}\ge 0$;
equivalently, it measures path costs relative to the dominant eigenmode
$\log\lmax$.
\paragraph{What the dissimilarity captures.}
Writing the softmax weight of a path as
$p(\gamma)=e^{-E(\gamma)}/(\tilde A^k)_{ij}$, the log-sum-exp decomposes as
\[
  D_{ij}
  = \underbrace{\textstyle\sum_\gamma p(\gamma)\,E(\gamma)}_{\text{mean path cost}}
  \;-\;
  \underbrace{\Bigl(-\textstyle\sum_\gamma p(\gamma)\log p(\gamma)\Bigr)}_{\text{path entropy }\mathcal S},
\]
the standard free-energy identity $-\log\sum e^{-E}=\langle E\rangle-\mathcal S$.
Thus $D_{ij}$ is small when $i$ and $j$ are joined by short paths (low mean cost)
\emph{and/or} by many comparable paths (high entropy): closeness and connectivity
both shorten the dissimilarity. This entropy term---the source of the name
\entropath---distinguishes the dissimilarity from shortest-path distance, which
rewards only the single best path, and from the diffusion-based distances
discussed in \Cref{sec:related}, which compare whole transition-probability
profiles rather than a pairwise path ensemble.
\paragraph{Soft-min and the geodesic limit.}
With an inverse temperature $\beta$, the log-sum-exp is a smooth relaxation of the
minimum,
\[
  \lim_{\beta\to\infty}-\tfrac1\beta\log\!\sum_\gamma e^{-\beta E(\gamma)}
  = \min_\gamma E(\gamma),
\]
so $D_{ij}$ interpolates between the affinity-weighted graph geodesic (the lowest-cost path) and a full average over paths. The short-time geodesic recovery
of \Cref{thm:geodesic} is the continuum counterpart of this soft-min limit.

\subsection{Symmetrised free-energy dissimilarity}
\label{sec:symm}

\begin{definition}[Symmetrised free-energy dissimilarity]\label{def:D}
For nodes $i,j\in V$ and diffusion depth $k$,
\begin{equation}\label{eq:D}
  D_{ij} = -\log\sqrt{T^k_{ij}\,T^k_{ji}},
\end{equation}
the geometric-mean symmetrisation of the forward and backward MERW transition
probabilities.
\end{definition}

By \eqref{eq:Tk}, $T^k_{ij}=(\psi_j/\psi_i)(\tilde A^k)_{ij}$, so the
Perron-vector ratios cancel in the product,
$T^k_{ij}T^k_{ji}=(\tilde A^k)_{ij}(\tilde A^k)_{ji}$, and
\[
  D_{ij} = -\tfrac12\log\!\bigl[(\tilde A^k)_{ij}(\tilde A^k)_{ji}\bigr],
\]
the average of the forward and backward free energies. Hence $D_{ij}$ is symmetric for any $A$; it reduces to
$D_{ij}=-\log(\tilde A^k)_{ij}$ when $A$ is symmetric, and $D_{ij}\ge 0$ since
$(\tilde A^k)_{ij}\in(0,1]$.

\definecolor{cpurple}{RGB}{120,92,196}
\definecolor{cteal}{RGB}{38,170,170}
\definecolor{cblue}{RGB}{56,118,208}
\definecolor{camber}{RGB}{233,160,38}
\tikzset{
  cell/.style={circle, draw=#1!75, fill=#1!20, text=#1!45!black,
               minimum size=5.6mm, inner sep=0, font=\small\bfseries},
}
\begin{figure}[t]
\centering
\begin{tikzpicture}[scale=1.0,>=Stealth]
  \coordinate (I) at (-4.1,-0.1);
  \coordinate (J) at (-1.8,1.8);
  \draw[camber, line width=2.3pt, opacity=0.95]
    (I) .. controls (-4.0,1.8) and (-2.4,2.2) ..
    node[above left=-1pt, font=\tiny, camber] {$\gamma_1$: low $E(\gamma)$} (J);
  \draw[cblue, line width=1.4pt, opacity=0.7]
    (I) .. controls (-3.0,1.0) and (-2.5,0.8) ..
    node[below=1pt, font=\tiny, cblue] {$\gamma_2$} (J);
  \draw[cblue, line width=1.0pt, opacity=0.35]
    (I) .. controls (-3.0,-0.8) and (-2.2,0.2) ..
    node[below right=-1pt, font=\tiny, cblue, opacity=1]
    {$\gamma_3$: high $E(\gamma)$} (J);
  \node[cell=cpurple] at (I) {$i$};
  \node[cell=cteal]   at (J) {$j$};
\end{tikzpicture}
\caption{The MERW free-energy dissimilarity as a soft-min over paths. Each
  $k$-step path $\gamma$ from $i$ to $j$ carries weight $e^{-E(\gamma)}$ with
  additive edge energy $E(u,v)=-\log\tilde A_{uv}$, and the partition function
  $(\tilde A^k)_{ij}=\sum_{\gamma:i\to j}e^{-E(\gamma)}$ is a log-sum-exp
  (soft-min) over these energies. The lowest-energy, geodesic-like path
  ($\gamma_1$, thick amber) dominates the sum while high-energy detours
  ($\gamma_3$, faded) are exponentially suppressed, so
  $D_{ij}=-\log(\tilde A^k)_{ij}$ tracks the cheapest route (\Cref{def:D}).}
\label{fig:path-ensemble}
\end{figure}
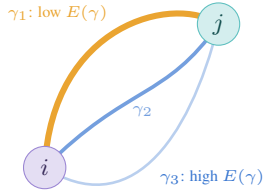

\begin{remark}[Metric axioms]
\label{rem:metric}
$D_{ij}$ is symmetric and non-negative but is not a metric in general: as a
log-kernel dissimilarity \eqref{eq:D} it need not satisfy the triangle inequality, and the diagonal $D_{ii}=-\log T^k_{ii}$ is in general nonzero —
the walk's return log-probability — failing the identity of indiscernibles;
we set it to zero before embedding by metric MDS, initialised at the
classical-MDS solution (\Cref{alg:entropath}). It is therefore a
\emph{dissimilarity}, which is all that MDS --- classical or metric ---
requires.\footnote{Classical MDS embeds the double-centred matrix $B=-\tfrac12 J D^{(2)} J$, where $D^{(2)}=(D_{ij}^2)$ is the matrix of squared dissimilarities and $J$ the centring matrix. For a non-metric dissimilarity, the
negative eigenvalues of $B$ are discarded, giving the least-squares
best-approximation embedding.}
\end{remark}

\begin{remark}[Simplification for symmetric $A$]
\label{rem:symm-simplification}
When $A$ is symmetric---the standard case, since the Gaussian affinity is
symmetric wherever defined and the only asymmetry is the kNN sparsity pattern,
repaired by $A\leftarrow\max(A,A^\top)$ (which fills missing reverse edges and
alters no existing entry; cf.\ \Cref{rem:gaussian-valid})---we have
$(\tilde{A}^k)_{ij}=(\tilde{A}^k)_{ji}$, so
\begin{equation}\label{eq:D-symm}
  D_{ij} = -\log(\tilde{A}^k)_{ij}.
\end{equation}
Computing the dissimilarity then needs only powers of $\tilde{A}=A/\lmax$; the
$\psi$-reweighting that defines the transition matrix $T$ is required only to
enforce symmetry when $A$ is not symmetric. The MERW operator $T$ is still used
to select the diffusion depth (\Cref{sec:algorithm}), but not for the distance
itself. The dissimilarity $D_{ij}$ is small when many short, high-affinity paths
connect $i$ to $j$, and large when connectivity is sparse or paths are long.
\end{remark}

\begin{remark}[Relation to existing log-diffusion distances]
\label{rem:novelty}
The quantity $-\log(\tilde A^k)_{ij}$ resembles prior log-diffusion distances but
differs in its normalisation. PHATE \citep{Moon2019} and HeatGeo
\citep{Huguet2023} operate on the row-normalised Markov operator $P=D^{-1}A$,
whose entries depend on local degree; normalising instead by the global spectral
radius, $\tilde A=A/\lmax$, makes each edge contribute $A_{uv}/\lmax$, so a
path's weight is its affinity product, undistorted by degree heterogeneity. The
entries $(\tilde A^k)_{ij}$ are then the MERW partition function
\eqref{eq:Z-is-Ak}---a discrete path integral over the $k$-step ensemble,
equivalently the Schr\"odinger heat-kernel amplitude $e^{-tH}$ in the continuum
limit ($t=k/\lmax$; \Cref{lem:dtoc}). Because this entry already pools
exponentially many paths, a single noisy edge is one term among many, so
\entropath\ uses the direct log-amplitude rather than the profile-based
(``triplet'') denoising of PHATE and HeatGeo (\Cref{sec:related}).
\end{remark}

\begin{remark}[Role of $\lmax$]
\label{rem:lmax}
The normalisation $t=k/\lmax$ is required for \Cref{thm:geodesic} and for
interpreting $\tilde{A}^{\,k}=(A/\lmax)^k$ as a discrete approximation of the
Schr\"odinger heat semigroup $e^{-tH}$.
\end{remark}

\subsection{Large-\emph{k} limit}

\begin{proposition}[Large diffusion-depth limit]\label{prop:largek}
Assume $G$ is connected and aperiodic\footnote{Aperiodicity holds automatically
for Gaussian affinity matrices with positive diagonal entries.} with spectral gap
$\delta=1-\mu_2>0$. Let $\{r_\ell\}$ be the right eigenvectors of the MERW
transition matrix $T=\Psimat^{-1}\tilde A\Psimat$\footnote{Equivalently $r_\ell(i)=\frac{\phi_i^{(\ell)}}{\psi_i}$ where $\phi_i$ are the eigenvectors of $A$; see \eqref{eq:vi}}.
Since $T$ is reversible with stationary distribution $\pi_i=\psi_i^2$, hence $\{r_\ell\}$ are orthonormal in $L^2(\pi)$. Then
\[
  D_{ij}
  = -\tfrac{1}{2}\log(\pi_i\pi_j)
    - \mu_2^k\,r_2(i)\,r_2(j)
    + o(\mu_2^k)
  \;\xrightarrow{k\to\infty}\;
  -\tfrac{1}{2}\log(\pi_i\pi_j).
\]
\end{proposition}

\begin{proof}
The spectral theorem for reversible Markov chains \citep{Levin2017} gives
\[
  T^k_{ij}=\pi_j\sum_{\ell\ge 1}\mu_\ell^k\,r_\ell(i)r_\ell(j)
          =\pi_j\Bigl(1+\textstyle\sum_{\ell\ge 2}\mu_\ell^k\,r_\ell(i)r_\ell(j)\Bigr),
\]
where the second equality uses $\mu_1=1$ and $r_1\equiv1$ (the top eigenpair of the
stochastic matrix $T$). Multiplying the $(i,j)$ and $(j,i)$ expansions, the
leading term is $\pi_i\pi_j$ and the two first-order corrections coincide:
\[
  T^k_{ij}\,T^k_{ji}
  = \pi_i\pi_j\bigl(1+2\mu_2^k\,r_2(i)r_2(j)+o(\mu_2^k)\bigr).
\]
By the definition \eqref{eq:D}, $D_{ij}=-\tfrac12\log(T^k_{ij}T^k_{ji})$;
expanding $\log(1+x)=x+o(x)$,
\[
  D_{ij}
  = -\tfrac12\log(\pi_i\pi_j)-\mu_2^k\,r_2(i)r_2(j)+o(\mu_2^k).
\]
Since $\mu_2<1$, the correction vanishes as $k\to\infty$.
\end{proof}

\noindent
In the symmetrised dissimilarity the Perron ratio $\psi_j/\psi_i$ cancels;
$\psi$ nonetheless enters as the leading eigenvector
$\phi^{(1)}$ of the rescaled kernel $\tilde A=A/\lmax$, supplying the stationary
contribution $-\log(\psi_i\psi_j)=-\tfrac12\log(\pi_i\pi_j)$ above, while the
finite-$k$ geodesic information resides in the subdominant modes (the $\mu_2^k$
correction and beyond).

\subsection{Geometric regimes of the dissimilarity}
\label{sec:regimes}

The diffusion depth $k$ sets the diffusion time $t=k/\lmax$ and selects between
two geometric regimes.

\begin{itemize}
  \item \emph{Geodesic regime (short time, $t=k/\lmax$ small).}
    When the diffusion time is short, the log-sum-exp \eqref{eq:free-energy} is
    dominated by the lowest-energy path, so $D_{ij}\approx\min_\gamma E(\gamma)$
    approximates the graph geodesic; \Cref{thm:geodesic} makes this precise via
    Varadhan's short-time formula. The dominance is sharpest when competing
    paths are separated by large energy gaps --- as on sparse, bottlenecked
    manifold graphs --- so $D_{ij}$ tracks the geodesic over a useful range of
    $k$ in that setting.
  \item \emph{Mixing regime (long time, $k$ large).}
    As $k$ grows the walk mixes: exponentially many paths of comparable weight
    contribute, and by \Cref{prop:largek}
    $D_{ij}\to-\tfrac12\log(\pi_i\pi_j)$, which depends only on the stationary
    distribution and discards the finer geodesic geometry.
\end{itemize}

The diffusion depth should therefore be large enough to connect $i$ and $j$ yet
small enough to avoid over-mixing --- the usual trade-off in the choice of $k$,
which we examine empirically in \Cref{app:ablation-knn}.

\begin{remark}[MERW bottleneck advantage for clustering]
\label{rem:bottleneck}
In the geodesic regime, consider two clusters joined by a thin bridge. Under a
standard random walk, probability leaks readily across the bridge because
transition probabilities depend only on local degree. Under MERW, the bridge
carries very few of the high-weight global paths, so its effective energy is
high and $D_{ij}$ across it is large. Inter-cluster dissimilarities are thus sharply
inflated relative to intra-cluster distances, yielding cleaner cluster
boundaries in the embedding --- the discrete analogue of energy barriers
separating metastable basins in a free-energy landscape.
\end{remark}

\definecolor{cnavy}{RGB}{40,72,150}
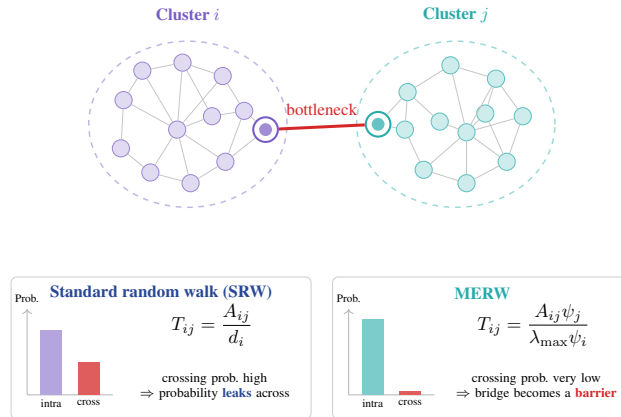
\begin{figure}[htb]
\centering
\resizebox{0.5\textwidth}{!}{%
\begin{tikzpicture}[
  ldot/.style={circle, draw=cpurple!75, fill=cpurple!22, minimum size=3.2mm, inner sep=0},
  rdot/.style={circle, draw=cteal!75,   fill=cteal!22,   minimum size=3.2mm, inner sep=0},
  lend/.style={circle, draw=cpurple, line width=1.1pt, fill=white, minimum size=4.6mm, inner sep=0},
  rend/.style={circle, draw=cteal,   line width=1.1pt, fill=white, minimum size=4.6mm, inner sep=0},
  lcore/.style={circle, fill=cpurple!70, minimum size=2.4mm, inner sep=0},
  rcore/.style={circle, fill=cteal!70,   minimum size=2.4mm, inner sep=0},
  ed/.style={gray!45, line width=0.5pt}]
  \begin{scope}[shift={(0,0)}]
    \draw[cpurple!45, dashed, thick] (0,0) ellipse (1.85cm and 1.55cm);
    \node[ldot] (Lc) at (-0.2,-0.1) {};            
    \node[ldot] (L1) at (-1.2,0.45) {};
    \node[ldot] (L2) at (-0.85,1.0) {};
    \node[ldot] (L3) at (-0.1,1.15) {};
    \node[ldot] (L4) at (0.65,0.9) {};
    \node[ldot] (L5) at (1.05,0.25) {};
    \node[ldot] (L6) at (-1.25,-0.45) {};
    \node[ldot] (L7) at (-0.7,-0.9) {};
    \node[ldot] (L8) at (0.05,-1.1) {};
    \node[ldot] (L9) at (0.7,-0.7) {};
    \node[ldot] (L10) at (0.45,0.15) {};           
    \node[lend] (Le) at (1.45,-0.1) {}; \node[lcore] at (1.45,-0.1) {};
    \foreach \n in {L1,L2,L3,L4,L7,L8,L9,L10}{ \draw[ed] (Lc) -- (\n); }
    \draw[ed] (L1)--(L2) (L2)--(L3) (L3)--(L4) (L4)--(L5) (L5)--(L10)
              (L6)--(L1) (L6)--(L7) (L7)--(L8) (L8)--(L9) (L9)--(L5)
              (L10)--(L3) (L5)--(Le) (L9)--(Le);
  \end{scope}
  \begin{scope}[shift={(5.0,0)}]
    \draw[cteal!45, dashed, thick] (0,0) ellipse (1.85cm and 1.55cm);
    \node[rdot] (Rc) at (0.2,-0.15) {};            
    \node[rdot] (R1) at (-0.9,0.6) {};
    \node[rdot] (R2) at (-0.1,1.1) {};
    \node[rdot] (R3) at (0.75,0.85) {};
    \node[rdot] (R4) at (1.25,0.15) {};
    \node[rdot] (R5) at (0.95,-0.7) {};
    \node[rdot] (R6) at (0.2,-1.1) {};
    \node[rdot] (R7) at (-0.6,-0.85) {};
    \node[rdot] (R8) at (-0.95,-0.1) {};
    \node[rdot] (R9) at (0.55,0.2) {};             
    \node[rdot] (R10) at (-0.3,0.05) {};           
    \node[rend] (Re) at (-1.45,0.0) {}; \node[rcore] at (-1.45,0.0) {};
    \foreach \n in {R2,R3,R4,R5,R6,R7,R9,R10}{ \draw[ed] (Rc) -- (\n); }
    \draw[ed] (R1)--(R2) (R2)--(R3) (R3)--(R4) (R4)--(R5) (R5)--(R6)
              (R6)--(R7) (R7)--(R8) (R8)--(R1) (R9)--(R3) (R10)--(R1)
              (R9)--(R5) (R8)--(Re) (R1)--(Re);
  \end{scope}
  \draw[ccoral, line width=1.6pt] (Le) -- (Re)
        node[midway, above=2pt, font=\footnotesize, ccoral]{bottleneck};
  \node[font=\small\bfseries, cpurple] at (0,2.05) {Cluster $i$};
  \node[font=\small\bfseries, cteal]   at (5.0,2.05) {Cluster $j$};
  \begin{scope}[shift={(-1.9,-4.2)}]
    \draw[gray!40, rounded corners=3pt] (-1.4,-1.3) rectangle (4.2,1.35);
    \node[font=\small\bfseries, cnavy] at (1.4,1.05) {Standard random walk (SRW)};
    \draw[->,gray!60] (-1.1,-0.85) -- (-1.1,0.75) node[above,font=\tiny,black]{Prob.};
    \draw[gray!60] (-1.1,-0.85) -- (0.5,-0.85);
    \fill[cpurple!55] (-0.85,-0.85) rectangle (-0.45,0.35);
    \fill[ccoral!75]  (-0.15,-0.85) rectangle (0.25,-0.25);
    \node[font=\tiny] at (-0.65,-1.0) {intra}; \node[font=\tiny] at (0.05,-1.0) {cross};
    \node at (2.35,0.45) {$T_{ij}=\dfrac{A_{ij}}{d_i}$};
    \node[align=center,font=\scriptsize,text width=3.0cm] at (2.4,-0.7)
      {crossing prob.\ high\\ $\Rightarrow$ probability {\color{cnavy}\bfseries leaks} across};
  \end{scope}
  \begin{scope}[shift={(4.1,-4.2)}]
    \draw[gray!40, rounded corners=3pt] (-1.4,-1.3) rectangle (4.2,1.35);
    \node[font=\small\bfseries, cteal] at (1.4,1.05) {MERW};
    \draw[->,gray!60] (-1.1,-0.85) -- (-1.1,0.75) node[above,font=\tiny,black]{Prob.};
    \draw[gray!60] (-1.1,-0.85) -- (0.5,-0.85);
    \fill[cteal!60]  (-0.85,-0.85) rectangle (-0.45,0.55);
    \fill[ccoral!75] (-0.15,-0.85) rectangle (0.25,-0.78);
    \node[font=\tiny] at (-0.65,-1.0) {intra}; \node[font=\tiny] at (0.05,-1.0) {cross};
    \node at (2.35,0.45) {$T_{ij}=\dfrac{A_{ij}\psi_j}{\lmax\psi_i}$};
    \node[align=center,font=\scriptsize,text width=3.0cm] at (2.4,-0.7)
      {crossing prob.\ very low\\ $\Rightarrow$ bridge becomes a {\color{ccoral}\bfseries barrier}};
  \end{scope}
\end{tikzpicture}%
}
\caption{The bottleneck advantage of MERW (\Cref{rem:bottleneck}). Two dense
  clusters are joined by a single thin bridge between boundary nodes (rings).
  Under the standard random walk the crossing probability depends only on the
  boundary node's local degree, so probability leaks across and cross-cluster
  dissimilarity is underestimated; under MERW the bridge carries few high-weight
  global paths, so its effective energy is high and inter-cluster dissimilarities are
  inflated, yielding cleaner cluster boundaries.}
\label{fig:bottleneck}
\end{figure}

\section{Kernel Structure and Relation to DTNE}
\label{sec:kernel}

\subsection{Diffusion-map feature vectors and the Gram matrix}
\label{sec:embedding}
The normalised affinity $\tilde A=A/\lmax=I-H/\lmax$ shares the
$\ell^2$-orthonormal eigenvectors $\phi^{(\ell)}$ of $H=\lmax I-A$, with
eigenvalues
\[
  \tilde\mu_\ell = 1-\frac{\mu_\ell}{\lmax}\in[-1,1],
  \qquad \tilde\mu_0=1\ \ (\phi^{(0)}=\psi,\ \text{the Perron/ground state}),
\]
where $\mu_\ell\ge0$ are the eigenvalues of $H$. At diffusion depth $k$, define
the \emph{MERW diffusion-map feature vectors} over the positive modes
$\tilde\mu_\ell>0$,
\begin{equation}\label{eq:vi}
  v_i = \bigl(\tilde\mu_\ell^{\,k/2}\,\phi^{(\ell)}_i\bigr)_{\ell:\,\tilde\mu_\ell>0},
  \qquad
  V=\Phi_{+}\,\diag\!\bigl(\tilde\mu_\ell^{\,k/2}\bigr)_{\tilde\mu_\ell>0},
\end{equation}
with $\Phi_+$ the eigenvectors having $\tilde\mu_\ell>0$. Restricting to
$\tilde\mu_\ell>0$ keeps $\tilde\mu_\ell^{\,k/2}$ real for every $k$. Low-frequency
modes ($\tilde\mu_\ell\approx1$, i.e.\ $\mu_\ell\approx0$) decay slowly and carry
global geometry; high-frequency modes decay exponentially in $k$ and are
suppressed, so in practice one retains only the top $d\ll n$ positive modes.

\begin{proposition}[Gram representation of the symmetrised kernel]\label{prop:gram}
Let $A$ be symmetric and entrywise nonnegative, and $T=\Psimat^{-1}\tilde A\Psimat$
the MERW transition matrix ($\Psimat=\diag(\psi)$). Define
$S_{ij}=\sqrt{T^k_{ij}\,T^k_{ji}}$. Then:
\begin{enumerate}
  \item[(i)] the Perron-vector ratios cancel exactly, giving
    \[
      S_{ij}=(\tilde A^k)_{ij}
            =\sum_\ell \tilde\mu_\ell^{\,k}\,\phi^{(\ell)}_i\phi^{(\ell)}_j\;\ge 0;
    \]
  \item[(ii)] if in addition $\tilde A^k$ is positive semidefinite---e.g.\ $k$
    even, or $A$ PSD---then $S$ is a Gram matrix, $S_{ij}=v_i\cdot v_j$ with the
    real vectors \eqref{eq:vi}, i.e.\ $S=VV^\top$.
\end{enumerate}
Without the PSD condition, $VV^\top$ is the positive-eigenvalue part of $S$,
the same part retained by the double-centring step of classical MDS
(\Cref{rem:metric}).
\end{proposition}

\begin{proof}
Since $A$ (hence $\tilde A$) is symmetric, $\tilde A^k$ is symmetric with
$(\tilde A^k)_{ij}=(\tilde A^k)_{ji}=\sum_\ell\tilde\mu_\ell^{\,k}
\phi^{(\ell)}_i\phi^{(\ell)}_j$. From $T^k=\Psimat^{-1}\tilde A^k\Psimat$,
\[
  T^k_{ij}=\frac{\psi_j}{\psi_i}(\tilde A^k)_{ij},
  \qquad
  T^k_{ji}=\frac{\psi_i}{\psi_j}(\tilde A^k)_{ji},
\]
so the Perron ratios cancel:
$T^k_{ij}T^k_{ji}
 =\tfrac{\psi_j}{\psi_i}\tfrac{\psi_i}{\psi_j}
  (\tilde A^k)_{ij}(\tilde A^k)_{ji}
 =(\tilde A^k)_{ij}^2$,
using $(\tilde A^k)_{ij}=(\tilde A^k)_{ji}$. Hence
$S_{ij}=(\tilde A^k)_{ij}\ge0$, the inequality because $A\ge0$ entrywise implies
$\tilde A^k\ge0$ entrywise. This proves (i). For (ii), $\tilde A^k\succeq0$ gives
$\tilde\mu_\ell^{\,k}\ge0$, so $\tilde\mu_\ell^{\,k/2}$ is real and
$S_{ij}=\sum_\ell\bigl(\tilde\mu_\ell^{\,k/2}\phi^{(\ell)}_i\bigr)
\bigl(\tilde\mu_\ell^{\,k/2}\phi^{(\ell)}_j\bigr)=v_i\cdot v_j$; stacking rows
gives $S=VV^\top$. With negative eigenvalues present, the same sum restricted to
$\{\tilde\mu_\ell>0\}$ yields the positive-eigenvalue part.
\end{proof}

\noindent
\Cref{prop:gram} shows that \entropath---the symmetrised free-energy dissimilarity
$D_{ij}=-\log S_{ij}$---is the log-kernel transform of the MERW diffusion kernel
$S$, Gram on its positive modes. Unlike the standard SRW diffusion map (built from
$D^{-1}A$), the MERW version uses the eigenvectors of the symmetric affinity $A$
directly; its stationary distribution $\pi_i\propto\psi_i^2$ encodes global graph
structure rather than purely local degree, which is what the construction is
designed to exploit for robustness to degree heterogeneity (cf.\ \Cref{rem:novelty}).

\section{Geodesic Approximation}
\label{sec:geodesic}

The chain of equivalences underlying the geodesic approximation is summarised in
\Cref{fig:chain} and proved in full in \Cref{app:proof}. Throughout, write
$K^H_t := e^{-tH}$ for the Schr\"odinger heat semigroup of $H=-\Delta_\calM+V$:
for initial data $u_0$ on $\calM$, $u(t)=K^H_t u_0$ solves
$\partial_t u + Hu = 0$, $u(0)=u_0$.

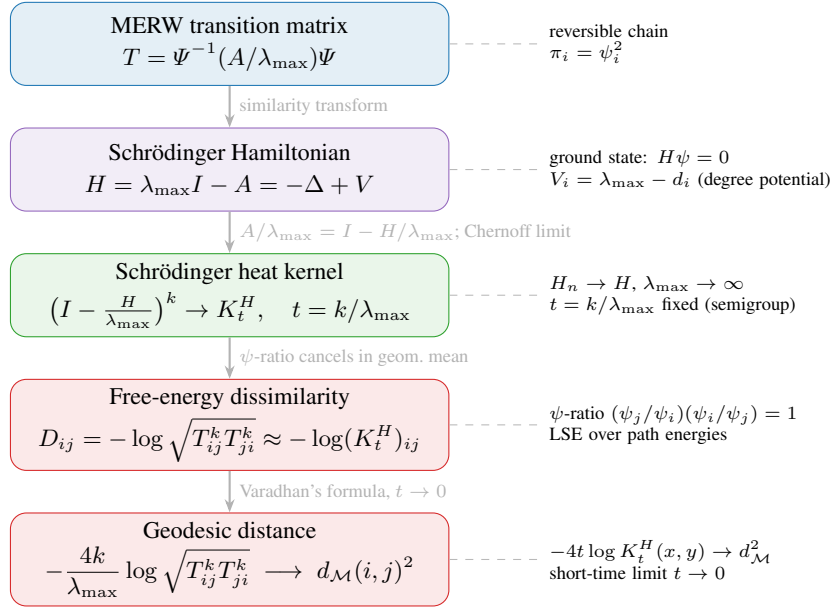
\begin{figure}[H]
\centering
\begin{tikzpicture}[
  box/.style={draw, rounded corners=6pt, align=center,
              minimum height=1.1cm, minimum width=5.8cm, font=\small},
  bluebox/.style  ={box, fill=merw!12,   draw=merw},
  purplebox/.style={box, fill=gram!12,   draw=gram},
  greenbox/.style ={box, fill=kernel!10, draw=kernel},
  redbox/.style   ={box, fill=dist!10,   draw=dist},
  ann/.style={font=\scriptsize, align=left, text width=5.5cm},
  arr/.style={-{Stealth[length=5pt]}, thick, gray!60},
  leader/.style={dashed, gray!55, thin},
  node distance=0.55cm
]
  \node[bluebox] (N1) {MERW transition matrix\\[2pt]
    $\displaystyle T = \Psimat^{-1}(A/\lmax)\Psimat$};
  \node[purplebox, below=of N1] (N2) {Schrödinger Hamiltonian\\[2pt]
    $H = \lmax I - A = -\Delta + V$};
  \node[greenbox, below=of N2] (N3) {Schrödinger heat kernel\\[2pt]
    $\bigl(I-\tfrac{H}{\lmax}\bigr)^k \to K^H_t,\quad t=k/\lmax$};
  \node[redbox, below=of N3] (N4) {Free-energy dissimilarity\\[2pt]
    $D_{ij}=-\log\sqrt{T^k_{ij}T^k_{ji}}\approx -\log(K^H_t)_{ij}$};
  \node[redbox, below=of N4] (N5) {Geodesic distance\\[2pt]
    $\displaystyle -\frac{4k}{\lmax}\log\sqrt{T^k_{ij}T^k_{ji}}\;\longrightarrow\;\dM(i,j)^2$};
  \draw[arr] (N1) -- node[right, font=\scriptsize]{similarity transform} (N2);
  \draw[arr] (N2) -- node[right, font=\scriptsize]{$A/\lmax=I-H/\lmax$; Chernoff limit} (N3);
  \draw[arr] (N3) -- node[right, font=\scriptsize]{$\psi$-ratio cancels in geom.\ mean} (N4);
  \draw[arr] (N4) -- node[right, font=\scriptsize]{Varadhan's formula, $t\to 0$} (N5);
  \node[ann, right=1.2cm of N1] (A1) {reversible chain\\$\pi_i=\psi_i^2$};
  \node[ann, right=1.2cm of N2] (A2)
    {ground state: $H\psi=0$\\$V_i=\lmax-d_i$ (degree potential)};
  \node[ann, right=1.2cm of N3] (A3)
    {$H_n\to H$, $\lmax\to\infty$\\$t=k/\lmax$ fixed (semigroup)};
  \node[ann, right=1.2cm of N4] (A4)
    {$\psi$-ratio $(\psi_j/\psi_i)(\psi_i/\psi_j)=1$\\LSE over path energies};
  \node[ann, right=1.2cm of N5] (A5)
    {$-4t\log K^H_t(x,y)\to\dM^2$\\short-time limit $t\to 0$};
  \draw[leader] (N1.east) -- (A1.west);
  \draw[leader] (N2.east) -- (A2.west);
  \draw[leader] (N3.east) -- (A3.west);
  \draw[leader] (N4.east) -- (A4.west);
  \draw[leader] (N5.east) -- (A5.west);
\end{tikzpicture}
\caption{Chain of equivalences from MERW to geodesic distance, via a two-stage
  limit: $\lmax\to\infty$ at fixed $t=k/\lmax$ (semigroup limit), then $t\to0$
  (Varadhan).
  \textcolor{merw}{Blue}: MERW operator.
  \textcolor{gram}{Purple}: Schrödinger/Hamiltonian structure.
  \textcolor{kernel}{Green}: heat kernel.
  \textcolor{dist}{Red}: free-energy dissimilarity and geodesic limit.}
\label{fig:chain}
\end{figure}


\begin{assumption}[Spectral convergence]\label{ass:spectral-conv}
For the graph sequence $G_n=(V_n,E_n,A_n)$, with Perron eigenvalue
$\lmax^{(n)}\to\infty$, the discrete Hamiltonians $H_n=\lmax^{(n)}I-A_n$ converge
spectrally to a Schr\"odinger operator $H=-\Delta_\calM+V$ on a closed Riemannian
manifold $(\calM,g)$, where $V\in C(\calM)$ (hence bounded, as $\calM$ is compact).
Concretely, for each fixed $\ell$ the $\ell$-th eigenvalue of $H_n$ converges to
the $\ell$-th eigenvalue of $H$, and the corresponding eigenvectors, interpolated
onto $\calM$, converge in $L^2(\calM)$ to the corresponding eigenfunctions.
\end{assumption}

\noindent
Convergence of the principal part to the Laplace--Beltrami operator $-\Delta_\calM$
is established for kernel graphs built from i.i.d.\ samples under appropriate
bandwidth scaling \citep{Belkin2007,Coifman2006}; the surviving potential $V$
reflects the unnormalised affinity used by MERW \citep{Hein2007}. Throughout, the
diffusion time is parametrised by $t:=k/\lmax^{(n)}$.

\noindent
\begin{remark}[Laplacian and potential]\label{rem:laplacians}
The limiting operator in \Cref{thm:geodesic} is the Schr\"odinger operator
$H=-\Delta_\calM+V$, whose leading term is the non-negative Laplace--Beltrami
operator
$-\Delta_\calM f=-\tfrac{1}{\sqrt{|g|}}\partial_i(\sqrt{|g|}\,g^{ij}\partial_j f)$
on $(\calM,g)$, the continuum limit of graph Laplacians built from kernel
affinities \citep{Belkin2007,Coifman2006}. A potential appears because MERW uses
the \emph{unnormalised} affinity $A_n$: unlike density-normalised constructions,
which cancel the sampling density and converge to $-\Delta_\calM$ alone
\citep{Hein2007}, the density here survives in the limit. It reflects the degree heterogeneity of the graph (discretely, $V_i=\lmax^{(n)}-d_i$).

\noindent
For \Cref{thm:geodesic} the precise form of $V$ is
immaterial: the proof uses only that the \emph{limiting} potential $V$ is bounded and continuous on the compact $\calM$, so by
Feynman--Kac it perturbs the heat kernel at $O(t)$ and leaves Varadhan's
exponential distance term unchanged.
\end{remark}

\noindent
\begin{theorem}[Short-time geodesic approximation for MERW]
\label{thm:geodesic}
Let $\{G_n\}_{n\ge 1}$ converge spectrally (\Cref{ass:spectral-conv}) to a closed
$d$-dimensional Riemannian manifold $(\calM,g)$, identify each vertex $i\in V_n$
with a sample point $x_i\in\calM$, and let $T_n$ be the MERW transition matrix on
$G_n$. For a diffusion time $t>0$ set $k=\lfloor t\,\lmax^{(n)}\rfloor$ and
\[
  D^{(n)}_{ij}(t) \;:=\; -\log\sqrt{T_n^{k}(i,j)\,T_n^{k}(j,i)}.
\]
Then there is a sequence $t_n\to0^+$, with depths
$k_n=\lfloor t_n\lmax^{(n)}\rfloor\to\infty$, along which $4t_n D^{(n)}_{ij}(t_n)$
recovers the squared geodesic distance, namely, for each pair
$x_i\neq x_j$,
\begin{equation}\label{eq:varadhan}
  \boxed{\,
    \lim_{n\to\infty}\;4\,t_n\,D^{(n)}_{ij}(t_n)
    \;=\; \dM(x_i,x_j)^2.
  \,}
\end{equation}
\end{theorem}

\begin{remark}[Level constant and the iterated limit]\label{rem:Zn}
At fixed $t>0$, Steps~1--2 below give
$D^{(n)}_{ij}(t)=\log Z_n-\log K^H_t(x_i,x_j)+o(1)$, where $K^H_t$ is the heat
kernel of $H=-\Delta_\calM+V$ and $Z_n=\Theta(n)$
is a pair-independent sampling-density factor (the log-partition / mixing offset; cf.\ the large-$k$ limit $-\tfrac12\log(\pi_i\pi_j)$ of \Cref{prop:largek}). The prefactor $4t_n$ suppresses this level term precisely when $t_n\log n\to0$, e.g.\
$t_n=\Theta(1/\log n)$. Equivalently, the centred dissimilarity
$\widetilde D^{(n)}_{ij}:=D^{(n)}_{ij}-\log Z_n$ satisfies the iterated limit
$\lim_{t\to0^+}\lim_{n\to\infty}4t\,\widetilde D^{(n)}_{ij}(t)=\dM(x_i,x_j)^2$.
The constant $\log Z_n$ is a global additive offset on $D$, not a geometric
distortion.
\end{remark}

\begin{proof}[Proof sketch]
We outline the steps; full details are in \Cref{app:proof}.

\textbf{Step 1 (Discrete-to-continuum limit, $n\to\infty$ at fixed $t$).}
Writing $A_n/\lmax^{(n)}=I-H_n/\lmax^{(n)}$ and expanding the matrix power in the
eigenbasis, spectral convergence (\Cref{ass:spectral-conv}) sends each fixed mode
$(1-\mu_\ell^{(n)}/\lmax^{(n)})^{k}\to e^{-t\mu_\ell}$. With the $\ell^2$-to-$L^2$
normalisation factor $Z_n=\Theta(n)$, the rescaled powers converge to the heat
kernel of $H=-\Delta_\calM+V$:
\[
  Z_n\,(A_n/\lmax^{(n)})^{k}_{ij}\;\longrightarrow\;K^H_t(x_i,x_j)
  \qquad (n\to\infty).
\]

\textbf{Step 2 (Symmetrisation cancels the Perron ratio, exactly at finite $n$).}
With $T_n=\Psimat^{-1}(A_n/\lmax^{(n)})\Psimat$ and $A_n$ symmetric, the Perron
ratios cancel in the geometric mean with no residual,
$\sqrt{T_n^{k}(i,j)\,T_n^{k}(j,i)}=(A_n/\lmax^{(n)})^{k}_{ij}$, so
\[
  D^{(n)}_{ij}(t)=\log Z_n-\log K^H_t(x_i,x_j)+o(1).
\]

\textbf{Step 3 (Short-time asymptotics, $t\to0$).}
By Varadhan's formula for $H=-\Delta_\calM+V$, the bounded potential affects only
the $O(1)$ prefactor, so $-4t\log K^H_t(x_i,x_j)\to\dM(x_i,x_j)^2$ as $t\to0^+$.

\textbf{Combining.} Choosing $t_n\to0^+$ with $t_n\log n\to0$ and
$k_n=\lfloor t_n\lmax^{(n)}\rfloor\to\infty$, the level term obeys
$4t_n\log Z_n = 4t_n\bigl(\log n+O(1)\bigr)\to 0$
while Step~3 gives
$-4t_n\log K^H_{t_n}\to\dM^2$; a standard joint extraction over $t\to0^+$ and
$n\to\infty$ then yields the sequence of \eqref{eq:varadhan}.
\eqref{eq:varadhan}.
\end{proof}

\begin{remark}[On the constant $4$]\label{rem:factor4}
The factor $4$ arises from Varadhan's formula for the convention
$\partial_t u=\Delta u$ used throughout; it becomes $2$ under the probabilist's
convention $\partial_t u=\tfrac12\Delta u$. HeatGeo \citep{Huguet2023} uses the
same convention, $d_t(x,y)^2=-4t\log(H_t)_{xy}\to\dM(x,y)^2$.
\end{remark}

\section{Algorithm and Application}
\label{sec:algorithm}

\subsection{\entropath embedding}

Given a point cloud $X=\{x_i\}_{i=1}^N\subset\R^d$, the pipeline constructs a
low-dimensional embedding via the steps of \Cref{alg:entropath}.

\begin{algorithm}[t]
\caption{\entropath embedding}
\label{alg:entropath}
\begin{algorithmic}[1]
\Require $N\times d$ dataset $X$, neighbours $k_{\mathrm{NN}}$, embedding
         dimension $d_{\mathrm{emb}}$, maximum depth $t_{\max}$
\Ensure Embedding $Y\in\R^{N\times d_{\mathrm{emb}}}$
\Statex
\Statex \textbf{// Step 1: Build affinity graph}
\State $\sigma_i \gets k_{\mathrm{NN}}\text{-NN distance of }x_i$
  \Comment{Local bandwidth}
\State $A_{ij} \gets \exp\!\left(-\|x_i-x_j\|^2/(\sigma_i\sigma_j)\right)$ on the
  $k_{\mathrm{NN}}$ graph
  \Comment{Adaptive Gaussian}
\State $A \gets \max(A,A^\top)$ \Comment{Symmetrise}
\Statex
\Statex \textbf{// Step 2: MERW operator and normalisation}
\State $(\lmax,\psi)\gets\mathrm{Perron}(A)$
\State $T_{ij}\gets (A_{ij}/\lmax)\,(\psi_j/\psi_i)$;\quad $\tilde A\gets A/\lmax$
  \Comment{$T$ stochastic; $\tilde A$ symmetric, similar to $T$}
\Statex
\Statex \textbf{// Step 3: Select diffusion depth on the MERW operator}
\State $k \gets \arg\text{knee}\bigl\{\mathcal{S}_{\mathrm{vN}}(T^{\,t})\bigr\}_{t=1}^{t_{\max}}$
  \Comment{von Neumann entropy knee (PHATE criterion)}
\Statex
\Statex \textbf{// Step 4: Diffusion power and squared free-energy dissimilarity}
\State $\tilde A^{\,k} \gets \textsc{MatrixPower}(\tilde A, k)$
\State $D^{(2)}_{ij}\gets -\log(\tilde A^{\,k})_{ij}$,\quad $D^{(2)}_{ii}\gets 0$
  \Comment{Squared free-energy dissimilarity}
\Statex
\Statex \textbf{// Step 5: Embed}
\State $Y \gets \textsc{MetricMDS}(D^{(2)},\, d_{\mathrm{emb}})$, initialised at the classical-MDS solution
       \Comment{SMACOF, cMDS init}
\Statex
\Return $Y$
\end{algorithmic}
\end{algorithm}

The affinity in Step~1 is written for the adaptive-bandwidth Gaussian; we use
it on the synthetic manifolds, and the $\alpha$-decay kernel
$A_{ij}=\exp(-(d_{ij}/\sigma_i)^{\alpha})$ of \citet{Moon2019} (decay $\alpha=40$)
on the single-cell data, following that literature. By \Cref{rem:kernel} any
affinity satisfying \Cref{ass:spectral-conv} may be substituted without
affecting the recovered geodesic, and the two are compared empirically in
\Cref{app:ablation-kernel}.

\paragraph{Diffusion depth.}
The depth $k$ is selected, not supplied: following PHATE \citep{Moon2019}, we take
the knee of the von Neumann entropy $\mathcal{S}_{\mathrm{vN}}(T^{\,t})$ as a
function of $t$, the point past which further diffusion mainly mixes rather than
resolves structure (cf.\ the mixing limit of \Cref{prop:largek}). The entropy is
computed on the MERW transition operator $T$; since $T=\Psimat^{-1}\tilde A\Psimat$
is similar to $\tilde A$, the two share the spectrum $\{\mu_m\}$ that drives the
diffusion, so the criterion is intrinsic to the MERW walk while the distance is
read off the symmetric power $\tilde A^{\,k}$.

\paragraph{Distance scaling.}
\Cref{alg:entropath} embeds the free-energy dissimilarity using metric MDS
(SMACOF), initialized with the classical-MDS solution computed from the
squared-distance matrix $D^{(2)}$. The global prefactor $4k/\lambda_{\max}$ in
\Cref{thm:geodesic} is a constant shared by all pairwise dissimilarities, so
omitting it merely changes the overall scale of the embedding without affecting
its relative geometry. Consequently, it has no effect on the trustworthiness,
continuity, or rank-correlation metrics reported in our experiments. When a non-squared dissimilarity is required (e.g.\ the Level-1 metrics), we use $\sqrt{D^{(2)}}$, which recovers the geodesic up to the same global factor $\sqrt{4k/\lambda_{\max}}$; being a single constant, it too leaves the scale-invariant metrics we report unchanged.

\paragraph{Complexity.}
Step~4 forms a dense power, $O(N^3)$ time and $O(N^2)$ memory --- the cost
\entropath shares with other dense-distance methods at full resolution.
Scalability to large $N$ comes from the landmark and out-of-sample extensions of
\Cref{sec:scalable}.

\section{Scalable Extensions}
\label{sec:scalable}

For large datasets, forming the full $N\times N$ dissimilarity matrix and
diffusing the complete operator become impractical. We therefore introduce two
scalable extensions: a landmark-based out-of-sample projection for embedding new
points, and a diffusion-potential pseudotime method for trajectory inference.

\subsection{Landmark-based out-of-sample projection}
\label{sec:landmarks}

Let $X_{\mathrm{lm}}=\{x_j\}_{j=1}^M\subset X$ be a set of $M\ll N$ landmarks and
let $Z_{\mathrm{lm}}\in\R^{M\times d_{\mathrm{emb}}}$ be the embedding computed on
them via \Cref{alg:entropath}. We extend the embedding by interpolation to any
point $x\in\R^d$, including both the remaining $N-M$ data points and genuinely
new samples, without recomputing the global MERW operator. When the ambient
dimension exceeds $100$, graph construction is performed in the first $100$
principal components (as in PHATE), which reduces noise and accelerates
subsequent computations; lower-dimensional inputs are used directly.

\paragraph{Landmark selection.}
By default, landmarks are $k$-means centroids of $X$ (replaced by their nearest
data points), which provide good coverage of the manifold and capture its
geometry better than uniform sampling; specified points (e.g.\ a known root cell)
may be forced into the set. Farthest-point sampling is a useful alternative that
maximises coverage and better preserves sparsely sampled regions, such as branch
tips, although at the expense of representing dense regions in proportion to
their sampling density --- often preferable for trajectory data with rare
progenitor or terminal states.

\paragraph{Adaptive bandwidths.}
Following local-scaling practice \citep{ZelnikManor2004}, each landmark $x_j$ is
assigned the bandwidth $\sigma_j=\|x_j-x_{j,(k_{\mathrm{NN}})}\|$, its
$k_{\mathrm{NN}}$-th nearest-neighbour distance \emph{among landmarks} (matching
the affinity graph of \Cref{alg:entropath}), and a query $x$ is assigned
$\sigma(x)=\|x-x_{(k_{\mathrm{proj}})}\|$, the distance to its
$k_{\mathrm{proj}}$-th nearest landmark, where $k_{\mathrm{proj}}$ ($\approx 50$)
is the projection neighbourhood size. These pointwise scales adapt to local
density.

\paragraph{Kernel weights.}
Let $\mathcal{N}(x)$ index the $k_{\mathrm{proj}}$ nearest landmarks to $x$. For
each $j\in\mathcal{N}(x)$ define the affinity
\begin{equation}\label{eq:lm-kernel}
  w_j(x) = \exp\!\left(-\frac{\|x-x_j\|^2}{\sigma(x)\,\sigma_j}\right),
\end{equation}
the same adaptive product-bandwidth Gaussian form used to construct $A$ in
\Cref{alg:entropath} (with a coarser query bandwidth $\sigma(x)$), so the
projection is consistent with the operator that produced $Z_{\mathrm{lm}}$.
Normalising over neighbours,
\[
  \tilde{w}_j(x) = \frac{w_j(x)}{\sum_{l\in\mathcal{N}(x)}w_l(x)},
\]
defines a projection operator $P\in\R^{N\times M}$ that maps each point to a
distribution over landmarks.

\paragraph{Embedding interpolation.}
The out-of-sample embedding is the convex combination of landmark coordinates,
\begin{equation}\label{eq:lm-embed}
  Z(x) = \sum_{j\in\mathcal{N}(x)}\tilde{w}_j(x)\,Z_{\mathrm{lm}}(j).
\end{equation}
This is a localised Nystr\"om extension \citep{Coifman2006}: \eqref{eq:lm-embed}
interpolates $x$ into the diffusion geometry already encoded by the MERW landmark
embedding $Z_{\mathrm{lm}}$ via a standard (row-normalised) kernel projection. The
MERW geometry --- the global $\lmax$ normalisation and the free-energy
dissimilarity --- is computed once, in $Z_{\mathrm{lm}}$; the projection simply
places new points within this existing geometry, so no global operator is
recomputed.

\paragraph{Complexity.}
Embedding the landmarks costs $O(M^3)$ (the matrix power and MDS on the
$M\times M$ landmark dissimilarities). Projecting the remaining points is
dominated by the nearest-landmark search, $O(NM)$ by brute force (less with a
spatial index in low dimension), plus $O(N\,k_{\mathrm{proj}}\,d_{\mathrm{emb}})$
for the interpolation \eqref{eq:lm-embed} --- both linear in $N$, versus $O(N^3)$
for the full embedding. With $M\ll N$ this is the source of \entropath's
scalability to large datasets (\Cref{sec:experiments}).

\section{Experiments}
\label{sec:experiments}

We evaluate \entropath on synthetic manifolds with known geometry and on
biological single-cell datasets. Within each setting, all methods use the
same underlying graph, so that performance differences reflect the choice
of random walk and dissimilarity rather than graph construction. For the
synthetic benchmarks (\Cref{sec:exp-synthetic}) we use an adaptive
Gaussian kernel
$A_{ij}=\exp\!\big(-\|x_i-x_j\|^2/(\sigma_i\sigma_j)\big)$
\citep{ZelnikManor2004}, with a per-point bandwidth $\sigma_i$ equal to
the distance to the $k_{\mathrm{NN}}$-th nearest neighbour; this is the
kernel we use to instantiate \Cref{thm:geodesic}. For the single-cell
datasets (\Cref{sec:singlecell}) we use the $\alpha$-decay kernel standard
in that literature \citep{Moon2019}, which adapts more gracefully to the
strong density variation of scRNA-seq data. The kernel ablation
(\Cref{app:ablation-kernel}) compares the two profiles directly and
confirms that this choice is not what drives \entropath's advantage. Both
kernels are symmetrised by $A\leftarrow\max(A,A^\top)$, with
$k_{\mathrm{NN}}=15$ as the default neighbourhood size (varied in
\Cref{app:ablation-knn}).

\paragraph{Distance-level and embedding-level evaluation.}
Distance-level evaluation asks whether a method recovers the correct
geometry; embedding-level evaluation asks whether that geometry survives
the embedding stage. Each method produces a dissimilarity matrix
$D \in \mathbb{R}^{N \times N}$ and, after an embedding step, a
two-dimensional layout $Y \in \mathbb{R}^{N \times 2}$; we evaluate both.
\textbf{Distance-level evaluation} compares $D$ directly against the
ground-truth geodesic matrix $G$ (row-wise Spearman and Pearson
correlation), isolating the quality of the kernel and dissimilarity
construction from any downstream embedding. \textbf{Embedding-level
evaluation} assesses the layout $Y$ using neighbourhood-preservation
metrics (trustworthiness and continuity) and geodesic correlation
(Spearman and Pearson), computed from embedded Euclidean distances
$\|y_i - y_j\|$; on the single-cell datasets we additionally report DEMaP.

We require a ground-truth geodesic against which to score each method.
Where a closed-form geodesic exists (Swiss roll, sphere) we use it as the
primary reference, since it is the exact Riemannian distance rather than a
graph-based proxy; elsewhere (torus, Swiss hole, trees, and all
single-cell data) we approximate it by shortest-path distances on a 15-NN
graph over the clean reference data, the standard DEMaP convention
\citep{Moon2019}. Although widely used, this shortest-path proxy introduces a systematic bias:
it favours methods whose own dissimilarities are built on the
same $k$-NN graph, inflating their measured geodesic correlation.
Appendix~\ref{app:analytic-validation} quantifies this effect, comparing
the two protocols directly --- they coincide on the sphere, whereas on the
Swiss roll the proxy systematically advantages kNN-aligned methods. We
therefore report both protocols wherever an analytic geodesic is
available, and treat the analytic one as primary.

\paragraph{Baselines.}
We evaluate two groups of methods. \emph{Geodesic-preserving methods} ---
which are explicitly designed to recover manifold geometry --- form our
primary comparison: PHATE \citep{Moon2019}, HeatGeo \citep{Huguet2023},
DTNE \citep{Wei2025}, Diffusion Maps \citep{Coifman2006}, and Isomap
\citep{Tenenbaum2000} (shortest-path baseline). \emph{Local
neighbourhood methods} --- UMAP \citep{McInnes2018}, $t$-SNE
\citep{vanderMaaten2008}, and PCA --- are reported for completeness in
Appendix~\ref{app:additional-experiments} but optimise a fundamentally
different objective (neighbourhood preservation rather than geodesic
recovery) and are not expected to perform well on geodesic correlation.
Diffusion-scale parameters ($t$ or $k$) follow each method's native
selection procedure rather than being fixed externally: PHATE and
\entropath both use the von Neumann entropy criterion, applied to their
respective diffusion operators, while the remaining baselines use their
own author-recommended procedures. Each method is therefore evaluated at
its intended scale.

\paragraph{Metrics.}
We use the following complementary metrics. \emph{Geodesic correlation}:
Spearman ($\rho$) and Pearson correlation between pairwise dissimilarities
and ground-truth geodesic distances $d_{\mathcal{M}}(i,j)$; higher is
better. Spearman measures rank agreement, whereas Pearson measures the
approximately linear relationship predicted by the short-time form of
\Cref{thm:geodesic}; together they assess the geometric prediction of the
theorem. \emph{Trustworthiness} ($\mathcal{T}$) \citep{Venna2006}: whether
points close in the embedding were also close in the original space;
higher is better. \emph{Continuity} ($\mathcal{C}$): the converse ---
whether points close in the original space remain close in the embedding;
higher is better. On the single-cell datasets, where no ground-truth
geodesic exists, we additionally report \emph{DEMaP} \citep{Moon2019}, the
field-standard denoised manifold-preservation score: the global Spearman
correlation between shortest-path geodesic distances on a kNN graph over
the high-dimensional reference data and Euclidean distances in the
embedding, averaged over random subsamples. Together with \emph{pseudotime
correlation} and \emph{silhouette score}
(Appendix~\ref{app:pseudotime_comparison}), DEMaP provides the primary
basis for comparison with PHATE, HeatGeo, and DTNE; full protocol
parameters are given in Appendix~\ref{app:pseudotime_comparison}.

\subsection{Synthetic Benchmarks}
\label{sec:exp-synthetic}

We evaluate on manifolds with known geometry: a Swiss roll under uniform
and non-uniform (Beta$(1,4)$) sampling, where the latter imposes a smooth
density gradient that tests the path-uniformization claim
(Remark~\ref{rem:bottleneck}), and a sparse artificial tree
\citep{Moon2019} ($6$ branches, $\mathbb{R}^{10}$) as a branching
trajectory. Each uses $N=2{,}000$ points and $30$ seeds. Four further manifolds --- sphere, torus, Swiss hole, and a denser tree ---
are uniformly sampled or locally flat and do not probe the density-adaptive
mechanism of \entropath; we report them, with the full method panel,
in Appendix~\ref{app:additional-experiments}.

\paragraph{Distance-level results.}
We first ask how faithfully each method's dissimilarity matrix recovers the
ground-truth geodesic --- the quantity \Cref{thm:geodesic} characterises ---
before any embedding. On the sparse tree (Table~\ref{tab:level1-tree-main})
\entropath attains the highest Spearman and Pearson correlation among all
methods, exceeding even the shortest-path Isomap baseline from which the
tree's ground truth is built. On the Swiss roll
(Table~\ref{tab:level1-swiss-main}) \entropath is the strongest
diffusion-based method under the analytic geodesic in both sampling regimes,
ahead of PHATE, HeatGeo, and DTNE; ambient Euclidean distance and the
shortest-path Isomap reference score higher on the non-uniform roll, a point
we return to below. The shortest-path columns expose the effect of the
evaluation protocol: methods whose dissimilarities are computed on the same
15-NN graph as the proxy approach perfect agreement with it (Isomap and
Shortest Path reach $1.000$ on the non-uniform roll, where their construction
coincides with the ground truth). This same protocol-alignment explains the
apparent strength of ambient and shortest-path distances at the distance
level: on a near-developable manifold with a strong density gradient, raw
Euclidean distance correlates highly with arc-length over the densely sampled
region, which dominates the row-wise average. The embedding level
(Table~\ref{tab:emb-main}) removes this advantage --- neither Euclidean nor
Isomap unrolls the manifold --- and there \entropath leads all methods.
Pearson closely tracks Spearman throughout, indicating that the recovered
dissimilarities are not merely rank-consistent but approximately linear in
the geodesic, as predicted by the short-time form of \Cref{thm:geodesic}.
All comparisons use a single pipeline with graph construction held fixed
across methods; consequently, absolute correlations should not be compared
across datasets, and only within-table comparisons are meaningful.

\paragraph{Embedding-level results.}
The embedding level is the more decisive test: it asks whether the recovered
geometry survives projection to two dimensions, and it is where a method's
practical value lies. Here \entropath leads every diffusion-based method by a
clear margin on both manifolds (Table~\ref{tab:emb-main}). On the non-uniform
Swiss roll it attains the best trustworthiness ($0.981$, vs.\ $\le 0.957$ for
the other diffusion methods) and the best embedded geodesic correlation
(Spearman $0.875$ vs.\ $0.805$ for DTNE and $\le 0.741$ for PHATE, HeatGeo,
and Diffusion Maps); on the sparse tree it leads the embedded correlations
(Spearman $0.869$ vs.\ $0.859$ for DTNE) and ties for the best continuity.
Notably, \entropath recovers this geometry \emph{after} embedding even where
the linear baselines lead at the distance level, confirming that its
advantage is in producing a faithful low-dimensional layout rather than a raw
distance matrix. The embeddings themselves (Figures~\ref{fig:swiss-roll}
and~\ref{fig:tree}) illustrate the mechanism. On the non-uniform Swiss roll,
\entropath keeps sparsely sampled points separated from the main arc instead
of smoothing them into it, whereas PHATE and HeatGeo blur local geometry in
sparse regions. On the sparse tree, the path-integral methods recover the
branch topology while PHATE fragments it. A denser tree configuration
($20$ short branches in $\mathbb{R}^{100}$), where \entropath's
entropy-based depth selection becomes noisier and DTNE leads, is analysed in
Appendix~\ref{app:tree-full} and discussed as a limitation in
Section~\ref{sec:conclusion}. Full embedding-level panels with the complete
method baseline set, and DEMaP, are in
Appendix~\ref{app:additional-experiments}.

\begin{table}[H]
\centering
\small
\setlength{\tabcolsep}{6pt}
\caption{Distance-level row-wise correlation with the ground-truth
  geodesic on the \textbf{sparse tree} (shortest-path geodesic; DLA
  trees admit no closed form). This isolates dissimilarity quality from
  any embedding step. \entropath attains the highest correlation in
  both columns, exceeding even the shortest-path Isomap baseline that
  the ground truth is built from. \textbf{Bold} = best,
  \underline{underlined} = second-best.}
\label{tab:level1-tree-main}
\begin{tabular}{lcc}
\toprule
Distance               & Spearman          & Pearson \\
\midrule
Euclidean              & 0.806 $\pm$ 0.047 & 0.830 $\pm$ 0.041 \\
Isomap / Shortest Path & 0.917 $\pm$ 0.042 & 0.921 $\pm$ 0.038 \\
Diffusion Maps         & 0.491 $\pm$ 0.088 & 0.543 $\pm$ 0.044 \\
PHATE                  & 0.318 $\pm$ 0.070 & 0.615 $\pm$ 0.043 \\
HeatGeo                & 0.918 $\pm$ 0.040 & 0.886 $\pm$ 0.026 \\
DTNE                   & \underline{0.927 $\pm$ 0.039} & \underline{0.921 $\pm$ 0.031} \\
\entropath              & \textbf{0.934 $\pm$ 0.037} & \textbf{0.935 $\pm$ 0.033} \\
\bottomrule
\end{tabular}
\end{table}

\begin{table}[H]
\centering
\small
\setlength{\tabcolsep}{6pt}
\caption{Distance-level row-wise correlation on the \textbf{Swiss
  roll}, under the \emph{analytic} arc-length geodesic ($A$) and the
  \emph{shortest-path} proxy ($SP$), for uniform and non-uniform
  (Beta$(1,4)$) sampling. Among the diffusion-based methods \entropath is
  strongest under the analytic protocol in both regimes; ambient Euclidean
  distance and the shortest-path Isomap reference score higher on the
  non-uniform roll, for the protocol-alignment reasons discussed in the
  text. The shortest-path proxy inflates methods built on the same 15-NN
  backbone (Isomap and Shortest Path reach $1.000$ on the non-uniform roll,
  where their construction coincides with the ground truth). \entropath's
  correlation is also far less protocol-sensitive than the kNN-aligned
  baselines (e.g.\ HeatGeo $0.791\!\to\!0.962$ vs.\ \entropath
  $0.877\!\to\!0.935$ on uniform Spearman). \textbf{Bold} = best,
  \underline{underlined} = second-best.}
\label{tab:level1-swiss-main}
\begin{tabular}{l cc cc}
\toprule
 & \multicolumn{2}{c}{Analytic ($A$)} & \multicolumn{2}{c}{Shortest-path ($SP$)} \\
\cmidrule(lr){2-3}\cmidrule(lr){4-5}
Distance               & Spearman & Pearson & Spearman & Pearson \\
\midrule
\multicolumn{5}{l}{\emph{Uniform sampling}} \\
Euclidean              & 0.528 $\pm$ 0.008 & 0.537 $\pm$ 0.008 & 0.734 $\pm$ 0.045 & 0.756 $\pm$ 0.046 \\
Isomap / Shortest Path & 0.829 $\pm$ 0.032 & \underline{0.827 $\pm$ 0.033} & \textbf{0.993 $\pm$ 0.017} & \textbf{0.994 $\pm$ 0.014} \\
Diffusion Maps         & 0.499 $\pm$ 0.016 & 0.537 $\pm$ 0.009 & 0.448 $\pm$ 0.015 & 0.541 $\pm$ 0.014 \\
PHATE                  & 0.806 $\pm$ 0.029 & 0.804 $\pm$ 0.029 & 0.936 $\pm$ 0.008 & 0.934 $\pm$ 0.006 \\
HeatGeo                & 0.791 $\pm$ 0.027 & 0.786 $\pm$ 0.029 & \underline{0.962 $\pm$ 0.013} & \underline{0.964 $\pm$ 0.011} \\
DTNE                   & \underline{0.833 $\pm$ 0.040} & 0.823 $\pm$ 0.038 & 0.958 $\pm$ 0.029 & 0.954 $\pm$ 0.019 \\
\entropath              & \textbf{0.877 $\pm$ 0.023} & \textbf{0.864 $\pm$ 0.025} & 0.935 $\pm$ 0.030 & 0.931 $\pm$ 0.029 \\
\midrule
\multicolumn{5}{l}{\emph{Non-uniform sampling}} \\
Euclidean              & \underline{0.899 $\pm$ 0.007} & \underline{0.885 $\pm$ 0.009} & \underline{0.977 $\pm$ 0.002} & \underline{0.979 $\pm$ 0.002} \\
Isomap / Shortest Path & \textbf{0.934 $\pm$ 0.006} & \textbf{0.919 $\pm$ 0.008} & \textbf{1.000 $\pm$ 0.000} & \textbf{1.000 $\pm$ 0.000} \\
Diffusion Maps         & 0.443 $\pm$ 0.025 & 0.502 $\pm$ 0.014 & 0.406 $\pm$ 0.024 & 0.444 $\pm$ 0.020 \\
PHATE                  & 0.816 $\pm$ 0.006 & 0.803 $\pm$ 0.009 & 0.892 $\pm$ 0.004 & 0.893 $\pm$ 0.004 \\
HeatGeo                & 0.796 $\pm$ 0.008 & 0.779 $\pm$ 0.009 & 0.889 $\pm$ 0.007 & 0.887 $\pm$ 0.007 \\
DTNE                   & 0.808 $\pm$ 0.007 & 0.782 $\pm$ 0.009 & 0.900 $\pm$ 0.006 & 0.891 $\pm$ 0.007 \\
\entropath              & 0.882 $\pm$ 0.014 & 0.854 $\pm$ 0.013 & 0.916 $\pm$ 0.008 & 0.906 $\pm$ 0.007 \\
\bottomrule
\end{tabular}
\end{table}

\begin{table}[t]
\centering
\small
\setlength{\tabcolsep}{6pt}
\caption{Embedding-level metrics on the two main-text manifolds.
  Trustworthiness ($\mathcal{T}$) and continuity ($\mathcal{C}$) are
  protocol-independent; the embedded geodesic correlations (row-wise
  Spearman, Pearson of the 2D layout) use the \emph{analytic} geodesic
  for the Swiss roll and the shortest-path geodesic for the sparse tree
  (no closed form). \entropath leads trustworthiness and embedded geodesic
  correlation on the non-uniform Swiss roll, and the embedded correlations
  on the sparse tree. \textbf{Bold} = best, \underline{underlined} =
  second-best; ties for best are shown by two bold entries.}
\label{tab:emb-main}
\begin{tabular}{l cccc cccc}
\toprule
 & \multicolumn{4}{c}{Swiss roll (non-unif., analytic)}
 & \multicolumn{4}{c}{Sparse tree (shortest-path)} \\
\cmidrule(lr){2-5}\cmidrule(lr){6-9}
Method         & $\mathcal{T}$ & $\mathcal{C}$ & Sp.\ & Pe.\
               & $\mathcal{T}$ & $\mathcal{C}$ & Sp.\ & Pe.\ \\
\midrule
Diffusion Maps & \underline{0.957} & 0.976 & 0.725 & 0.726
               & 0.980 & 0.991 & 0.767 & 0.758 \\
PHATE          & 0.945 & 0.984 & 0.741 & 0.737
               & 0.980 & 0.989 & 0.404 & 0.500 \\
HeatGeo        & 0.943 & 0.983 & 0.714 & 0.707
               & \textbf{0.992} & \textbf{0.993} & 0.770 & 0.808 \\
DTNE           & 0.949 & \textbf{0.987} & \underline{0.805} & \underline{0.789}
               & 0.989 & 0.992 & \underline{0.859} & \underline{0.874} \\
\entropath      & \textbf{0.981} & \underline{0.985} & \textbf{0.875} & \textbf{0.855}
               & \underline{0.991} & \textbf{0.993} & \textbf{0.869} & \textbf{0.884} \\
\bottomrule
\end{tabular}
\end{table}

\begin{figure}[t]
\centering
\includegraphics[width=\textwidth]{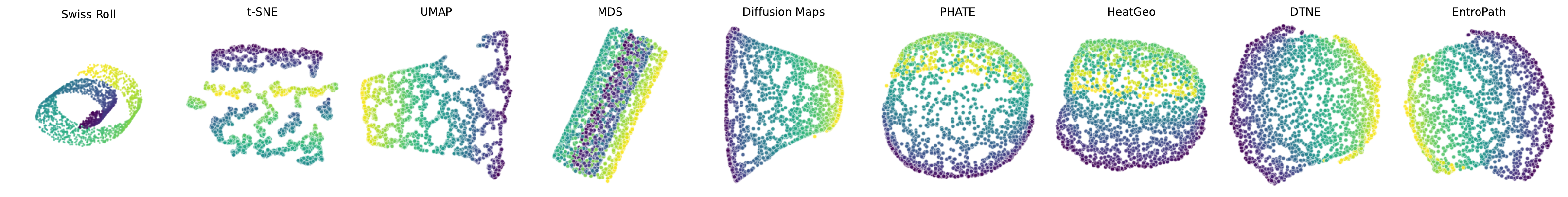}\\[0.3em]
\includegraphics[width=\textwidth]{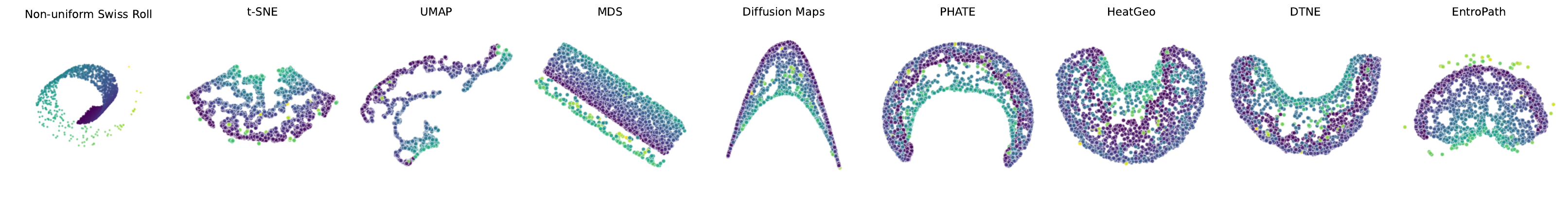}
\caption{Representative 2D embeddings of (top) uniform and (bottom)
  non-uniform Swiss roll (Beta$(1,4)$ density gradient), coloured by the
  angular parameter $t$. Local methods (UMAP, $t$-SNE) fragment the
  manifold; linear MDS preserves ambient Euclidean distance but fails to
  unroll. Under non-uniform sampling, \entropath maintains coherent global
  structure while holding sparsely sampled points apart from the main arc
  rather than smoothing them into it, whereas PHATE and HeatGeo show
  diluted local geometry near sparse regions.}
\label{fig:swiss-roll}
\end{figure}

\begin{figure}[t]
\centering
\includegraphics[width=\textwidth]{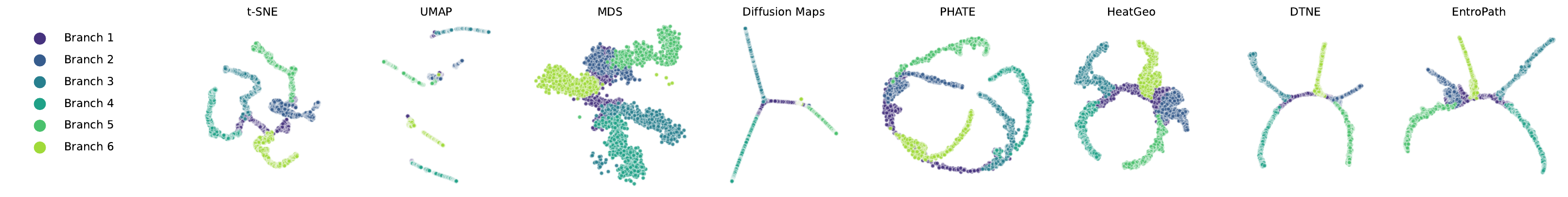}
\caption{Embeddings of the sparse artificial tree, coloured by branch identity.
The diffusion-based methods (PHATE, HeatGeo, DTNE, \entropath) recover the
branching topology, while $t$-SNE and UMAP fragment it and Diffusion Maps
collapses it into fewer arms. Among them the along-branch point spacing
varies: HeatGeo spreads points most heterogeneously, DTNE most uniformly, and
\entropath stays close to DTNE's clean branch separation while still resolving
large point aggregations rather than smoothing them into an even spread.}
\label{fig:tree}
\end{figure}



\subsection{Single-Cell Developmental Data}
\label{sec:singlecell}

Single-cell RNA-seq datasets capture continuous developmental processes, in
which cells progress through intermediate states along one or more lineages,
and are a standard benchmark for manifold learning. We evaluate \entropath against
PCA, t-SNE, UMAP, PHATE~\citep{Moon2019}, HeatGeo~\citep{Huguet2023}, and
DTNE~\citep{Wei2025} on six established datasets with known developmental
structure: Paul15 (mouse myeloid progenitors)~\citep{Paul2015}, Nestorowa (mouse
haematopoiesis)~\citep{Nestorowa2016}, Pancreas (endocrinogenesis)~\citep{Pancreas2019},
Lymphoid~\citep{Lymphoid2019}, Embryoid Body~\citep{Moon2019}, and the Arabidopsis
root atlas ($1.1\times10^{5}$ cells)~\citep{Shahan2022}. Baselines use each
method's per-dataset settings as published by DTNE; \entropath uses a single
fixed configuration ($k=15$, $\alpha$-decay kernel with decay $40$, k-means
landmarks) on every dataset, with no per-dataset tuning. HeatGeo is omitted on
the root atlas, where it is computationally infeasible at this scale.
Full hyperparameters and the remaining metrics (DEMaP standard deviations, ARI, NMI, continuity, runtime) are given in Appendix~\ref{app:singlecell}.

\paragraph{Neighbourhood size and the role of the adaptive kernel.}
A practical consequence of the adaptive affinity is that \entropath operates at a
small, fixed neighbourhood size where diffusion methods built on hard
$k$-nearest-neighbour graphs require substantially larger neighbourhoods. In a
hard-kNN construction each point connects only to its $k$ nearest neighbours, so
the graph must be dense enough for the random walk to mix across the manifold; on large datasets this pushes $k$ substantially higher---for the
$1.1\times10^{5}$-cell root atlas, DTNE adopts $k=800$---which densifies the
transition operator and makes its powers progressively more expensive to form. \entropath instead derives
a \emph{per-point} bandwidth $\sigma_i$ from the distance to the $k$-th neighbour
and applies the $\alpha$-decay kernel $\exp\!\big(-(d_{ij}/\sigma_i)^{\alpha}\big)$,
whose tails extend smoothly beyond the $k$ nearest points. Effective connectivity
is therefore governed by the kernel rather than the edge count, and a small $k$
already yields a well-connected, well-mixing diffusion. Two mechanisms act in
concert: the adaptive kernel decouples connectivity from $k$, keeping the
affinity matrix sparse, while the landmark construction decouples cost from $n$
by restricting the spectral computation to $m \ll n$ anchors. 

On the root atlas this computational advantage is substantial: \entropath's
landmark construction and adaptive kernel keep the spectral computation at
$m\ll n$ anchors and a small fixed $k$, where hard-kNN diffusion methods
densify (DTNE uses $k=800$; \Cref{app:singlecell}). \entropath embeds the
$1.1\times10^{5}$ cells in approximately $15$\,s while attaining comparable
geometric fidelity (\Cref{tab:sc-geometry}).



\paragraph{Low-dimensional visualization.}
Figure~\ref{fig:sc-embeddings} shows the two-dimensional embeddings produced by
each method. We quantify embedding quality using two complementary metrics
(Table~\ref{tab:sc-geometry}): \emph{DEMaP}~\citep{Moon2019}, the Spearman
correlation between geodesic distances in the original space and Euclidean
distances in the embedding ($k=30$, following the PHATE convention), which
measures preservation of \emph{global} geometry; and \emph{trustworthiness},
which measures preservation of \emph{local} neighbourhoods.\footnote{DEMaP
settings follow DTNE's per-dataset protocol (\Cref{app:singlecell}): subsample
size $500$ (Paul15, Nestorowa, Pancreas) or $2{,}000$ (Lymphoid, Embryoid Body,
root atlas), $50$ repetitions, and DTNE's per-dataset geodesic $k$-NN. For
Embryoid Body and Lymphoid the DEMaP reference is the raw representation
(sqrt-transformed expression / full LSI) rather than the PCA-reduced input,
matching DTNE's published notebooks.}
These metrics reveal a sharp trade-off on the most challenging datasets. On
Lymphoid and Embryoid Body, where nonlinear methods lose substantial geodesic
structure, the highest DEMaP is achieved by the \emph{linear} PCA projection.
However, PCA fails to resolve cell states (its clustering score is essentially
zero on Lymphoid; Appendix~\ref{app:singlecell}), preserving coarse global
geometry at the expense of lineage separation. In contrast, the local-neighbourhood
methods (t-SNE, UMAP) achieve high trustworthiness but collapse global distances
(UMAP DEMaP $0.28$ and $0.18$ on the two datasets, respectively).
\entropath occupies the middle ground: it attains the best DEMaP among nonlinear
methods on Paul15, Lymphoid, and Embryoid Body, while maintaining trustworthiness
comparable to the diffusion baselines. This balance is particularly valuable for
trajectory analysis: on the hardest datasets \entropath achieves the strongest
nonlinear DEMaP, and on Paul15 the strongest DEMaP overall,
while producing embeddings in which
the branching structure remains visible and usable for downstream ordering. On the
cleaner Nestorowa, Pancreas, and root-atlas manifolds, where the trade-off is
milder and most methods score well, \entropath performs comparably to the strong
diffusion baselines.

\begin{figure}[t]
\centering
\includegraphics[width=\textwidth]{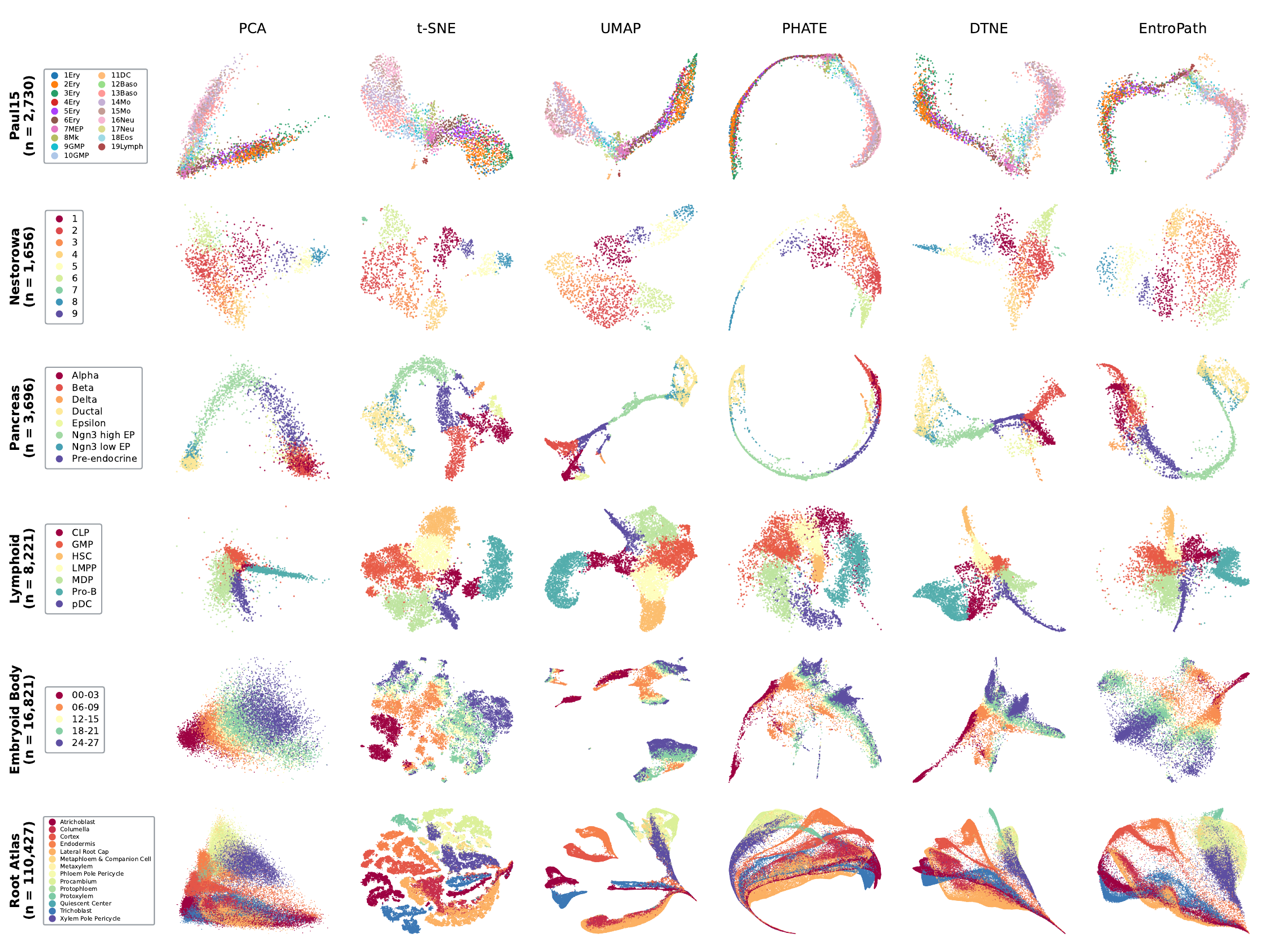}
\caption{Two-dimensional embeddings of six single-cell datasets (rows) produced
by the compared methods (columns, \entropath last; HeatGeo omitted on the root
atlas). Datasets, top to bottom: Paul15 ($n=2{,}730$), Nestorowa ($n=1{,}656$),
Pancreas ($n=3{,}696$), Lymphoid ($n=8{,}221$), Embryoid Body ($n=16{,}821$), and
root atlas ($n=110{,}427$); points are coloured by the annotated cell type or
cluster in the per-row legend. PCA leaves lineage structure largely buried in
noise; t-SNE and UMAP recover compact clusters but break the connections between
them. The diffusion-based methods (PHATE, DTNE, \entropath) render the lineages as
continuous, branching structures, with \entropath retaining visible separation
between cell states. Aspect ratio is not preserved; colours are assigned by a
Spectral map over sorted cell types. Quantitative scores are in
Table~\ref{tab:sc-geometry} and Appendix~\ref{app:singlecell}.}
\label{fig:sc-embeddings}
\end{figure}

\begin{table}[t]
\centering
\caption{Embedding geometry on six single-cell datasets: DEMaP ($k=30$, mean)
and trustworthiness; higher is better for both. \textbf{Bold}: best overall per
dataset. \underline{Underline}: best nonlinear method (excludes PCA). HeatGeo is
infeasible on the root atlas (---). Standard deviations and the remaining metrics
(ARI, NMI, continuity, runtime) are in Appendix~\ref{app:singlecell}.}
\label{tab:sc-geometry}
\setlength{\tabcolsep}{4pt}
\begin{tabular}{l rrrrrr}
\toprule
Method & Paul15 & Nestorowa & Pancreas & Lymphoid & Embryoid & Root \\
\midrule
\multicolumn{7}{l}{\emph{DEMaP} ($k=30$, mean)} \\
\midrule
PCA       & 0.875 & \textbf{0.925} & 0.878 & \textbf{0.615} & \textbf{0.486} & ---                       \\
t-SNE     & 0.912 & 0.906 & 0.786 & 0.368 & 0.334 & 0.716                                                 \\
UMAP      & 0.901 & \underline{0.919} & \textbf{\underline{0.936}} & 0.284 & 0.176 & 0.690                \\
PHATE     & 0.885 & 0.890 & 0.901 & 0.338 & 0.327 & \textbf{\underline{0.886}}                            \\
HeatGeo   & 0.760 & 0.862 & 0.869 & 0.173 & 0.315 & ---                                                   \\
DTNE      & 0.886 & 0.906 & 0.932 & 0.254 & 0.364 & 0.850                                                 \\
\entropath & \textbf{\underline{0.946}} & 0.908 & 0.895 & \underline{0.493} & \underline{0.463} & 0.843    \\
\midrule
\multicolumn{7}{l}{\emph{Trustworthiness}} \\
\midrule
PCA       & 0.919 & 0.913 & 0.907 & 0.756 & 0.809 & ---                                                   \\
t-SNE     & \textbf{\underline{0.940}} & \textbf{\underline{0.985}} & \textbf{\underline{0.988}} & \textbf{\underline{0.922}} & \textbf{\underline{0.984}} & 0.992 \\
UMAP      & 0.929 & 0.963 & 0.980 & 0.860 & 0.952 & \textbf{\underline{0.994}}                            \\
PHATE     & 0.933 & 0.932 & 0.954 & 0.817 & 0.910 & 0.968                                                 \\
HeatGeo   & 0.926 & 0.932 & 0.959 & 0.549 & 0.860 & ---                                                   \\
DTNE      & 0.926 & 0.930 & 0.961 & 0.836 & 0.890 & 0.978                                                 \\
\entropath & 0.920 & 0.946 & 0.958 & 0.812 & 0.877 & 0.980                                                 \\
\bottomrule
\end{tabular}
\end{table}

\paragraph{Reference robustness on Embryoid Body.}
On Embryoid Body, EntroPath's DEMaP advantage over the diffusion baselines is
robust to the choice of high-dimensional reference: it ranks first among the
diffusion methods under both the strict raw-expression reference used in the
main text (EntroPath $0.463$ vs.\ DTNE $0.364$, PHATE $0.327$, HeatGeo $0.315$)
and the PCA-input reference (EntroPath $0.823$ vs.\ DTNE $0.743$, PHATE $0.728$,
HeatGeo $0.714$). The reduced-input reference yields uniformly higher absolute
scores; we report the stricter raw-expression protocol for comparability with
DTNE's published results.

\paragraph{Pseudotime inference.}
\entropath also induces a pseudotime ordering as the free-energy dissimilarity
from a designated root cell. Across the developmental datasets this ordering is
strongly concordant with DTNE's---the two share the same free-energy dissimilarity
and differ only in the underlying walk (maximum entropy versus standard)---and
competitive with reference orderings, though it does not lead this metric: the
ordering is sensitive to the diffusion power, and the diffusion depth selected for
visualization need not coincide with the optimal depth for pseudotime inference.
We report the full pseudotime
analysis, including this limitation and a per-dataset comparison with DTNE, in
Appendix~\ref{app:pseudotime}, and illustrate it on Paul15 in
Figure~\ref{fig:paul15-pseudotime}.

\subsection{Ablation Study}
\label{sec:ablation}

We ablate the key design choices of \entropath to isolate the contribution of each
component. Unless stated otherwise, experiments use the Swiss roll
($N=2{,}000$, 30 seeds) with the same graph construction as
Section~\ref{sec:exp-synthetic}.

\subsubsection{Effect of diffusion depth k.}
\label{par:ablation-k}
The depth~$k$ is the power applied to the spectrally normalized affinity $\tilde A$
in $D_{ij}=-\log(\tilde A^{\,k})_{ij}$, and is distinct from the connectivity
$k_{\mathrm{NN}}{=}15$: $k_{\mathrm{NN}}$ sets the locality of the kernel, whereas
$k$ controls how far the maximum entropy walk integrates over the manifold. Sweeping
$k\in[1,500]$ and scoring the \entropath dissimilarity against the analytic
arc-length geodesic with the row-wise Spearman metric of
\Cref{tab:swiss-roll-l1}, the fidelity $\rho(k)$ traces the three regimes of
\Cref{sec:regimes} (\Cref{fig:ablation-k-main}): at small~$k$ the walk has not
bridged the manifold and the distance is local ($\rho{=}0.15$ at $k{=}1$); $\rho$
rises through the short-time (Varadhan) regime to a maximum of $\rho{=}0.88$ at
$k{=}65$; and it decays thereafter as geometric structure is lost. Fidelity exceeds
the Isomap geodesic baseline ($\rho{=}0.83$) across a broad plateau ($k\in[20,100]$),
so the gain is not tied to one fortunate depth. Crucially, the \emph{unsupervised}
selector \entropath uses by default---the knee of the von Neumann entropy of the MERW
operator---recovers this optimum without the geodesic: the median selected depth
$k{=}57$ lies inside the peak band ($k\in[32,85]$, within one s.d.\ of the maximum)
and attains $\rho$ within $0.01$ of the grid optimum. The entropy criterion
identifies the geodesic-faithful regime from the spectrum alone---the property the
single-cell experiments rely on, where no ground-truth geodesic exists to tune
against.

As $k\to\infty$, \Cref{prop:largek} predicts the dissimilarity collapses onto the
stationary form $-\tfrac12\log(\pi_i\pi_j)$, with $\pi_i\propto\psi_i^2$ the MERW
stationary distribution. \Cref{fig:ablation-k-largek} confirms this empirically: the
correlation between $D(k)$ and $-\log(\psi_i\psi_j)$ is near zero throughout the
geodesic regime, crosses zero near the fidelity peak ($k\approx65$), and rises
monotonically thereafter, reaching $\rho{=}0.72$ at $k{=}500$ and still increasing.
The rate is governed by the spectral gap of $\tilde A$; the direction of convergence
matches the proposition and identifies the large-$k$ decay in
\Cref{fig:ablation-k-main} as loss of geometric information to the stationary
measure rather than numerical degradation.

\begin{figure}[t]
  \centering
  \begin{subfigure}{0.45\linewidth}
    \includegraphics[width=\linewidth]{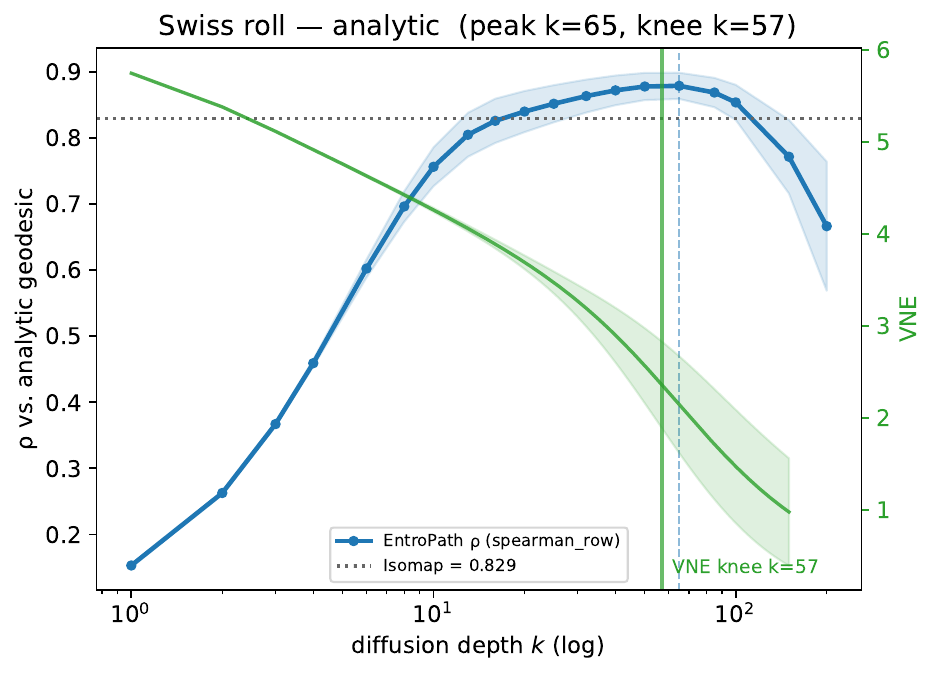}
    \caption{}
    \label{fig:ablation-k-main}
  \end{subfigure}\hfill
  \begin{subfigure}{0.5\linewidth}
    \includegraphics[width=\linewidth]{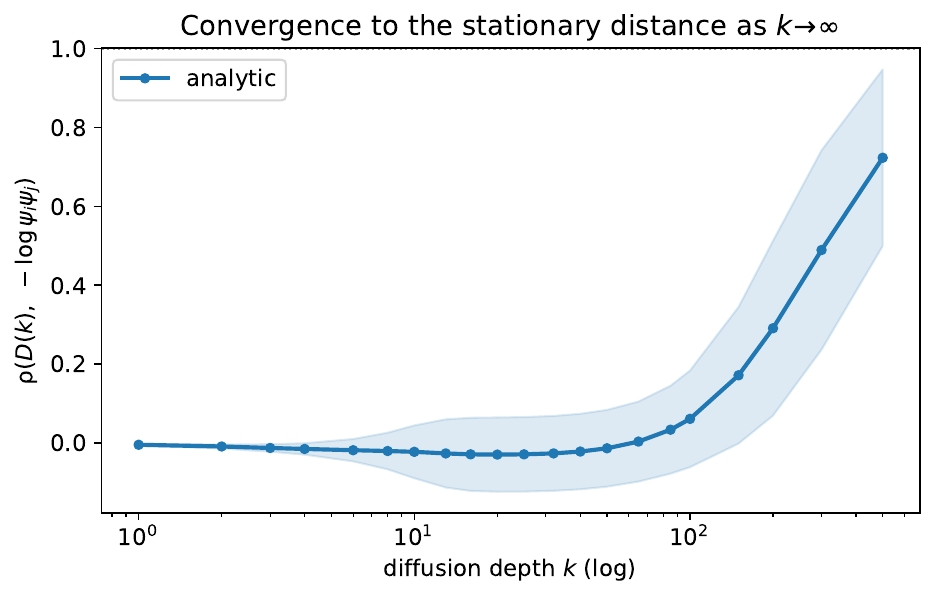}
    \caption{}
    \label{fig:ablation-k-largek}
  \end{subfigure}
  \caption{\textbf{Diffusion-depth ablation on the Swiss roll}
    ($N=2{,}000$, 30 seeds; shaded bands $\pm1$ s.d.).
    \textbf{(a)} Geodesic fidelity $\rho(k)$ (blue, left axis) between the \entropath
    dissimilarity and the analytic geodesic, with the von Neumann entropy
    (green, right axis) and its selected knee (vertical line, $k=57$); the Isomap
    geodesic baseline is dotted. Fidelity peaks at $k=65$ ($\rho=0.88$) and the
    entropy knee falls inside the peak.
    \textbf{(b)} Correlation of $D(k)$ with the stationary distance
    $-\log(\psi_i\psi_j)$, near zero in the geodesic regime and rising
    monotonically toward the large-$k$ limit of \Cref{prop:largek}.}
  \label{fig:ablation-k}
\end{figure}


\section{Related Work}
\label{sec:related}
\paragraph{Spectral and graph-based embeddings.}
Spectral methods embed data using eigenfunctions of a graph operator built from
pairwise affinities. Laplacian Eigenmaps \citep{Belkin2003} use graph-Laplacian
eigenvectors, which converge to the eigenfunctions of the Laplace--Beltrami
operator on the underlying manifold; Locally Linear Embedding \citep{Roweis2000}
recovers a related geometry from local reconstruction weights. Both operate on
local $k$NN affinities without global normalisation and do not yield a pairwise
dissimilarity suitable for metric MDS.

\paragraph{Diffusion-based methods.}
Diffusion Maps \citep{Coifman2006} normalise affinities into a Markov operator
and embed via its eigenspace at diffusion time $t$, with coordinates given by
diffusion distances that average over all paths. PHATE \citep{Moon2019} compares
log-diffusion profiles through a potential distance capturing local and global
structure, and HeatGeo \citep{Huguet2023} defines a heat-geodesic dissimilarity
$-\log h_t(x,y)$ from the symmetric heat kernel that converges to geodesic
distance as $t\to0$ via Varadhan's formula. Both PHATE and HeatGeo compare
\emph{entire} diffusion profiles (a triplet distance), pooling per-entry noise
across points. \entropath belongs to this family but replaces the
degree-normalised walk with MERW and the symmetric heat kernel with the
Schr\"odinger heat kernel $e^{-tH}$, and reads off the single pairwise
log-amplitude rather than a profile comparison; the resulting pairwise
free-energy dissimilarity admits an explicit diffusion-map (Gram) representation
(\Cref{prop:gram}).

\paragraph{Geodesic and shortest-path methods.}
Isomap \citep{Tenenbaum2000} estimates geodesics by all-pairs shortest paths on
the $k$NN graph and embeds via classical MDS; it is faithful on convex,
densely sampled manifolds but costs $O(n^3)$ and is sensitive to noise and
shortcut edges. Landmark Isomap \citep{Silva2002} reduces cost via landmark
geodesics. \entropath instead aggregates over path ensembles, reducing sensitivity to
shortcut edges (\Cref{app:ablation-noise}) while targeting the same geodesic
geometry (\Cref{thm:geodesic}).

\paragraph{Neighbour-based methods.}
$t$-SNE \citep{vanderMaaten2008} and UMAP \citep{McInnes2018} preserve local
neighbourhood and cluster structure through probabilistic or topological
objectives, often at the expense of global geometry, so inter-cluster distances
may not be meaningful. \entropath instead preserves connectivity across scales
through path-ensemble aggregation.

\paragraph{Maximum entropy random walk.}
MERW \citep{Burda2009} is the random walk maximising entropy over entire paths,
with stationary distribution $\pi_i=\psi_i^2$ concentrating on globally central
rather than high-degree nodes. Its mixing and hitting times \citep{Ochab2013} and
its links to Markov stability and community detection \citep{Delvenne2010} have
been studied on networks, and PAN \citep{Ma2020} uses a path-integral MERW
weighting as a learnable convolution operator for graph neural networks. Walk
sums over powers of the adjacency have a long history in spectral graph theory;
an influential statistical-mechanical reading is the communicability and
free-energy interpretation of \citet{Estrada2007}. To our knowledge, \entropath is
the first \emph{manifold-learning} method built on MERW, and the first to use the
MERW walk sum as a free-energy dissimilarity rather than a centrality or
propagation operator.

\paragraph{Relation to DTNE.}
Closest to our work, DTNE \citep{Wei2025} uses the same
$-2\log(\hat v_i\cdot\hat v_j)$ dissimilarity, derived via the kernel trick on
Bhattacharyya overlaps of personalised PageRank vectors. The two are
structurally identical; what differs is the walk (standard random walk
vs.\ MERW) and the diffusion horizon (an infinite geometric PageRank mixture
vs.\ our finite $k$-step ensemble). \Cref{sec:dtne-compare} gives the full
correspondence.

\paragraph{Trajectory and pseudotime methods.}
In single-cell trajectory inference, pseudotime is typically derived from a
reduced representation \citep{Trapnell2014}, from diffusion distance to a root
cell \citep{Haghverdi2016}, or from RNA velocity \citep{LaManno2018}. Our
construction (\Cref{app:pseudotime}) follows the diffusion-pseudotime paradigm
but uses the MERW diffusion potential, inheriting the geodesic approximation
of \Cref{thm:geodesic}.

\section{Conclusion}
\label{sec:conclusion}

We introduced \entropath, a manifold embedding method based on the
maximum entropy random walk and a free-energy dissimilarity derived from path
ensembles. Unlike diffusion-profile approaches that compare entire transition
distributions, \entropath assigns dissimilarities directly from the
aggregate weight of finite-horizon paths between pairs of vertices,
yielding a dissimilarity with a statistical-mechanical interpretation and a
geometric guarantee: under spectral convergence of the graph to a manifold,
with limiting Schr\"odinger operator $H=-\Delta+V$, the symmetrised
free-energy dissimilarity converges to squared geodesic distance in the
short-time limit (\Cref{thm:geodesic}). The same Schr\"odinger structure
also yields an exact Gram factorisation of the symmetrised kernel
(\Cref{prop:gram}), absent in comparable diffusion methods, linking the
geometric and kernel perspectives of the method.

The resulting geometry captures both local and global structure: at small
diffusion depth $k$ the dissimilarity approximates manifold geodesics, and at
larger $k$ it reflects the global organisation of the graph, with the
zero-temperature limit recovering the classical graph geodesic. Because
MERW weights paths by global importance rather than local degree,
bottlenecks and sparsely sampled transition regions act as effective energy
barriers (\Cref{rem:bottleneck}), improving the separation of weakly
connected regions and conferring robustness to heterogeneous sampling---
properties especially relevant in developmental single-cell data where
branching states are underrepresented. Computationally, the method admits
scalable landmark approximation and a diffusion-potential pseudotime
construction (\Cref{sec:scalable}), both inheriting the short-time geodesic
approximation. Together these place \entropath within the diffusion-distance
family of \Cref{sec:related}, sharing the heat-kernel foundation of HeatGeo
and the dissimilarity form of DTNE while differing in its global-entropy walk,
its finite diffusion horizon, and the Schr\"odinger and Gram structure from
which its geodesic guarantee follows.

\paragraph{Limitations and future work.}
\entropath's entropy-based depth selection becomes noisier on short branches. A
natural extension is a continuous-time formulation: rather than the discrete
powers $\tilde A^{\,k}$, one could approximate the Schr\"odinger heat semigroup
$e^{-tH}$ for $H=-\Delta_\mathcal M+V$ at continuous diffusion time $t$, e.g.\ by a
Chebyshev expansion, analogous to HeatGeo's treatment of the graph heat kernel
$e^{-tL}$ \citep{Huguet2023}. Unlike that construction, the operator here retains
the MERW potential $V$, which our geodesic theorem identifies as the mechanism
underlying geodesic recovery (\Cref{thm:geodesic}); continuous diffusion time
would then replace the discrete depth $k$ and smooth its selection. A second
extension concerns directed affinities: the geometric-mean dissimilarity
\eqref{eq:D} is defined for any nonnegative affinity, and for asymmetric $A$ the
MERW operator remains well defined through the left and right Perron
eigenvectors, so $-\log\sqrt{T^k_{ij}T^k_{ji}}$ combines distinct forward and
backward path ensembles rather than collapsing to one---making \entropath a
candidate for \emph{directed} data such as RNA-velocity matrices. Finally, the
direct dissimilarity could be combined with HeatGeo's triplet (row-comparison)
regulariser.

\paragraph{Data and code availability.}
Code to reproduce all experiments is available at
\href{https://github.com/rpprzemek/entropath}{https://github.com/rpprzemek/entropath}.

The synthetic manifolds (Swiss roll, sphere, torus, Swiss hole) are generated
using \texttt{scikit-learn} and standard sampling
procedures; generation code is provided in our release. The branching-tree
benchmark uses PHATE's diffusion-limited aggregation generator
\texttt{phate.tree.gen\_dla} \citep{Moon2019}. The clustering benchmarks comprise Tree, PBMC
(\texttt{scanpy} \citep{Wolf2018}), and MNIST, with preprocessing implemented in our codebase.
The single-cell trajectory datasets
(Paul15, Nestorowa, Pancreas, Lymphoid, Embryoid Body, and the root-cell atlas),
together with their reference cell orderings and per-dataset root cells, are
obtained from public sources, including the DTNE release \citep{Wei2025}. Preprocessing and
evaluation follow the DTNE protocol where applicable.

\bibliographystyle{plainnat}
\bibliography{references}  

@article{Varadhan1967,
author = {Varadhan, S. R. S.},
title = {On the behavior of the fundamental solution of the heat equation with variable coefficients},
journal = {Communications on Pure and Applied Mathematics},
volume = {20},
number = {2},
pages = {431-455},
doi = {https://doi.org/10.1002/cpa.3160200210},
url = {https://onlinelibrary.wiley.com/doi/abs/10.1002/cpa.3160200210},
eprint = {https://onlinelibrary.wiley.com/doi/pdf/10.1002/cpa.3160200210},
year = {1967}
}

@article{Coifman2006,
title = {Diffusion maps},
journal = {Applied and Computational Harmonic Analysis},
volume = {21},
number = {1},
pages = {5-30},
year = {2006},
note = {Special Issue: Diffusion Maps and Wavelets},
issn = {1063-5203},
doi = {https://doi.org/10.1016/j.acha.2006.04.006},
url = {https://www.sciencedirect.com/science/article/pii/S1063520306000546},
author = {Ronald R. Coifman and Stéphane Lafon}
}

@article {Moon2019,
	Title = {Visualizing structure and transitions in high-dimensional biological data},
	Author = {Moon, Kevin R and van Dijk, David and Wang, Zheng and Gigante, Scott and Burkhardt, Daniel B and Chen, William S and Yim, Kristina and Elzen, Antonia van den and Hirn, Matthew J and Coifman, Ronald R and Ivanova, Natalia B and Wolf, Guy and Krishnaswamy, Smita},
	DOI = {10.1038/s41587-019-0336-3},
	Number = {12},
	Volume = {37},
	Month = {December},
	Year = {2019},
	Journal = {Nature Biotechnology},
	ISSN = {1087-0156},
	Pages = {1482--1492},
	URL = {https://europepmc.org/articles/PMC7073148},
}

@article{Tenenbaum2000,
author = {Tenenbaum, Joshua B. and de Silva, Vin   and Langford, John C.},
title = {A Global Geometric Framework for Nonlinear Dimensionality Reduction},
journal = {Science},
volume = {290},
number = {5500},
pages = {2319-2323},
year = {2000},
doi = {10.1126/science.290.5500.2319},
URL = {https://www.science.org/doi/abs/10.1126/science.290.5500.2319},
eprint = {https://www.science.org/doi/pdf/10.1126/science.290.5500.2319}
}

@article{Balasubramanian2002,
author = {Balasubramanian, Mukund and Schwartz, Eric L.},
title = {The Isomap Algorithm and Topological Stability},
journal = {Science},
volume = {295},
number = {5552},
pages = {7-7},
year = {2002},
doi = {10.1126/science.295.5552.7a},
URL = {https://www.science.org/doi/abs/10.1126/science.295.5552.7a},
eprint = {https://www.science.org/doi/pdf/10.1126/science.295.5552.7a}
}

@inproceedings{Huguet2023,
 author = {Huguet, Guillaume and Tong, Alexander and De Brouwer, Edward and Zhang, Yanlei and Wolf, Guy and Adelstein, Ian and Krishnaswamy, Smita},
 booktitle = {Advances in Neural Information Processing Systems},
 editor = {A. Oh and T. Naumann and A. Globerson and K. Saenko and M. Hardt and S. Levine},
 pages = {6986--7016},
 publisher = {Curran Associates, Inc.},
 title = {A Heat Diffusion Perspective on Geodesic Preserving Dimensionality Reduction},
 url = {https://proceedings.neurips.cc/paper_files/paper/2023/file/16063a1c0f0cddd4894585cf44cebb2c-Paper-Conference.pdf},
 volume = {36},
 year = {2023}
}

@article{McInnes2018,
author = {McInnes, Leland and Healy, John and Melville, James},
journal={arXiv},
pages = {},
title = {UMAP: Uniform Manifold Approximation and Projection for Dimension Reduction},
year = {2018},
month = {02},
doi = {10.48550/arXiv.1802.03426}
}

@article{Wei2025,
author = {Wei, Jiangyong and Zhang, Bin and Wang, Qiu and Zhou, Tianshou and Tian, Tianhai and Chen, Luonan},
title = {Diffusive topology preserving manifold distances for single-cell data analysis},
journal = {Proceedings of the National Academy of Sciences},
volume = {122},
number = {4},
pages = {e2404860121},
year = {2025},
doi = {10.1073/pnas.2404860121},
URL = {https://www.pnas.org/doi/abs/10.1073/pnas.2404860121},
eprint = {https://www.pnas.org/doi/pdf/10.1073/pnas.2404860121}
}

@article{Burda2009,
  title = {Localization of the Maximal Entropy Random Walk},
  author = {Burda, Z. and Duda, J. and Luck, J. M. and Waclaw, B.},
  journal = {Phys. Rev. Lett.},
  volume = {102},
  issue = {16},
  pages = {160602},
  numpages = {4},
  year = {2009},
  month = {Apr},
  publisher = {American Physical Society},
  doi = {10.1103/PhysRevLett.102.160602},
  url = {https://link.aps.org/doi/10.1103/PhysRevLett.102.160602}
}

@inproceedings{Silva2002,
 author = {Silva, Vin and Tenenbaum, Joshua},
 booktitle = {Advances in Neural Information Processing Systems},
 editor = {S. Becker and S. Thrun and K. Obermayer},
 pages = {},
 publisher = {MIT Press},
 title = {Global Versus Local Methods in Nonlinear Dimensionality Reduction},
 url = {https://proceedings.neurips.cc/paper_files/paper/2002/file/5d6646aad9bcc0be55b2c82f69750387-Paper.pdf},
 volume = {15},
 year = {2002}
}

@article{vanderMaaten2008,
  author  = {van der Maaten, Laurens and Hinton, Geoffrey},
  title   = {Visualizing Data using t-SNE},
  journal = {Journal of Machine Learning Research},
  year    = {2008},
  volume  = {9},
  number  = {86},
  pages   = {2579--2605},
  url     = {http://jmlr.org/papers/v9/vandermaaten08a.html}
}

@Article{Ochab2013,
author={Ochab, J. K. and Burda, Z.},
title={Maximal entropy random walk in community detection},
journal={The European Physical Journal Special Topics},
year={2013},
month={Jan},
day={01},
volume={216},
number={1},
pages={73-81},
issn={1951-6401},
doi={10.1140/epjst/e2013-01730-6},
url={https://doi.org/10.1140/epjst/e2013-01730-6}
}

@article{Delvenne2010,
author = {Delvenne, J.-C. and Yaliraki, S. N. and Barahona, M.},
title = {Stability of graph communities across time scales},
journal = {Proceedings of the National Academy of Sciences},
volume = {107},
number = {29},
pages = {12755-12760},
year = {2010},
doi = {10.1073/pnas.0903215107},
URL = {https://www.pnas.org/doi/abs/10.1073/pnas.0903215107},
eprint = {https://www.pnas.org/doi/pdf/10.1073/pnas.0903215107}
}

@Article{Trapnell2014,
author={Trapnell, Cole
and Cacchiarelli, Davide
and Grimsby, Jonna
and Pokharel, Prapti
and Li, Shuqiang
and Morse, Michael
and Lennon, Niall J.
and Livak, Kenneth J.
and Mikkelsen, Tarjei S.
and Rinn, John L.},
title={The dynamics and regulators of cell fate decisions are revealed by pseudotemporal ordering of single cells},
journal={Nature Biotechnology},
year={2014},
month={Apr},
day={01},
volume={32},
number={4},
pages={381-386},
issn={1546-1696},
doi={10.1038/nbt.2859},
url={https://doi.org/10.1038/nbt.2859}
}

@Article{Haghverdi2016,
author={Haghverdi, Laleh
and B{\"u}ttner, Maren
and Wolf, F. Alexander
and Buettner, Florian
and Theis, Fabian J.},
title={Diffusion pseudotime robustly reconstructs lineage branching},
journal={Nature Methods},
year={2016},
month={Oct},
day={01},
volume={13},
number={10},
pages={845-848},
issn={1548-7105},
doi={10.1038/nmeth.3971},
url={https://doi.org/10.1038/nmeth.3971}
}

@Article{LaManno2018,
author={La Manno, Gioele
and Soldatov, Ruslan
and Zeisel, Amit
and Braun, Emelie
and Hochgerner, Hannah
and Petukhov, Viktor
and Lidschreiber, Katja
and Kastriti, Maria E.
and L{\"o}nnerberg, Peter
and Furlan, Alessandro
and Fan, Jean
and Borm, Lars E.
and Liu, Zehua
and van Bruggen, David
and Guo, Jimin
and He, Xiaoling
and Barker, Roger
and Sundstr{\"o}m, Erik
and Castelo-Branco, Gon{\c{c}}alo
and Cramer, Patrick
and Adameyko, Igor
and Linnarsson, Sten
and Kharchenko, Peter V.},
title={RNA velocity of single cells},
journal={Nature},
year={2018},
month={Aug},
day={01},
volume={560},
number={7719},
pages={494-498},
issn={1476-4687},
doi={10.1038/s41586-018-0414-6},
url={https://doi.org/10.1038/s41586-018-0414-6}
}

@inproceedings{ZelnikManor2004,
 author = {Zelnik-Manor, Lihi and Perona, Pietro},
 booktitle = {Advances in Neural Information Processing Systems},
 editor = {L. Saul and Y. Weiss and L. Bottou},
 pages = {},
 publisher = {MIT Press},
 title = {Self-Tuning Spectral Clustering},
 url = {https://proceedings.neurips.cc/paper_files/paper/2004/file/40173ea48d9567f1f393b20c855bb40b-Paper.pdf},
 volume = {17},
 year = {2004}
}

@article{Roweis2000,
author = {Sam T. Roweis  and Lawrence K. Saul },
title = {Nonlinear Dimensionality Reduction by Locally Linear Embedding},
journal = {Science},
volume = {290},
number = {5500},
pages = {2323-2326},
year = {2000},
doi = {10.1126/science.290.5500.2323},
URL = {https://www.science.org/doi/abs/10.1126/science.290.5500.2323},
eprint = {https://www.science.org/doi/pdf/10.1126/science.290.5500.2323}
}

@article{Belkin2003,
    author = {Belkin, Mikhail and Niyogi, Partha},
    title = {Laplacian Eigenmaps for Dimensionality Reduction and Data Representation},
    journal = {Neural Computation},
    volume = {15},
    number = {6},
    pages = {1373-1396},
    year = {2003},
    month = {06},
    issn = {0899-7667},
    doi = {10.1162/089976603321780317},
    url = {https://doi.org/10.1162/089976603321780317},
    eprint = {https://direct.mit.edu/neco/article-pdf/15/6/1373/815527/089976603321780317.pdf},
}

@incollection{Belkin2007,
    author = {Belkin, Mikhail and Niyogi, Partha},
    editor = {Sch\"olkopf, Bernhard and Platt, John and Hofmann, Thomas},
    isbn = {9780262256919},
    title = {Convergence of Laplacian Eigenmaps},
    booktitle = {Advances in Neural Information Processing Systems 19: Proceedings of the 2006 Conference},
    publisher = {The MIT Press},
    year = {2007},
    month = {09},
    doi = {10.7551/mitpress/7503.003.0021},
    url = {https://doi.org/10.7551/mitpress/7503.003.0021},
    eprint = {https://direct.mit.edu/book/chapter-pdf/2288838/9780262256919_cap.pdf},
}

@article{Venna2006,
title = {Local multidimensional scaling},
journal = {Neural Networks},
volume = {19},
number = {6},
pages = {889-899},
year = {2006},
note = {Advances in Self Organising Maps - WSOM’05},
issn = {0893-6080},
doi = {https://doi.org/10.1016/j.neunet.2006.05.014},
url = {https://www.sciencedirect.com/science/article/pii/S0893608006000724},
author = {Jarkko Venna and Samuel Kaski}
}

@Article{Wolf2018,
author={Wolf, F. Alexander
and Angerer, Philipp
and Theis, Fabian J.},
title={SCANPY: large-scale single-cell gene expression data analysis},
journal={Genome Biology},
year={2018},
month={Feb},
day={06},
volume={19},
number={1},
pages={15},
issn={1474-760X},
doi={10.1186/s13059-017-1382-0},
url={https://doi.org/10.1186/s13059-017-1382-0}
}

@article{Paul2015,
title = {Transcriptional Heterogeneity and Lineage Commitment in Myeloid Progenitors},
journal = {Cell},
volume = {163},
number = {7},
pages = {1663--1677},
year = {2015},
issn = {0092-8674},
doi = {https://doi.org/10.1016/j.cell.2015.11.013},
url = {https://www.sciencedirect.com/science/article/pii/S0092867415014932},
author = {Paul, Franziska and Arkin, Ya'ara and Giladi, Amir and Jaitin, Diego Adhemar and Kenigsberg, Ephraim and Keren-Shaul, Hadas and Winter, Deborah and Lara-Astiaso, David and Gury, Meital and Weiner, Assaf and David, Eyal and Cohen, Nadav and Lauridsen, Felicia Kathrine Bratt and Haas, Simon and Schlitzer, Andreas and Mildner, Alexander and Ginhoux, Florent and Jung, Steffen and Trumpp, Andreas and Porse, Bo Torben and Tanay, Amos and Amit, Ido}
}

@article{Nestorowa2016,
    author = {Nestorowa, Sonia and Hamey, Fiona K. and Pijuan Sala, Blanca and Diamanti, Evangelia and Shepherd, Mairi and Laurenti, Elisa and Wilson, Nicola K. and Kent, David G. and G\"ottgens, Berthold},
    title = {A single-cell resolution map of mouse hematopoietic stem and progenitor cell differentiation},
    journal = {Blood},
    volume = {128},
    number = {8},
    pages = {e20--e31},
    year = {2016},
    month = {08},
    issn = {0006-4971},
    doi = {10.1182/blood-2016-05-716480},
    url = {https://doi.org/10.1182/blood-2016-05-716480},
    eprint = {https://ashpublications.org/blood/article-pdf/128/8/e20/1465712/e20.pdf},
}

@article{Pancreas2019,
    author = {Bastidas-Ponce, Aim{\'e}e and Tritschler, Sophie and Dony, Leander and Scheibner, Katharina and Tarquis-Medina, Marta and Salinno, Ciro and Schirge, Silvia and Burtscher, Ingo and B{\"o}ttcher, Anika and Theis, Fabian J. and Lickert, Heiko and Bakhti, Mostafa and Klein, Allon and Treutlein, Barbara},
    title = {Comprehensive single cell mRNA profiling reveals a detailed roadmap for pancreatic endocrinogenesis},
    journal = {Development},
    volume = {146},
    number = {12},
    pages = {dev173849},
    year = {2019},
    month = {06},
    issn = {0950-1991},
    doi = {10.1242/dev.173849},
    url = {https://doi.org/10.1242/dev.173849},
    eprint = {https://journals.biologists.com/dev/article-pdf/146/12/dev173849/3480158/dev173849.pdf},
}

@article{Shahan2022,
  author={Shahan, Rachel and Hsu, Che-Wei and Nolan, Trevor M. and Cole, Benjamin J. and Taylor, Isaiah W. and Greenstreet, Laura and Zhang, Stephen and Afanassiev, Anton and Vlot, Anna Hendrika Cornelia and Schiebinger, Geoffrey and Benfey, Philip N. and Ohler, Uwe},
  title={A single-cell Arabidopsis root atlas reveals developmental trajectories in wild-type and cell identity mutants},
  journal={Developmental Cell},
  volume={57},
  number={4},
  pages={543--560.e9},
  year={2022},
  doi={10.1016/j.devcel.2022.01.008},
  url={https://doi.org/10.1016/j.devcel.2022.01.008}
}

@article{Lymphoid2019,
  author={Satpathy, Ansuman T. and Granja, Jeffrey M. and Yost, Kathryn E. and Qi, Yanyan and Meschi, Francesca and McDermott, Geoffrey P. and Olsen, Brett N. and Mumbach, Maxwell R. and Pierce, Sarah E. and Corces, M. Ryan and Shah, Preyas and Bell, Jason C. and Jhutty, Darisha and Nemec, Corey M. and Wang, Jean and Wang, Li and Yin, Yifeng and Giresi, Paul G. and Chang, Anne Lynn S. and Zheng, Grace X. Y. and Greenleaf, William J. and Chang, Howard Y.},
  title={Massively parallel single-cell chromatin landscapes of human immune cell development and intratumoral T cell exhaustion},
  journal={Nature Biotechnology},
  volume={37},
  number={8},
  pages={925--936},
  year={2019},
  doi={10.1038/s41587-019-0206-z},
  url={https://doi.org/10.1038/s41587-019-0206-z}
}

@article{Donoho2003,
author = {Donoho, David L. and Grimes, Carrie},
title = {Hessian eigenmaps: Locally linear embedding techniques for high-dimensional data},
journal = {Proceedings of the National Academy of Sciences},
volume = {100},
number = {10},
pages = {5591-5596},
year = {2003},
doi = {10.1073/pnas.1031596100},
URL = {https://www.pnas.org/doi/abs/10.1073/pnas.1031596100},
eprint = {https://www.pnas.org/doi/pdf/10.1073/pnas.1031596100},
}

@techreport{Bernstein2000,
  title={Graph approximations to geodesics on embedded manifolds},
  author={Bernstein, Mira and de Silva, Vin and Langford, John C and Tenenbaum, Joshua B},
  year={2000},
  institution={Michigan State University}
}

@book{Levin2017,
author = {Levin, David A. and Peres, Yuval},
year = {2017},
month = {10},
pages = {},
title = {Markov Chains and Mixing Times},
publisher = {American Mathematical Society},
isbn = {9781470429621},
doi = {10.1090/mbk/107}
}

@article{Estrada2007,
title = {Statistical-mechanical approach to subgraph centrality in complex networks},
journal = {Chemical Physics Letters},
volume = {439},
number = {1},
pages = {247-251},
year = {2007},
issn = {0009-2614},
doi = {https://doi.org/10.1016/j.cplett.2007.03.098},
url = {https://www.sciencedirect.com/science/article/pii/S0009261407004058},
author = {Estrada, Ernesto and Hatano, Naomichi}
}

@inproceedings{Ma2020,
 author = {Ma, Zheng and Xuan, Junyu and Wang, Yu Guang and Li, Ming and Li\`{o}, Pietro},
 booktitle = {Advances in Neural Information Processing Systems},
 editor = {H. Larochelle and M. Ranzato and R. Hadsell and M.F. Balcan and H. Lin},
 pages = {16421--16433},
 publisher = {Curran Associates, Inc.},
 title = {Path Integral Based Convolution and Pooling for Graph Neural Networks},
 url = {https://proceedings.neurips.cc/paper_files/paper/2020/file/be53d253d6bc3258a8160556dda3e9b2-Paper.pdf},
 volume = {33},
 year = {2020}
}

@article{Hein2007,
author = {Hein, Matthias and Audibert, Jean-Yves and Luxburg, Ulrike von},
title = {Graph Laplacians and their Convergence on Random Neighborhood Graphs},
year = {2007},
issue_date = {12/1/2007},
publisher = {JMLR.org},
volume = {8},
issn = {1532-4435},
journal = {J. Mach. Learn. Res.},
month = dec,
pages = {1325–1370},
numpages = {46}
}

@article{Muller1959,
author = {Muller, Mervin E.},
title = {A note on a method for generating points uniformly on n-dimensional spheres},
year = {1959},
issue_date = {April 1959},
publisher = {Association for Computing Machinery},
address = {New York, NY, USA},
volume = {2},
number = {4},
issn = {0001-0782},
url = {https://doi.org/10.1145/377939.377946},
doi = {10.1145/377939.377946},
journal = {Commun. ACM},
month = apr,
pages = {19–20},
numpages = {2}
}

\newpage

\appendix

{\Huge \bfseries Appendix}

\section{Additional Experimental Results}
\label{app:additional-experiments}

This appendix extends the main-text experiments by reporting (i) the full
twelve-method baseline set (linear, local, geodesic, spectral, and diffusion
methods), (ii) the complete metric panel (trustworthiness, continuity, row-wise
Spearman/Pearson correlation with the geodesic ground truth, and, where
applicable, branch silhouette), and (iii) the geodesic ground truth under both
the analytic protocol, where a closed form exists, and the shortest-path
protocol of \citet{Moon2019} elsewhere. \Cref{app:analytic-validation}
establishes that the two protocols agree on the sphere and preserve the relative
ordering of diffusion methods on the Swiss roll, but that the shortest-path
protocol is biased toward methods whose dissimilarity construction shares its
kNN-graph backbone. We report row-wise correlations throughout rather than
global ones: across all datasets the two aggregations track each other closely
and never reorder the diffusion methods, so the global columns are omitted.
Subsequent subsections present per-dataset full-metric tables and discussion.

\subsection{Ground-Truth Geodesic Protocols and Validation}
\label{app:analytic-validation}

\paragraph{Protocols.}
We compute geodesic ground truth two ways. The \emph{analytic}
protocol uses a closed-form intrinsic distance: arc length in the
$(t, h)$ parameter domain for the Swiss roll variants, and the great-circle distance
\[
d_{\text{geo}}(p_i, p_j) = R\,\arccos\!\big(\langle p_i, p_j\rangle / R^2\big)
\]
for the sphere (with $p_i,p_j$ lying on the sphere of radius $R$, i.e.\
$\|p_i\|=\|p_j\|=R$, so $\langle p_i,p_j\rangle/R^2 \in [-1,1]$).
The \emph{shortest-path} protocol follows the
convention of \citet{Moon2019}: ground-truth distances are
computed as Dijkstra shortest paths on a 15-nearest-neighbor graph
built over the clean (noise-free) reference coordinates. An analytic
geodesic is available only for the Swiss roll (uniform and
non-uniform) and the sphere; on Swiss Hole the geodesic admits no
closed form (the hole must be routed around); the torus embedded in $\mathbb{R}^3$ has non-zero Gaussian curvature of both signs and admits no closed-form geodesic; on the DLA-generated artificial trees there is no natural ambient
parameterization. We therefore use the analytic protocol where it
exists and the shortest-path protocol elsewhere.\footnote{At the
distance level, Isomap and Shortest Path use the same
dissimilarity matrix by construction (classical Isomap is MDS on a
shortest-path-on-kNN matrix); they differ only in the embedding
step and so are reported separately at the embedding level.}

\paragraph{Agreement on the sphere.}
On the sphere the two protocols are essentially interchangeable.
\Cref{tab:protocol-agreement} reports row-wise Spearman of method
dissimilarity matrices against each ground truth: \entropath, DTNE,
HeatGeo, PHATE, and Diffusion Maps agree to within $\le 0.002$
across protocols, and rankings are identical. The reason is geometric:
the \emph{chord} (ambient Euclidean) distance $d_{\text{chord}} = \|p_i - p_j\|$
between two surface points is a strictly monotone function of the great-circle
distance, $d_{\text{chord}} = 2R\sin(d_{\text{geo}}/2R)$ for central angle
$\theta = d_{\text{geo}}/R \in [0,\pi]$. Spearman correlation is invariant under
monotone transformations, so a 15-NN graph built from clean samples recovers the
geodesic order with negligible error. The sphere therefore represents a best-case
scenario for the shortest-path protocol: when ambient and intrinsic geometry are
tightly coupled, the shortest-path proxy is effectively exact at the ranking level.

\paragraph{Protocol bias on the Swiss roll.}
The Swiss roll exposes a genuine bias: the shortest-path protocol
inflates the scores of methods whose dissimilarity construction shares
the kNN-graph backbone with the ground truth, while the analytic
protocol does not.

\begin{table}[t]
\centering
\small
\setlength{\tabcolsep}{5pt}
\caption{Distance-level row-wise Spearman correlation between method
  dissimilarity matrices and ground-truth geodesic distances under the
  analytic ($A$) and shortest-path ($SP$) protocols. The sphere is a
  \emph{negative control}: chord distance is a monotone function of
  great-circle distance and Spearman is invariant to monotone
  reparameterisation, so the two protocols yield identical rankings (exact to
  three decimals here; the $\le 0.002$ residual appears only in the Pearson
  columns of \Cref{tab:sphere-l1}). On the Swiss roll the protocols diverge:
  the shortest-path protocol favours methods whose dissimilarity is induced by
  the kNN graph. Shortest Path reaches $1.000$ on the non-uniform Swiss roll
  because its dissimilarity matrix coincides with the shortest-path ground
  truth used for evaluation --- the strongest form of kNN-backbone bias.
  Per-method standard deviations across 30 seeds are reported in the
  per-dataset tables.}
\label{tab:protocol-agreement}
\begin{tabular}{l cc c cc c cc}
\toprule
 & \multicolumn{2}{c}{Swiss Roll (unif.)} & & \multicolumn{2}{c}{Swiss Roll (non-unif.)} & & \multicolumn{2}{c}{Sphere} \\
\cmidrule(lr){2-3} \cmidrule(lr){5-6} \cmidrule(lr){8-9}
Method        & $A$   & $SP$  & & $A$   & $SP$  & & $A$   & $SP$  \\
\midrule
Euclidean     & 0.528 & 0.734 & & 0.899 & 0.977 & & 0.990 & 0.990 \\
Isomap / Shortest Path       & 0.829 & \textbf{0.993} & & \textbf{0.934} & \textbf{1.000} & & \textbf{0.994} & \textbf{0.994} \\
PHATE         & 0.806 & 0.936 & & 0.816 & 0.892 & & 0.993 & 0.993 \\
HeatGeo       & 0.791 & 0.962 & & 0.796 & 0.889 & & 0.992 & 0.992 \\
DTNE          & 0.833 & 0.958 & & 0.808 & 0.900 & & 0.991 & 0.991 \\
\entropath     & \textbf{0.877} & 0.935 & & 0.882 & 0.916 & & 0.981 & 0.981 \\
\bottomrule
\end{tabular}
\end{table}

On the Swiss roll under the analytic protocol \entropath achieves the
highest distance-level row-wise Spearman ($0.877$), ahead of DTNE
($0.833$) and Isomap/SP ($0.829$), and well ahead of PHATE ($0.806$)
and HeatGeo ($0.791$). Under the shortest-path protocol, by contrast,
Isomap and Shortest Path saturate at $\approx 0.993$ because they share
the construction of the ground truth itself, while HeatGeo and DTNE ---
which combine heat-kernel or random-walk powers with the same 15-NN
graph --- climb to $\approx 0.96$; \entropath is largely unchanged
($0.935$). On the sphere this distinction collapses because the ambient
and intrinsic geometries are tightly coupled. We therefore consider the
analytic protocol the primary ground truth where available.

\paragraph{Consequence for the shortest-path datasets.}
The four datasets without a closed-form geodesic---Torus, Swiss Hole, and
the two tree datasets---use the shortest-path geodesic as the sole ground
truth and therefore inherit the evaluation bias quantified on the Swiss roll,
where an analytic reference allows it to be measured directly. There, methods
whose dissimilarity most closely approximates shortest-path distances on the
same $k$NN graph benefit most from the shortest-path protocol: Isomap most, by
construction; HeatGeo and DTNE next; and \entropath least
(\Cref{tab:protocol-agreement}). For the shortest-path-only datasets no
analytic reference exists, so the magnitude of this bias cannot be measured
directly. The same mechanism nonetheless applies, and methods aligned with the
underlying $k$NN graph can be expected to receive the largest boost. Rather
than repeating this caveat for every table, we state it once here for all
shortest-path row-wise metrics reported in
\Cref{app:torus-full,app:tree-full}. This bias affects absolute correlations
but not the paper's principal conclusion: on the non-uniform Swiss roll
\entropath remains the highest-ranked diffusion method under both protocols
(\Cref{tab:swiss-roll-l1}), though PHATE and DTNE exchange places
under the shortest-path one.


\subsection{Synthetic benchmarks (full tables)}
\label{app:synthetic_benchmarks}

\subsubsection{Swiss Roll}
\label{app:swiss-roll-full}

\paragraph{Setup.}
The same Swiss-roll manifold is evaluated under two sampling regimes: uniform
sampling of the angular parameter $t$, and non-uniform sampling with
$t\sim\mathrm{Beta}(1,4)$, which oversamples the low-$t$ region. Both use
$N=2{,}000$ points, Gaussian noise $\sigma=0.05$, height scale $3$, and $30$
random seeds. The standard Swiss roll $(t\cos t,\, h,\, t\sin t)$ is developable
(isometric to a planar strip), so the intrinsic geodesic is Euclidean in
arc-length coordinates $(s, h)$ with
$s(t) = \tfrac12\big[t\sqrt{1+t^2} + \operatorname{arcsinh} t\big]$
\citep{Tenenbaum2000}:
\[
  d_{\mathrm{geo}}(p_i, p_j) = \sqrt{\,(s_i - s_j)^2 + (h_i - h_j)^2\,}.
\]

\begin{table}[t]
\centering
\scriptsize
\setlength{\tabcolsep}{4pt}
\caption{Swiss roll, \textbf{uniform} sampling: embedding-level metrics.
  Trustworthiness and continuity are protocol-independent; Spearman and Pearson
  (row) use the \emph{analytic} geodesic ground truth. Higher is better.
  \textbf{Bold} = best, \underline{underlined} = second-best in each column.
  Under uniform sampling the diffusion methods perform comparably on geodesic
  correlation, with Diffusion Maps and \entropath the strongest; the
  \entropath advantage emerges under non-uniform sampling
  (\Cref{tab:swiss-roll-nonuniform-l2}), the density-heterogeneous regime it
  targets.}
\label{tab:swiss-roll-uniform-l2}
\begin{tabular}{lcccc}
\toprule
Method              & Trustworth.\                  & Continuity                    & Spearman (row)                & Pearson (row) \\
\midrule
PCA                 & 0.941 $\pm$ 0.002             & \textbf{0.989 $\pm$ 0.000}             & 0.517 $\pm$ 0.008             & 0.524 $\pm$ 0.008 \\
MDS                 & 0.948 $\pm$ 0.002             & \underline{0.987 $\pm$ 0.001}             & 0.520 $\pm$ 0.009             & 0.529 $\pm$ 0.010 \\
UMAP                & \textbf{0.994 $\pm$ 0.001}             & 0.929 $\pm$ 0.005             & \textbf{0.947 $\pm$ 0.022}             & \textbf{0.937 $\pm$ 0.035} \\
$t$-SNE             & \underline{0.992 $\pm$ 0.001}             & 0.953 $\pm$ 0.004             & 0.510 $\pm$ 0.034             & 0.527 $\pm$ 0.038 \\
Isomap              & 0.935 $\pm$ 0.013             & 0.958 $\pm$ 0.005             & 0.777 $\pm$ 0.056             & 0.760 $\pm$ 0.059 \\
Shortest Path       & 0.944 $\pm$ 0.015             & 0.956 $\pm$ 0.005             & 0.810 $\pm$ 0.059             & 0.797 $\pm$ 0.063 \\
Laplacian Eigenmaps & 0.974 $\pm$ 0.011             & 0.938 $\pm$ 0.003             & 0.893 $\pm$ 0.018             & 0.883 $\pm$ 0.020 \\
Diffusion Maps      & 0.979 $\pm$ 0.009    & 0.937 $\pm$ 0.003             & \underline{0.903 $\pm$ 0.013}    & \underline{0.894 $\pm$ 0.015} \\
PHATE               & 0.938 $\pm$ 0.011             & 0.949 $\pm$ 0.004             & 0.802 $\pm$ 0.042             & 0.787 $\pm$ 0.046 \\
HeatGeo             & 0.938 $\pm$ 0.012             & 0.953 $\pm$ 0.004    & 0.769 $\pm$ 0.059             & 0.751 $\pm$ 0.064 \\
DTNE                & 0.951 $\pm$ 0.020             & 0.950 $\pm$ 0.004 & 0.829 $\pm$ 0.063             & 0.815 $\pm$ 0.068 \\
\entropath           & 0.961 $\pm$ 0.013 & 0.947 $\pm$ 0.003             & 0.879 $\pm$ 0.037 & 0.869 $\pm$ 0.041 \\
\bottomrule
\end{tabular}
\end{table}

\begin{table}[t]
\centering
\small
\setlength{\tabcolsep}{5pt}
\caption{Swiss roll distance-level row-wise correlation with the
  \emph{analytic} geodesic ground truth under uniform and non-uniform
  (Beta$(1,4)$) sampling. Under this unbiased evaluation protocol, \entropath
  achieves the highest correlation among the diffusion methods in both sampling
  regimes and, on the uniform Swiss roll, also exceeds the shortest-path-based
  Isomap/Shortest Path baseline. On the non-uniform Swiss roll, the ambient
  Euclidean and Isomap references achieve higher correlations.
  \Cref{tab:protocol-agreement} compares the analytic and shortest-path
  evaluation protocols, showing that the latter reorders the diffusion methods.
  \textbf{Bold} denotes the best result, \underline{underline} the second-best.}
\label{tab:swiss-roll-l1}
\begin{tabular}{l cc cc}
\toprule
 & \multicolumn{2}{c}{Uniform} & \multicolumn{2}{c}{Non-uniform} \\
\cmidrule(lr){2-3}\cmidrule(lr){4-5}
Method                 & Spearman & Pearson & Spearman & Pearson \\
\midrule
Euclidean              & 0.528 $\pm$ 0.008 & 0.537 $\pm$ 0.008 & \underline{0.899 $\pm$ 0.007} & \underline{0.885 $\pm$ 0.009} \\
Isomap / Shortest Path & 0.829 $\pm$ 0.032 & 0.827 $\pm$ 0.033 & \textbf{0.934 $\pm$ 0.006} & \textbf{0.919 $\pm$ 0.008} \\
Diffusion Maps         & 0.499 $\pm$ 0.016 & 0.537 $\pm$ 0.009 & 0.443 $\pm$ 0.025 & 0.502 $\pm$ 0.014 \\
PHATE                  & 0.806 $\pm$ 0.029 & 0.804 $\pm$ 0.029 & 0.816 $\pm$ 0.006 & 0.803 $\pm$ 0.009 \\
HeatGeo                & 0.791 $\pm$ 0.027 & 0.786 $\pm$ 0.029 & 0.796 $\pm$ 0.008 & 0.779 $\pm$ 0.009 \\
DTNE                   & \underline{0.833 $\pm$ 0.040} & \underline{0.823 $\pm$ 0.038} & 0.808 $\pm$ 0.007 & 0.782 $\pm$ 0.009 \\
\entropath             & \textbf{0.877 $\pm$ 0.023} & \textbf{0.864 $\pm$ 0.025} & 0.882 $\pm$ 0.014 & 0.854 $\pm$ 0.013 \\
\bottomrule
\end{tabular}
\end{table}

\begin{table}[t]
\centering
\scriptsize
\setlength{\tabcolsep}{4pt}
\caption{Swiss roll, \textbf{non-uniform} (Beta$(1,4)$) sampling:
  embedding-level metrics. Conventions as in \Cref{tab:swiss-roll-uniform-l2}.
  Among the diffusion methods, \entropath attains the best trustworthiness and
  the best geodesic correlation; its geodesic correlation drops only marginally
  from the uniform regime (Spearman $0.879\!\to\!0.875$).}
\label{tab:swiss-roll-nonuniform-l2}
\begin{tabular}{lcccc}
\toprule
Method              & Trustworth.\                  & Continuity                    & Spearman (row)                & Pearson (row) \\
\midrule
PCA                 & 0.968 $\pm$ 0.002             & \underline{0.993 $\pm$ 0.000}             & 0.876 $\pm$ 0.009             & 0.862 $\pm$ 0.011 \\
MDS                 & 0.986 $\pm$ 0.002             & 0.974 $\pm$ 0.003             & 0.880 $\pm$ 0.008             & 0.865 $\pm$ 0.010 \\
UMAP                & \underline{0.992 $\pm$ 0.002}             & 0.967 $\pm$ 0.005             & 0.751 $\pm$ 0.059             & 0.733 $\pm$ 0.055 \\
$t$-SNE             & \textbf{0.995 $\pm$ 0.001}             & 0.988 $\pm$ 0.002             & 0.887 $\pm$ 0.015             & 0.862 $\pm$ 0.015 \\
Isomap              & 0.976 $\pm$ 0.002             & \textbf{0.994 $\pm$ 0.000}             & \textbf{0.918 $\pm$ 0.007}             & \textbf{0.899 $\pm$ 0.010} \\
Shortest Path       & 0.991 $\pm$ 0.001             & 0.972 $\pm$ 0.004             & \underline{0.899 $\pm$ 0.008}             & \underline{0.882 $\pm$ 0.010} \\
Laplacian Eigenmaps & 0.949 $\pm$ 0.006             & 0.976 $\pm$ 0.002             & 0.716 $\pm$ 0.008             & 0.721 $\pm$ 0.009 \\
Diffusion Maps      & 0.957 $\pm$ 0.005 & 0.976 $\pm$ 0.002             & 0.725 $\pm$ 0.007             & 0.726 $\pm$ 0.009 \\
PHATE               & 0.945 $\pm$ 0.003             & 0.984 $\pm$ 0.001             & 0.741 $\pm$ 0.009             & 0.737 $\pm$ 0.010 \\
HeatGeo             & 0.943 $\pm$ 0.004             & 0.983 $\pm$ 0.002             & 0.714 $\pm$ 0.015             & 0.707 $\pm$ 0.015 \\
DTNE                & 0.949 $\pm$ 0.004             & 0.987 $\pm$ 0.001    & 0.805 $\pm$ 0.013 & 0.789 $\pm$ 0.013 \\
\entropath           & 0.981 $\pm$ 0.007    & 0.985 $\pm$ 0.001 & 0.875 $\pm$ 0.016    & 0.855 $\pm$ 0.016 \\
\bottomrule
\end{tabular}
\end{table}



\subsubsection{Sphere}
\label{app:sphere-full}

\paragraph{Setup.}
$N = 2{,}000$ points sampled uniformly on $S^2$ of radius $R = 1$ by
normalising i.i.d.\ Gaussian vectors, $x \sim \mathcal{N}(0, I_3)$,
$p = R\,x/\|x\|$ \citep{Muller1959}; Gaussian noise $\sigma = 0.05$ added
off-surface; $30$ random seeds. The sphere has constant positive Gaussian
curvature $K = 1/R^2$ and cannot be flattened isometrically to a plane,
making it a stringent test of how methods handle intrinsic curvature. Row
correlations use the closed-form great-circle (analytic) distance; the
shortest-path protocol agrees to $\le 0.002$ on every method
(\Cref{tab:protocol-agreement}), so we report analytic only.

\begin{figure}[t]
\centering
\includegraphics[width=\textwidth]{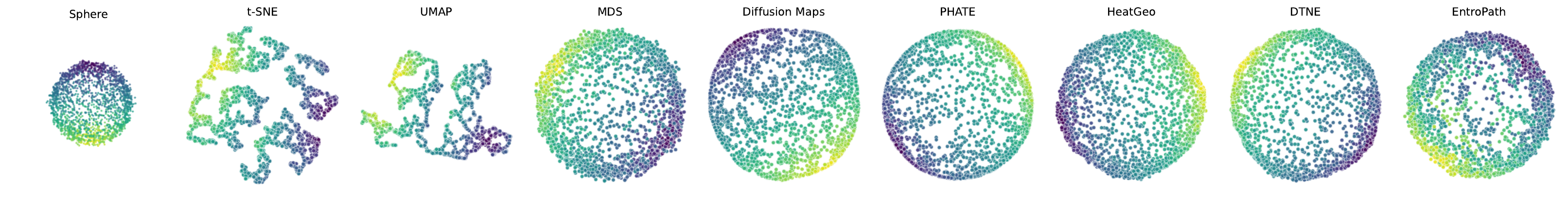}
\caption{Sphere embeddings ($N = 2{,}000$, $\sigma = 0.05$) coloured by
  polar angle $\theta$. $t$-SNE and UMAP tear the spherical topology into
  flat, fragmented layouts; MDS, Diffusion Maps, PHATE, HeatGeo, DTNE, and
  \entropath produce projection-like embeddings in which one hemisphere is
  occluded, with $\theta$ varying smoothly across the disc.}
\label{fig:sphere}
\end{figure}

\begin{table}[t]
\centering
\scriptsize
\setlength{\tabcolsep}{4pt}
\caption{Sphere: embedding-level metrics; row-wise correlations use the
  \emph{analytic} great-circle geodesic. Higher is better. \textbf{Bold} =
  best, \underline{underlined} = second-best in each column. Among the diffusion
  methods \entropath attains the highest trustworthiness, by a clear margin
  ($0.904$ vs.\ $\le 0.861$). UMAP and $t$-SNE reach nearly perfect
  trustworthiness by tearing the spherical topology into locally faithful
  planar patches, sacrificing global geometry. The diffusion methods cluster
  between $0.75$ and $0.78$ on the geodesic-correlation columns because chord
  distance is a monotone function of great-circle distance, so rank correlations
  nearly saturate for any method that respects the ambient geometry;
  differences there are typically $\le 0.01$.}
\label{tab:sphere-l2}
\begin{tabular}{lcccc}
\toprule
Method              & Trustworth.\                  & Continuity                    & Spearman (row)                & Pearson (row) \\
\midrule
PCA                 & 0.848 $\pm$ 0.002             & \underline{0.988 $\pm$ 0.000}             & 0.755 $\pm$ 0.006             & 0.751 $\pm$ 0.006 \\
MDS                 & 0.875 $\pm$ 0.013             & 0.979 $\pm$ 0.000             & \underline{0.771 $\pm$ 0.007}             & \underline{0.766 $\pm$ 0.008} \\
UMAP                & \textbf{0.996 $\pm$ 0.001}             & 0.963 $\pm$ 0.004             & 0.679 $\pm$ 0.027             & 0.686 $\pm$ 0.026 \\
$t$-SNE             & \underline{0.995 $\pm$ 0.002}             & 0.966 $\pm$ 0.001             & 0.682 $\pm$ 0.011             & 0.692 $\pm$ 0.010 \\
Isomap              & 0.848 $\pm$ 0.002             & \textbf{0.989 $\pm$ 0.000}             & 0.757 $\pm$ 0.006             & 0.752 $\pm$ 0.006 \\
Shortest Path       & 0.876 $\pm$ 0.014             & 0.980 $\pm$ 0.000             & \textbf{0.776 $\pm$ 0.007}             & \textbf{0.768 $\pm$ 0.008} \\
Laplacian Eigenmaps & 0.856 $\pm$ 0.004             & 0.987 $\pm$ 0.001             & 0.752 $\pm$ 0.007             & 0.747 $\pm$ 0.007 \\
Diffusion Maps      & 0.852 $\pm$ 0.003             & \underline{0.988 $\pm$ 0.001}    & 0.753 $\pm$ 0.007             & 0.748 $\pm$ 0.006 \\
PHATE               & 0.861 $\pm$ 0.005 & 0.981 $\pm$ 0.001             & 0.766 $\pm$ 0.008 & 0.757 $\pm$ 0.008 \\
HeatGeo             & 0.860 $\pm$ 0.006             & 0.982 $\pm$ 0.001 & 0.766 $\pm$ 0.006 & 0.759 $\pm$ 0.006 \\
DTNE                & 0.860 $\pm$ 0.006             & 0.979 $\pm$ 0.001             & 0.768 $\pm$ 0.007    & 0.759 $\pm$ 0.007 \\
\entropath           & 0.904 $\pm$ 0.021    & 0.977 $\pm$ 0.002             & 0.762 $\pm$ 0.008             & 0.761 $\pm$ 0.009 \\
\bottomrule
\end{tabular}
\end{table}

\begin{table}[t]
\centering
\small
\setlength{\tabcolsep}{5pt}
\caption{Sphere: distance-level row-wise correlation against the \emph{analytic}
  great-circle geodesic. \textbf{Bold} = best, \underline{underlined} =
  second-best in each column. All ambient-geometry-respecting methods cluster
  near $0.98$--$0.99$; \entropath is marginally below the leaders, consistent
  with the entropy term softening long-range path costs on a closed manifold.}
\label{tab:sphere-l1}
\begin{tabular}{lcc}
\toprule
Method        & Spearman (row)          & Pearson (row) \\
\midrule
Euclidean     & 0.990 $\pm$ 0.000       & 0.975 $\pm$ 0.000 \\
Isomap / Shortest Path        & \textbf{0.994 $\pm$ 0.000}       & \textbf{0.994 $\pm$ 0.000} \\
Diffusion Maps & 0.430 $\pm$ 0.013      & 0.518 $\pm$ 0.007 \\
PHATE         & \underline{0.993 $\pm$ 0.000} & 0.974 $\pm$ 0.001 \\
HeatGeo       & 0.992 $\pm$ 0.001 & \underline{0.991 $\pm$ 0.001} \\
DTNE          & 0.991 $\pm$ 0.001       & 0.974 $\pm$ 0.002 \\
\entropath     & 0.981 $\pm$ 0.003       & 0.975 $\pm$ 0.003 \\
\bottomrule
\end{tabular}
\end{table}


\subsubsection{Torus and Swiss Hole}
\label{app:torus-full}

\paragraph{Setup.}
The \emph{torus} has major radius $R = 3$ and minor radius $r = 1$, embedded in
$\mathbb{R}^3$ with periodic angular coordinates in $[0, 2\pi)$; it has genus $1$
and regions of both positive and negative Gaussian curvature, so it is not
developable. The \emph{Swiss hole} is the standard Swiss roll with a rectangular
patch removed, creating a boundary and a non-trivial fundamental group, so
geodesics must route around the missing region. Both use $N = 2{,}000$ points,
Gaussian noise $\sigma = 0.05$, and $30$ seeds. Neither admits a closed-form
geodesic, so the ground truth is shortest-path on a 15-NN graph over clean
coordinates; the standing caveat of \Cref{app:analytic-validation} applies.

\begin{figure}[t]
\centering
\includegraphics[width=\textwidth]{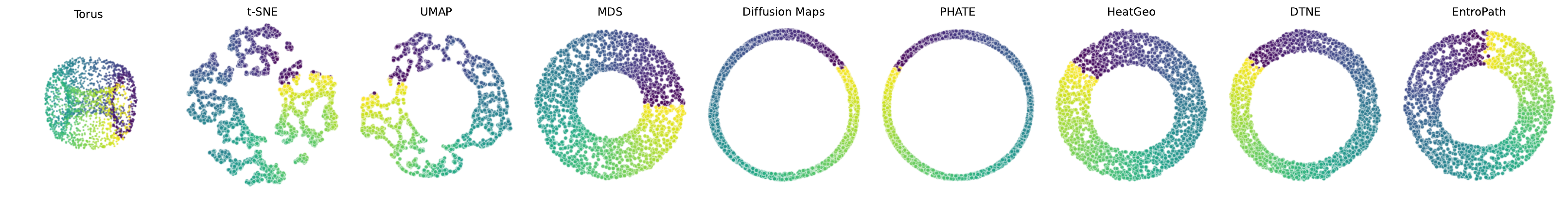}\\[0.3em]
\includegraphics[width=\textwidth]{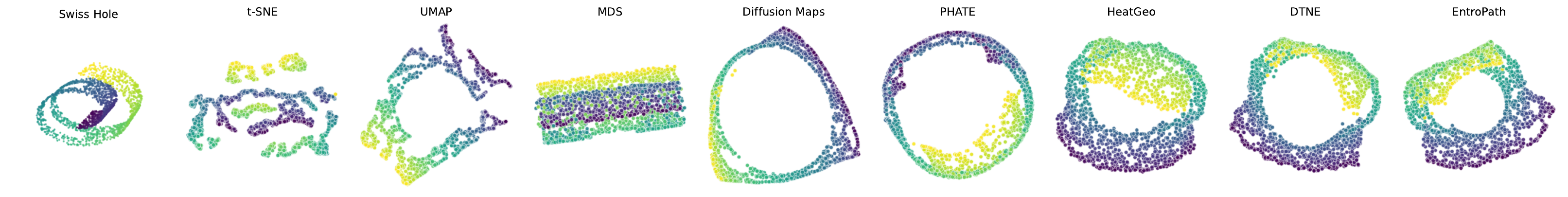}
\caption{Embeddings of (top) torus coloured by major angle $\theta$ (cyclic
  colormap) and (bottom) Swiss hole coloured by angular parameter $t$. Methods
  that preserve global topology recover the toroidal layout; the hole appears as
  a gap in embeddings that respect the manifold geometry.}
\label{fig:torus-hole}
\end{figure}

\begin{table}[t]
\centering
\scriptsize
\setlength{\tabcolsep}{4pt}
\caption{Torus: embedding-level metrics, shortest-path ground truth. \textbf{Bold}
  = best, \underline{underlined} = second-best in each column. All methods score
  highly on this uniformly sampled, locally flat manifold; see the consolidated
  discussion for interpretation.}
\label{tab:torus-l2}
\begin{tabular}{lcccc}
\toprule
Method              & Trustworth.\                  & Continuity                 & Spearman (row)                & Pearson (row) \\
\midrule
PCA                 & 0.976 $\pm$ 0.001             & \textbf{0.995 $\pm$ 0.000}          & 0.977 $\pm$ 0.000             & \textbf{0.972 $\pm$ 0.001} \\
MDS                 & 0.976 $\pm$ 0.001             & \underline{0.994 $\pm$ 0.000}          & 0.975 $\pm$ 0.001             & 0.971 $\pm$ 0.001 \\
UMAP                & \underline{0.993 $\pm$ 0.001}             & 0.979 $\pm$ 0.001          & 0.934 $\pm$ 0.012             & 0.928 $\pm$ 0.011 \\
$t$-SNE             & \textbf{0.995 $\pm$ 0.001}             & 0.980 $\pm$ 0.001          & 0.933 $\pm$ 0.007             & 0.930 $\pm$ 0.007 \\
Isomap              & 0.975 $\pm$ 0.001             & 0.992 $\pm$ 0.000          & \textbf{0.980 $\pm$ 0.000}             & 0.970 $\pm$ 0.001 \\
Shortest Path       & 0.976 $\pm$ 0.001             & \underline{0.994 $\pm$ 0.000}          & \underline{0.979 $\pm$ 0.000}             & \textbf{0.972 $\pm$ 0.000} \\
Laplacian Eigenmaps & 0.959 $\pm$ 0.002             & 0.977 $\pm$ 0.000          & 0.962 $\pm$ 0.001             & 0.937 $\pm$ 0.002 \\
Diffusion Maps      & 0.959 $\pm$ 0.002             & 0.977 $\pm$ 0.000          & 0.963 $\pm$ 0.001             & 0.937 $\pm$ 0.001 \\
PHATE               & 0.955 $\pm$ 0.001             & 0.976 $\pm$ 0.000          & 0.961 $\pm$ 0.001             & 0.935 $\pm$ 0.001 \\
HeatGeo             & 0.975 $\pm$ 0.001 & 0.991 $\pm$ 0.000 & 0.976 $\pm$ 0.001    & 0.963 $\pm$ 0.002 \\
DTNE                & 0.971 $\pm$ 0.001             & 0.984 $\pm$ 0.001          & 0.972 $\pm$ 0.001             & 0.951 $\pm$ 0.002 \\
\entropath           & 0.976 $\pm$ 0.001    & 0.991 $\pm$ 0.001 & 0.974 $\pm$ 0.003 & 0.964 $\pm$ 0.003 \\
\bottomrule
\end{tabular}
\end{table}

\begin{table}[t]
\centering
\small
\setlength{\tabcolsep}{5pt}
\caption{Torus: distance-level row-wise correlation against the shortest-path
  ground truth. \textbf{Bold} = best, \underline{underlined} = second-best in
  each column.}
\label{tab:torus-l1}
\begin{tabular}{lcc}
\toprule
Method        & Spearman (row)          & Pearson (row) \\
\midrule
Euclidean     & 0.983 $\pm$ 0.000       & 0.982 $\pm$ 0.000 \\
Isomap / Shortest Path        & \textbf{0.999 $\pm$ 0.000}       & \textbf{0.999 $\pm$ 0.000} \\
Diffusion Maps & 0.464 $\pm$ 0.014      & 0.522 $\pm$ 0.006 \\
PHATE         & 0.960 $\pm$ 0.001       & 0.939 $\pm$ 0.002 \\
HeatGeo       & \underline{0.987 $\pm$ 0.002} & \underline{0.984 $\pm$ 0.002} \\
DTNE          & 0.984 $\pm$ 0.001 & 0.971 $\pm$ 0.002 \\
\entropath     & 0.976 $\pm$ 0.006       & 0.968 $\pm$ 0.006 \\
\bottomrule
\end{tabular}
\end{table}

\begin{table}[t]
\centering
\scriptsize
\setlength{\tabcolsep}{4pt}
\caption{Swiss hole: embedding-level metrics, shortest-path ground truth.
  \textbf{Bold} = best, \underline{underlined} = second-best in each column.
  Among the diffusion methods, HeatGeo achieves the highest geodesic-correlation
  scores and DTNE and \entropath tie for the highest trustworthiness ($0.970$).}
\label{tab:swiss-hole-l2}
\begin{tabular}{lcccc}
\toprule
Method              & Trustworth.\               & Continuity                    & Spearman (row)                & Pearson (row) \\
\midrule
PCA                 & 0.951 $\pm$ 0.001          & \textbf{0.990 $\pm$ 0.000}             & 0.683 $\pm$ 0.056             & 0.699 $\pm$ 0.056 \\
MDS                 & 0.959 $\pm$ 0.002          & \underline{0.988 $\pm$ 0.001}             & 0.683 $\pm$ 0.058             & 0.700 $\pm$ 0.059 \\
UMAP                & \textbf{0.994 $\pm$ 0.001}          & 0.937 $\pm$ 0.006             & 0.802 $\pm$ 0.065             & 0.815 $\pm$ 0.069 \\
$t$-SNE             & \textbf{0.994 $\pm$ 0.001}          & 0.967 $\pm$ 0.002             & 0.632 $\pm$ 0.048             & 0.656 $\pm$ 0.050 \\
Isomap              & 0.955 $\pm$ 0.015          & 0.959 $\pm$ 0.006             & \underline{0.943 $\pm$ 0.024}             & \underline{0.942 $\pm$ 0.020} \\
Shortest Path       & 0.964 $\pm$ 0.011          & 0.958 $\pm$ 0.005             & \textbf{0.948 $\pm$ 0.022}             & \textbf{0.948 $\pm$ 0.019} \\
Laplacian Eigenmaps & 0.966 $\pm$ 0.010          & 0.943 $\pm$ 0.003             & 0.866 $\pm$ 0.025             & 0.873 $\pm$ 0.019 \\
Diffusion Maps      & 0.969 $\pm$ 0.009          & 0.943 $\pm$ 0.003             & 0.864 $\pm$ 0.026             & 0.871 $\pm$ 0.020 \\
PHATE               & 0.953 $\pm$ 0.012          & 0.948 $\pm$ 0.004             & 0.907 $\pm$ 0.018             & 0.908 $\pm$ 0.010 \\
HeatGeo             & 0.966 $\pm$ 0.008          & 0.956 $\pm$ 0.004    & 0.931 $\pm$ 0.014    & 0.934 $\pm$ 0.012 \\
DTNE                & 0.970 $\pm$ 0.012 & 0.951 $\pm$ 0.004 & 0.920 $\pm$ 0.026 & 0.924 $\pm$ 0.019 \\
\entropath           & 0.970 $\pm$ 0.011 & 0.951 $\pm$ 0.004 & 0.910 $\pm$ 0.030             & 0.917 $\pm$ 0.024 \\
\bottomrule
\end{tabular}
\end{table}

\begin{table}[t]
\centering
\small
\setlength{\tabcolsep}{5pt}
\caption{Swiss hole: distance-level row-wise correlation against the shortest-path
  ground truth. \textbf{Bold} = best, \underline{underlined} = second-best in
  each column.}
\label{tab:swiss-hole-l1}
\begin{tabular}{lcc}
\toprule
Method        & Spearman (row)          & Pearson (row) \\
\midrule
Euclidean     & 0.700 $\pm$ 0.059       & 0.720 $\pm$ 0.060 \\
Isomap / Shortest Path       & \textbf{0.991 $\pm$ 0.023}       & \textbf{0.992 $\pm$ 0.021} \\
Diffusion Maps & 0.350 $\pm$ 0.021      & 0.477 $\pm$ 0.019 \\
PHATE         & 0.925 $\pm$ 0.017       & 0.926 $\pm$ 0.008 \\
HeatGeo       & \underline{0.958 $\pm$ 0.021} & \underline{0.961 $\pm$ 0.016} \\
DTNE          & 0.950 $\pm$ 0.027 & 0.949 $\pm$ 0.015 \\
\entropath     & 0.935 $\pm$ 0.030       & 0.933 $\pm$ 0.029 \\
\bottomrule
\end{tabular}
\end{table}


\subsubsection{Artificial Trees}
\label{app:tree-full}

\paragraph{Setup.}
DLA-generated branching structures of \citet{Moon2019}. The
\emph{sparse tree} has $B = 6$ branches of length $\ell = 500$ in
$\mathbb{R}^{10}$ ambient space with noise $\sigma = 2.0$; the
\emph{dense tree} matches the original PHATE configuration with
$B = 20$ branches of length $\ell = 100$ in $\mathbb{R}^{100}$
with noise $\sigma = 4.0$. $N = B \times \ell$ samples per draw,
$30$ seeds. Ground truth: shortest-path on 15-NN over clean
coordinates (no natural ambient parameterization exists), so the
standing caveat of \Cref{app:analytic-validation} applies.

We additionally report a \emph{branch silhouette} score
(silhouette coefficient of the 2D embedding under the
ground-truth branch-identity labels): a higher value means the
embedding more cleanly separates branches in 2D. This is the
direct measure of branch-structure recovery used in
trajectory-inference settings.

\begin{table}[t]
\centering
\tiny
\setlength{\tabcolsep}{3pt}
\caption{Sparse tree ($B=6$, $\ell=500$, $\mathbb{R}^{10}$, $\sigma=2.0$,
  $30$ seeds): embedding-level metrics and branch silhouette, shortest-path
  ground truth. \textbf{Bold} = best, \underline{underlined} = second-best in
  each column. \entropath attains the highest geodesic-correlation scores;
  HeatGeo the highest branch silhouette.}
\label{tab:sparse-tree-l2}
\begin{tabular}{lccccc}
\toprule
Method              & Trustworth.\                  & Continuity                 & Spearman (row)                & Pearson (row)                 & Branch silh.\ \\
\midrule
PCA                 & 0.937 $\pm$ 0.024             & 0.975 $\pm$ 0.008          & 0.717 $\pm$ 0.068             & 0.740 $\pm$ 0.063             & 0.264 $\pm$ 0.077 \\
MDS                 & 0.969 $\pm$ 0.009             & 0.967 $\pm$ 0.020          & 0.748 $\pm$ 0.052             & 0.772 $\pm$ 0.050             & 0.306 $\pm$ 0.067 \\
UMAP                & \underline{0.996 $\pm$ 0.001}             & 0.987 $\pm$ 0.002          & 0.582 $\pm$ 0.081             & 0.600 $\pm$ 0.085             & \underline{0.418 $\pm$ 0.069} \\
$t$-SNE             & \textbf{0.997 $\pm$ 0.001}             & \textbf{0.994 $\pm$ 0.001}          & 0.734 $\pm$ 0.047             & 0.769 $\pm$ 0.043             & 0.389 $\pm$ 0.030 \\
Isomap              & 0.970 $\pm$ 0.017             & 0.986 $\pm$ 0.005          & 0.796 $\pm$ 0.055             & 0.812 $\pm$ 0.048             & 0.277 $\pm$ 0.079 \\
Shortest Path       & 0.984 $\pm$ 0.005             & 0.964 $\pm$ 0.022          & 0.835 $\pm$ 0.050             & 0.847 $\pm$ 0.047             & 0.328 $\pm$ 0.064 \\
Laplacian Eigenmaps & 0.979 $\pm$ 0.019             & 0.990 $\pm$ 0.003          & 0.770 $\pm$ 0.056             & 0.762 $\pm$ 0.061             & 0.326 $\pm$ 0.114 \\
Diffusion Maps      & 0.980 $\pm$ 0.018             & 0.991 $\pm$ 0.003          & 0.767 $\pm$ 0.055             & 0.758 $\pm$ 0.060             & 0.331 $\pm$ 0.110 \\
PHATE               & 0.980 $\pm$ 0.006             & 0.989 $\pm$ 0.002          & 0.404 $\pm$ 0.061             & 0.500 $\pm$ 0.051             & 0.246 $\pm$ 0.057 \\
HeatGeo             & 0.992 $\pm$ 0.002    & \underline{0.993 $\pm$ 0.001} & 0.770 $\pm$ 0.050             & 0.808 $\pm$ 0.043             & \textbf{0.428 $\pm$ 0.048} \\
DTNE                & 0.989 $\pm$ 0.003             & 0.992 $\pm$ 0.003          & \underline{0.859 $\pm$ 0.049} & \underline{0.874 $\pm$ 0.044} & 0.398 $\pm$ 0.070 \\
\entropath           & 0.991 $\pm$ 0.003 & \underline{0.993 $\pm$ 0.002} & \textbf{0.869 $\pm$ 0.046}    & \textbf{0.884 $\pm$ 0.041}    & 0.374 $\pm$ 0.046 \\
\bottomrule
\end{tabular}
\end{table}

\begin{table}[t]
\centering
\small
\setlength{\tabcolsep}{5pt}
\caption{Sparse tree: distance-level row-wise correlation against the
  shortest-path ground truth. \textbf{Bold} = best, \underline{underlined} =
  second-best in each column. \entropath leads both columns, exceeding the
  Isomap/Shortest Path baseline.}
\label{tab:sparse-tree-l1}
\begin{tabular}{lcc}
\toprule
Method        & Spearman (row)          & Pearson (row) \\
\midrule
Euclidean     & 0.806 $\pm$ 0.047       & 0.830 $\pm$ 0.041 \\
Isomap / Shortest Path        & 0.917 $\pm$ 0.042       & 0.921 $\pm$ 0.038 \\
Diffusion Maps & 0.491 $\pm$ 0.088      & 0.543 $\pm$ 0.044 \\
PHATE         & 0.318 $\pm$ 0.070       & 0.615 $\pm$ 0.043 \\
HeatGeo       & 0.918 $\pm$ 0.040       & 0.886 $\pm$ 0.026 \\
DTNE          & \underline{0.927 $\pm$ 0.039} & \underline{0.921 $\pm$ 0.031} \\
\entropath     & \textbf{0.934 $\pm$ 0.037} & \textbf{0.935 $\pm$ 0.033} \\
\bottomrule
\end{tabular}
\end{table}

\begin{table}[t]
\centering
\tiny
\setlength{\tabcolsep}{3pt}
\caption{Dense tree ($B=20$, $\ell=100$, $\mathbb{R}^{100}$, $\sigma=4.0$,
  $30$ seeds): embedding-level metrics and branch silhouette, shortest-path
  ground truth. \textbf{Bold} = best, \underline{underlined} = second-best in
  each column. In this high-dimensional, densely branching regime DTNE leads the
  geodesic-correlation columns with \entropath second; among the diffusion
  methods HeatGeo leads trustworthiness and branch silhouette.}
\label{tab:dense-tree-l2}
\begin{tabular}{lccccc}
\toprule
Method              & Trustworth.\                  & Continuity                    & Spearman (row)                & Pearson (row)                 & Branch silh.\ \\
\midrule
PCA                 & 0.807 $\pm$ 0.023             & 0.926 $\pm$ 0.010             & 0.621 $\pm$ 0.040             & 0.642 $\pm$ 0.033             & -0.020 $\pm$ 0.023 \\
MDS                 & 0.812 $\pm$ 0.019             & 0.856 $\pm$ 0.016             & 0.650 $\pm$ 0.023             & 0.647 $\pm$ 0.020             & -0.071 $\pm$ 0.019 \\
UMAP                & \textbf{0.989 $\pm$ 0.002}             & 0.972 $\pm$ 0.004             & 0.606 $\pm$ 0.034             & 0.618 $\pm$ 0.032             & \textbf{0.480 $\pm$ 0.028} \\
$t$-SNE             & \textbf{0.989 $\pm$ 0.002}             & \textbf{0.975 $\pm$ 0.003}             & 0.628 $\pm$ 0.028             & 0.669 $\pm$ 0.022             & \underline{0.468 $\pm$ 0.021} \\
Isomap              & 0.820 $\pm$ 0.021             & 0.943 $\pm$ 0.007             & 0.711 $\pm$ 0.034             & 0.723 $\pm$ 0.028             & -0.021 $\pm$ 0.022 \\
Shortest Path       & 0.843 $\pm$ 0.014             & 0.879 $\pm$ 0.014             & 0.751 $\pm$ 0.020             & 0.746 $\pm$ 0.016             & -0.035 $\pm$ 0.017 \\
Laplacian Eigenmaps & 0.892 $\pm$ 0.017             & 0.962 $\pm$ 0.005             & 0.697 $\pm$ 0.037             & 0.674 $\pm$ 0.043             & 0.061 $\pm$ 0.040 \\
Diffusion Maps      & 0.893 $\pm$ 0.020 & 0.962 $\pm$ 0.005 & 0.701 $\pm$ 0.033             & 0.682 $\pm$ 0.033             & 0.074 $\pm$ 0.045 \\
PHATE               & 0.779 $\pm$ 0.029             & 0.916 $\pm$ 0.010             & 0.349 $\pm$ 0.070             & 0.405 $\pm$ 0.064             & -0.106 $\pm$ 0.022 \\
HeatGeo             & 0.953 $\pm$ 0.008    & \underline{0.973 $\pm$ 0.008}    & 0.740 $\pm$ 0.025             & 0.757 $\pm$ 0.021             & 0.191 $\pm$ 0.013 \\
DTNE                & 0.872 $\pm$ 0.012             & 0.962 $\pm$ 0.003 & \textbf{0.829 $\pm$ 0.018}    & \textbf{0.827 $\pm$ 0.015}    & -0.006 $\pm$ 0.015 \\
\entropath           & 0.879 $\pm$ 0.041             & 0.953 $\pm$ 0.023             & \underline{0.785 $\pm$ 0.022} & \underline{0.787 $\pm$ 0.022} & 0.019 $\pm$ 0.075 \\
\bottomrule
\end{tabular}
\end{table}

\begin{table}[t]
\centering
\small
\setlength{\tabcolsep}{5pt}
\caption{Dense tree: distance-level row-wise correlation against the
  shortest-path ground truth. \textbf{Bold} = best, \underline{underlined} =
  second-best in each column. DTNE leads Spearman, HeatGeo leads Pearson;
  \entropath is competitive but exhibits higher variance.}
\label{tab:dense-tree-l1}
\begin{tabular}{lcc}
\toprule
Method        & Spearman (row)          & Pearson (row) \\
\midrule
Euclidean     & 0.828 $\pm$ 0.016       & 0.850 $\pm$ 0.014 \\
Isomap / Shortest Path        & 0.924 $\pm$ 0.010       & 0.923 $\pm$ 0.010 \\
Diffusion Maps & 0.846 $\pm$ 0.026      & 0.859 $\pm$ 0.018 \\
PHATE         & 0.358 $\pm$ 0.097       & 0.499 $\pm$ 0.074 \\
HeatGeo       & \underline{0.954 $\pm$ 0.006} & \textbf{0.961 $\pm$ 0.005} \\
DTNE          & \textbf{0.961 $\pm$ 0.004} & \underline{0.956 $\pm$ 0.004} \\
\entropath     & 0.947 $\pm$ 0.022       & 0.918 $\pm$ 0.049 \\
\bottomrule
\end{tabular}
\end{table}

\begin{figure}[t]
\centering
\includegraphics[width=\textwidth]{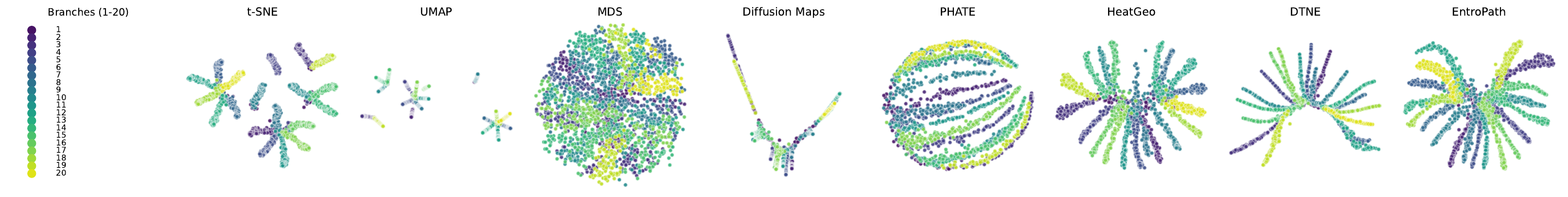}
\caption{Embeddings of the dense artificial tree, coloured by branch identity.
  Path-integral methods (DTNE, HeatGeo, \entropath) partially recover the
  branch structure, while PHATE collapses it (Table~\ref{tab:dense-tree-l2}).}
\label{fig:densetree}
\end{figure}


\paragraph{Discussion of synthetic benchmarks.}
Several patterns recur across the synthetic manifolds. First, under the
analytic protocol --- the most direct test of geodesic preservation, available
on the Swiss roll and sphere --- \entropath achieves the highest distance-level
fidelity among the diffusion methods on the Swiss roll in both sampling regimes
(uniform Spearman row $0.877$ vs.\ DTNE $0.833$, PHATE $0.806$, HeatGeo $0.791$;
non-uniform $0.882$ vs.\ PHATE $0.816$, DTNE $0.808$, HeatGeo $0.796$). Second,
density heterogeneity is where the maximum entropy walk earns its advantage:
under the non-uniform Swiss roll the non-MERW methods (Diffusion Maps, PHATE,
HeatGeo, DTNE) show substantially larger declines in embedding-level analytic
correlation than \entropath, which drops only marginally (Spearman row
$0.879\!\to\!0.875$). 

Where density is constant --- uniform Swiss roll and torus, both uniformly
sampled and locally flat --- this leverage disappears: the diffusion methods
perform comparably, and \entropath is competitive. This is
consistent with the path-uniformization mechanism (\Cref{rem:bottleneck})
conferring an advantage only when sampling density varies.

Third, the shortest-path protocol systematically favours methods aligned with
the kNN graph, quantified on the Swiss roll where an analytic reference exists
(\Cref{tab:protocol-agreement}): on uniform sampling the inflation is enough to
reverse the analytic ordering of \entropath relative to HeatGeo and DTNE at the
distance level, but on the non-uniform benchmark the \entropath advantage
survives both protocols. The sphere illustrates the opposite limit: chord and
great-circle distance are related by the strictly monotone
$d_{\mathrm{chord}}=2R\sin(d_{\mathrm{geo}}/2R)$, so rank correlations
near-saturate for any ambient-geometry-respecting method (the diffusion methods
cluster within $[0.75,0.78]$ on Spearman) and the two protocols agree to
$\le 0.002$ --- the empirical baseline for protocol equivalence. On the sphere
\entropath instead distinguishes itself on trustworthiness, leading by a
substantial margin ($0.904$ vs.\ $\le 0.861$ for the next-best diffusion
method) while ceding $\approx 0.01$ on geodesic Spearman.

On the Swiss hole, HeatGeo achieves the best geodesic-correlation metrics
(Spearman $0.931$, Pearson $0.934$), with DTNE and \entropath tying for the
best trustworthiness ($0.970$); the boundary forces geodesics to route around
the hole, which HeatGeo's heat-kernel distances appear to capture slightly more
accurately than the MERW path ensemble. As on the torus, the shortest-path bias
likely flatters the kNN-aligned methods here, though without an analytic
reference we cannot quantify it. We include the torus and Swiss hole as evidence
that \entropath remains competitive on non-trivial topology, not as benchmarks
where it leads.

The two tree benchmarks present opposite outcomes. On the sparse tree
\entropath leads both distance-level and embedding-level geodesic-correlation
columns among diffusion methods (embedding Spearman row $0.869$ vs.\ DTNE
$0.859$, HeatGeo $0.770$). On the dense tree ($20$ short branches in
$\mathbb{R}^{100}$) DTNE leads the embedding-level geodesic columns with
\entropath second, and DTNE and HeatGeo split the distance-level columns. A plausible explanation is \entropath's entropy-based diffusion-time selection:
the von Neumann entropy of $T^k$ is estimated from far fewer samples per branch
in the dense setting ($100$ vs.\ $500$), which would yield noisier knee-point
estimates and higher seed-to-seed variance --- consistent with \entropath's
trustworthiness standard deviation rising from $0.003$ on the sparse tree to
$0.040$ on the dense tree. DTNE and HeatGeo, using fixed-bandwidth heat kernels, are not subject to this estimation noise; we return to this as a limitation in \Cref{sec:conclusion}.
PHATE collapses on both tree configurations (embedding-level $\rho\approx0.4$
sparse, $\approx0.35$ dense), consistent with the failure mode reported by
\citet{Wei2025}. Together these results support the path-uniformization claim
(\Cref{rem:bottleneck}) and the short-time geodesic approximation of
\Cref{thm:geodesic}, while delineating where the entropy-based depth selection
is reliable.


\subsection{Clustering evaluation}
\label{app:clustering}
\entropath is designed to preserve manifold geometry---geodesic distances on the
data manifold---rather than to maximise the separation of discrete clusters. A
natural question is whether this geometric objective comes at the expense of
performance on standard clustering tasks. We therefore evaluate \entropath as a
general-purpose two-dimensional embedding against eight baselines, and find that
it remains competitive without any clustering-specific tuning.

\paragraph{Datasets.}
We use three datasets that span a spectrum from continuous geometry to discrete
clusters: a synthetic branching \emph{Tree} (a continuous one-dimensional
manifold with five branches, generated by diffusion-limited aggregation), the
\emph{MNIST} handwritten digits (ten discrete classes that nonetheless carry
strong intra-class manifold structure), and \emph{PBMC}, a single-cell RNA-seq
dataset of peripheral blood mononuclear cells with eight annotated cell types.
This ordering---continuous manifold, clusters-with-structure, discrete cell
types---lets us read \entropath's behaviour as a function of how cluster-like the
target structure is.

\paragraph{Protocol.}
Following the clustering protocol of \citet{Huguet2023}, each method
embeds the data into two dimensions; the embedding is standardised
(zero mean, unit variance per axis) and clustered with $k$-means, with the number
of clusters fixed to the number of ground-truth classes. We report the mean and
standard deviation over 30 random seeds. As headline metrics we use
\emph{homogeneity} and the \emph{adjusted mutual information} (aMI) between the
$k$-means assignment and the ground-truth labels.

Baselines are run at each method's published per-dataset configuration: for the
diffusion-based methods this means the manifold neighbourhood
($k_{\mathrm{nn}}=10$) on the synthetic Tree and the single-cell clustering
neighbourhood ($k_{\mathrm{nn}}=5$, with PCA to $40$ components) on MNIST and
PBMC. \entropath uses a single fixed recipe across all datasets ($k_{\mathrm{nn}}=15$,
$\alpha$-decay kernel, k-means landmarks with $M=2000$ auto-enabled for
$n>2000$). Under this rule landmarks are active only on the Tree ($n=2500$);
PBMC ($n=1319$) and MNIST ($n=898$) are embedded at full rank. We emphasise that absolute clustering scores are sensitive to protocol choices (standardisation, dataset construction, and neighbourhood size); we therefore re-evaluate every method under this single standardised protocol rather than
quoting previously published scores, and treat the relative ordering of methods
as the quantity of interest.

\begin{table}[t]
  \centering
  \small
  \setlength{\tabcolsep}{4pt}
  \caption{Clustering quality (mean\,$\pm$\,std over 30 seeds). Homogeneity and
  adjusted mutual information (aMI) of $k$-means on the 2D embedding. Best per
  column in \textbf{bold}. \entropath leads the continuous Tree manifold ---
  with the lowest variance among the competitive methods --- is competitive on
  MNIST, and is mid-table on the cluster-like PBMC data, where methods that
  produce more compact cluster separation lead.}
  \label{tab:clustering-main}
  \begin{tabular}{l cc cc cc}
    \toprule
     & \multicolumn{2}{c}{\textbf{Tree}} & \multicolumn{2}{c}{\textbf{PBMC}}
       & \multicolumn{2}{c}{\textbf{MNIST}} \\
    \cmidrule(lr){2-3}\cmidrule(lr){4-5}\cmidrule(lr){6-7}
    Method & Homog. & aMI & Homog. & aMI & Homog. & aMI \\
    \midrule
    t-SNE          & $0.621{\scriptstyle\,\pm\,0.092}$ & $0.623{\scriptstyle\,\pm\,0.093}$ & $0.514{\scriptstyle\,\pm\,0.020}$ & $0.450{\scriptstyle\,\pm\,0.018}$ & $\mathbf{0.880}{\scriptstyle\,\pm\,0.000}$ & $\mathbf{0.879}{\scriptstyle\,\pm\,0.000}$ \\
    UMAP           & $0.700{\scriptstyle\,\pm\,0.066}$ & $0.709{\scriptstyle\,\pm\,0.063}$ & $0.140{\scriptstyle\,\pm\,0.018}$ & $0.113{\scriptstyle\,\pm\,0.016}$ & $0.851{\scriptstyle\,\pm\,0.014}$ & $0.852{\scriptstyle\,\pm\,0.013}$ \\
    Isomap         & $0.607{\scriptstyle\,\pm\,0.038}$ & $0.630{\scriptstyle\,\pm\,0.037}$ & $0.155{\scriptstyle\,\pm\,0.001}$ & $0.139{\scriptstyle\,\pm\,0.001}$ & $0.665{\scriptstyle\,\pm\,0.005}$ & $0.663{\scriptstyle\,\pm\,0.004}$ \\
    Metric MDS     & $0.654{\scriptstyle\,\pm\,0.067}$ & $0.663{\scriptstyle\,\pm\,0.065}$ & $0.058{\scriptstyle\,\pm\,0.010}$ & $0.040{\scriptstyle\,\pm\,0.009}$ & $0.582{\scriptstyle\,\pm\,0.012}$ & $0.574{\scriptstyle\,\pm\,0.013}$ \\
    PHATE          & $0.555{\scriptstyle\,\pm\,0.077}$ & $0.557{\scriptstyle\,\pm\,0.078}$ & $0.725{\scriptstyle\,\pm\,0.011}$ & $0.703{\scriptstyle\,\pm\,0.016}$ & $0.732{\scriptstyle\,\pm\,0.006}$ & $0.734{\scriptstyle\,\pm\,0.006}$ \\
    Diffusion Maps & $0.650{\scriptstyle\,\pm\,0.046}$ & $0.672{\scriptstyle\,\pm\,0.037}$ & $0.325{\scriptstyle\,\pm\,0.002}$ & $0.339{\scriptstyle\,\pm\,0.002}$ & $0.613{\scriptstyle\,\pm\,0.004}$ & $0.678{\scriptstyle\,\pm\,0.008}$ \\
    HeatGeo        & $0.636{\scriptstyle\,\pm\,0.084}$ & $0.640{\scriptstyle\,\pm\,0.085}$ & $\mathbf{0.730}{\scriptstyle\,\pm\,0.006}$ & $0.633{\scriptstyle\,\pm\,0.005}$ & $0.842{\scriptstyle\,\pm\,0.002}$ & $0.840{\scriptstyle\,\pm\,0.002}$ \\
    DTNE           & $0.676{\scriptstyle\,\pm\,0.073}$ & $0.685{\scriptstyle\,\pm\,0.072}$ & $0.691{\scriptstyle\,\pm\,0.003}$ & $\mathbf{0.707}{\scriptstyle\,\pm\,0.001}$ & $0.733{\scriptstyle\,\pm\,0.002}$ & $0.734{\scriptstyle\,\pm\,0.005}$ \\
    \midrule
    \entropath      & $\mathbf{0.707}{\scriptstyle\,\pm\,0.049}$ & $\mathbf{0.717}{\scriptstyle\,\pm\,0.047}$ & $0.514{\scriptstyle\,\pm\,0.016}$ & $0.443{\scriptstyle\,\pm\,0.014}$ & $0.785{\scriptstyle\,\pm\,0.001}$ & $0.786{\scriptstyle\,\pm\,0.002}$ \\
    \bottomrule
  \end{tabular}
\end{table}

\begin{figure}[t]
  \centering
  \includegraphics[width=\linewidth]{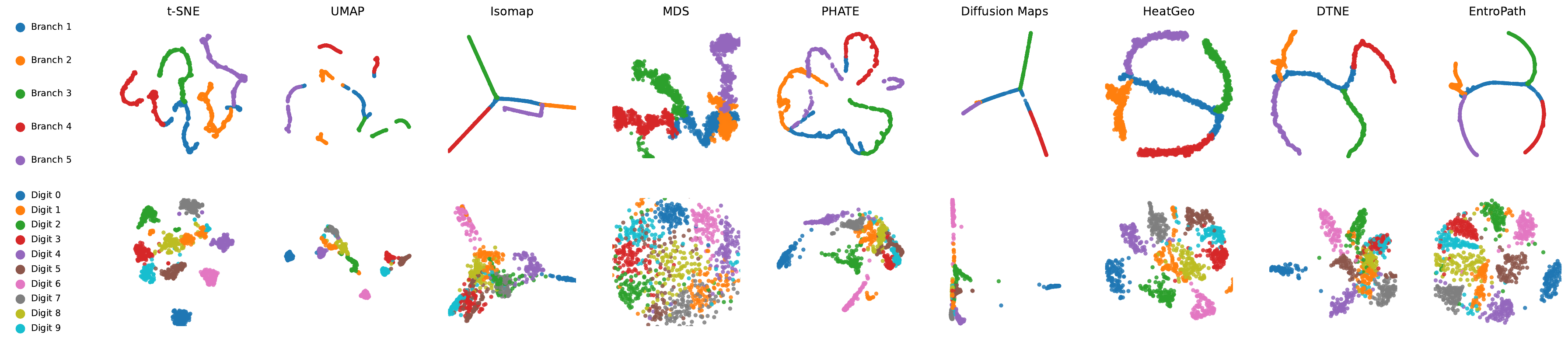}
  \caption{Two-dimensional embeddings on the Tree (top) and MNIST (bottom),
  coloured by ground-truth class. \entropath and DTNE render the Tree's branches
  as clean one-dimensional filaments, whereas the neighbour-embedding methods
  broaden them; this fidelity to the underlying one-dimensional geometry is the
  visual counterpart of \entropath's leading Tree scores in
  Table~\ref{tab:clustering-main}.}
  \label{fig:clustering-grid}
\end{figure}

\paragraph{Results.}
Table~\ref{tab:clustering-main} shows a clear and interpretable pattern across
the continuous-to-discrete spectrum. On the \emph{Tree}, a purely geometric
target, \entropath attains the highest homogeneity and aMI; it is effectively
tied with UMAP within seed-to-seed variance and ahead of DTNE, and among these
competitive methods it has the lowest variance across seeds ($\pm0.049$ versus
$\pm0.06$--$0.09$ for UMAP, DTNE, and HeatGeo). On \emph{MNIST}, \entropath is
competitive (fourth of nine), behind the neighbour-embedding methods t-SNE and
UMAP and the diffusion method HeatGeo, but ahead of PHATE, DTNE, and the
remaining baselines. On \emph{PBMC}, the most cluster-like dataset, \entropath is
mid-table, comparable to t-SNE on homogeneity (marginally behind on aMI) and
ahead of UMAP, Isomap, Diffusion Maps, and MDS, while HeatGeo, PHATE, and DTNE
lead.

This profile is exactly what one expects from a geodesic-preservation method.
\entropath preserves the continuous structure of the data rather than contracting
it into well-separated point clouds; on PBMC this keeps related cell types on a
connected continuum that a hard $k$-means partition then fragments, which a
partition-based score penalises. The same property is an advantage on the Tree,
where the target structure \emph{is} a continuum. Notably, among the geometry-
and trajectory-oriented methods, \entropath leads on the continuous manifold while
DTNE leads on the discrete single-cell data---the two trade places according to
how cluster-like the data are. Overall, the geometric objective does not
compromise general clustering competitiveness: \entropath is at or near the top on
continuous structure and remains a reasonable choice on discrete data, with no
task-specific tuning.

\subsection{Ablation studies}
\label{app:ablation}


\subsubsection{Landmark approximation}
\label{app:ablation-lm}

For $N$ beyond a few thousand, \entropath replaces the full $N\times N$ MDS with a landmark
(Nystr\"om) approximation: $M$ landmarks are embedded and the remaining points projected onto
them. We ablate the landmark count $M$ and the selection rule on the Swiss roll ($N=5000$,
10 seeds), forcing the landmark path below the usual threshold to probe the tradeoff, and
report the row-wise Spearman fidelity of the full embedding against the analytic geodesic
(\Cref{fig:ablation-lm-quality}).

Fidelity is governed by a \emph{complexity threshold} rather than a fixed fraction of $N$: the
spiral requires $M\approx2000$ landmarks to resolve. Below this the landmark set undersamples
the manifold and fidelity collapses ($\rho\approx0.49$ at $M\le500$, where the embedding
degenerates into a folded disc); at the design point $M=2000$
FPS recovers $\rho=0.94\pm0.01$, $95\%$ of the full-rank value ($\rho=0.99$); and by $M=4000$
all selectors are within noise of full-rank. Trustworthiness tracks the same curve ($0.98$ at
$M=2000$ vs.\ $0.997$ full-rank). Because the threshold is set by manifold complexity, the same
$M=2000$ that is $40\%$ of the points here is only $4\%$ at $N=50{,}000$, where it still yields
a faithful unrolling (\Cref{fig:scalability}); the relative speedup therefore grows with $N$.

Among selection rules, FPS $\ge$ $k$-means $\ge$ random holds throughout, with FPS the most
accurate and the most stable at the design point ($\rho=0.94\pm0.01$ for FPS, $0.91\pm0.02$ for
$k$-means, $0.83\pm0.06$ for random); we use FPS by default. At $M=2000$ the landmark path is
$12\times$ faster than full-rank ($8$\,s vs.\ $93$\,s; \Cref{fig:ablation-lm-runtime}), with the
speedup diminishing as $M\to N$ (the landmark MDS cost approaches full-rank by $M=4000$), so
$M\approx2000$ is the quality/speed sweet spot.

\begin{figure}[t]
  \centering
  \begin{subfigure}{0.42\linewidth}
    \includegraphics[width=\linewidth]{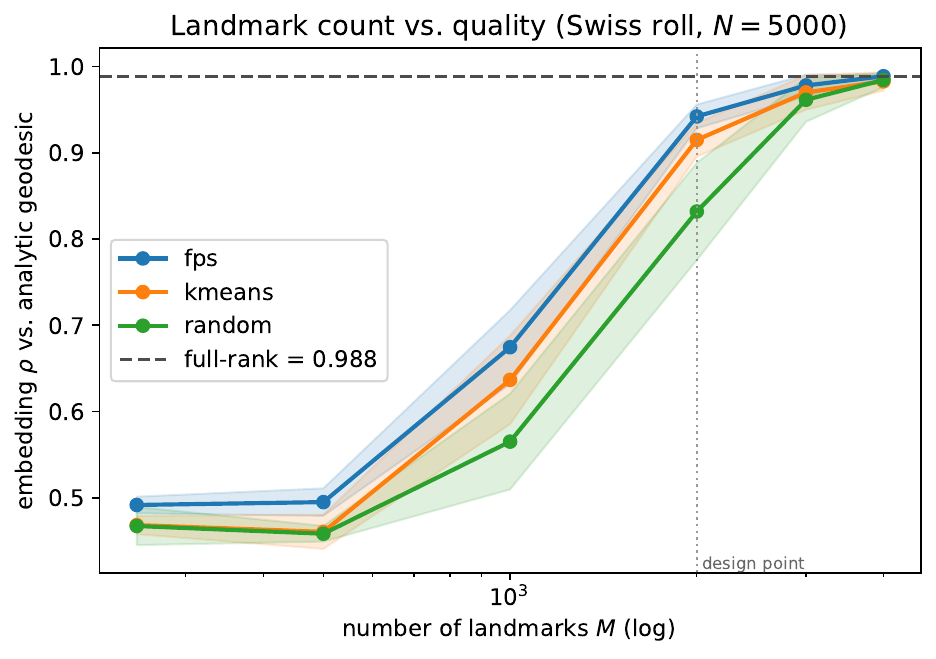}
    \caption{}\label{fig:ablation-lm-quality}
  \end{subfigure}\hspace{0.02\linewidth}
  \begin{subfigure}{0.42\linewidth}
    \includegraphics[width=\linewidth]{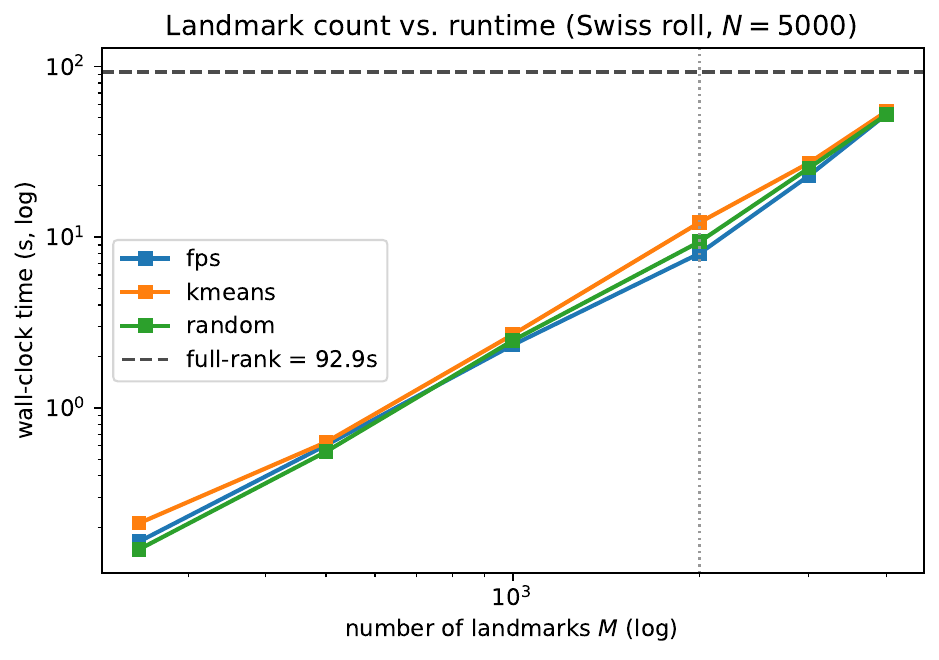}
    \caption{}\label{fig:ablation-lm-runtime}
  \end{subfigure}
  \caption{\textbf{Landmark ablation on the Swiss roll} ($N=5000$, 10 seeds; bands $\pm1$ s.d.).
    \textbf{(a)} Embedding fidelity (row-wise Spearman vs.\ the analytic geodesic) against the
    landmark count $M$, for FPS, $k$-means, and random selection, with the full-rank ceiling
    (dashed) and the $M=2000$ design point (dotted). \textbf{(b)} Wall-clock time on a log
    scale; the landmark path is an order of magnitude faster than full-rank at the design
    point, with the advantage eroding as $M\to N$.}
  \label{fig:ablation-lm}
\end{figure}

\begin{figure*}[t]
  \centering
  \includegraphics[width=0.55\linewidth]{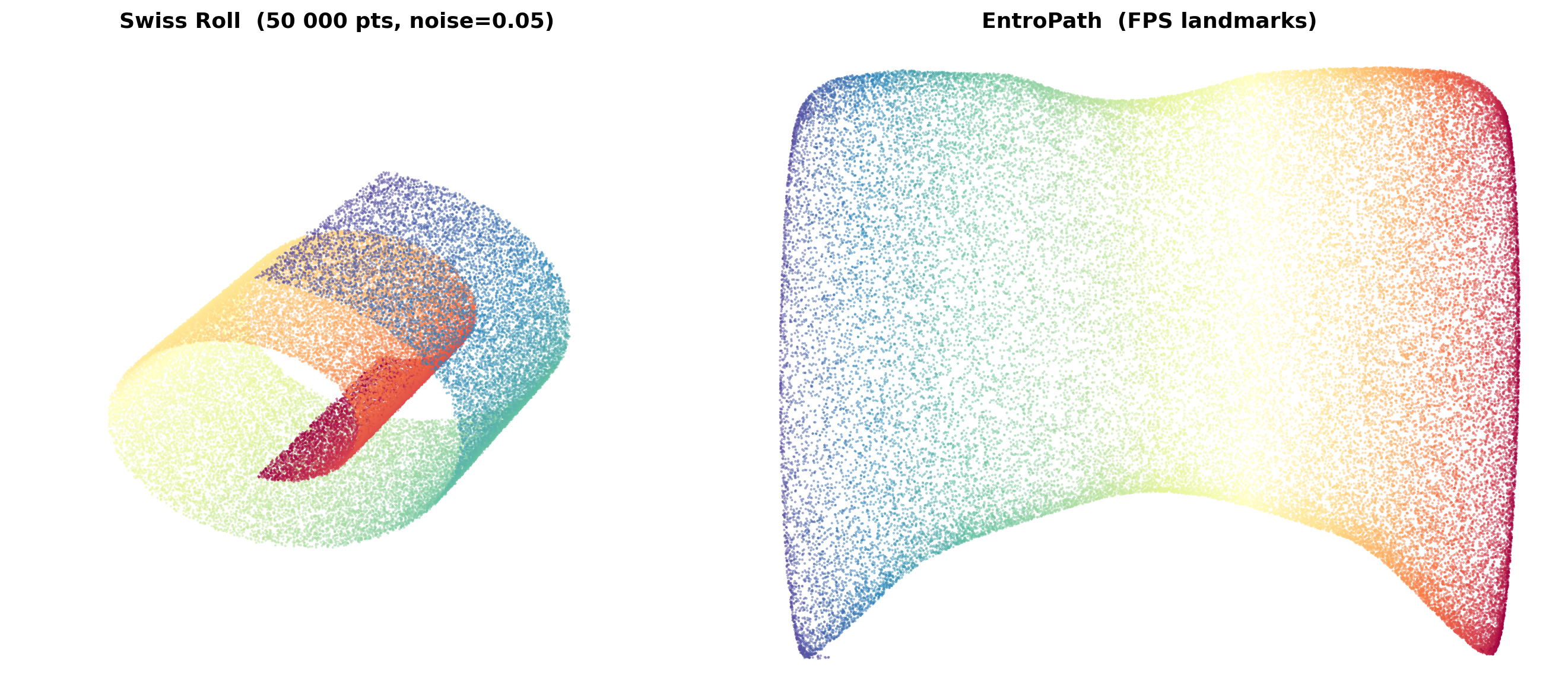}
  \caption{\textbf{Landmark embedding at scale.} A $50{,}000$-point Swiss roll
    (left, coloured by the spiral parameter; Gaussian noise $\sigma=0.05$) and its
    full \entropath embedding (right). $M=2000$ FPS landmarks (only $4\%$ of the
    points) are embedded directly, with the remaining points placed by the
    out-of-sample projection of \Cref{sec:landmarks}, at a scale where full-rank
    MDS is infeasible. The manifold remains well unrolled and the spiral parameter
    is recovered as a monotone gradient across all $50{,}000$ points, empirical
    evidence that the landmark budget of $M=2000$ (\Cref{app:ablation-lm}) remains
    effective at this scale.}
  \label{fig:scalability}
\end{figure*}

%

\subsubsection{Sensitivity to graph connectivity \texorpdfstring{$k_{\mathrm{NN}}$}{kNN}}
\label{app:ablation-knn}

All graph-based methods share a single neighbourhood size $k_{\mathrm{NN}}$,
which for \entropath also sets the adaptive bandwidth $\sigma_i$ (the $k$-th
nearest-neighbour distance) and hence the diffusion scale. We vary
$k_{\mathrm{NN}}\in\{5,10,15,20\}$ on the uniform and non-uniform swiss roll
($N=2{,}000$, 30 seeds) and report the row-wise Spearman correlation between the
embedding and the analytic geodesic (\Cref{fig:ablation-knn},
\Cref{tab:ablation-knn}).

On the \textbf{non-uniform} roll --- the regime that stresses density
heterogeneity --- \entropath is the best diffusion method at \emph{every}
$k_{\mathrm{NN}}$, and its margin over the next-best method (DTNE) grows from
$k_{\mathrm{NN}}=5$ to $10$ and then holds ($0.774,0.854,0.875,0.877$ vs.\
DTNE's $0.764,0.773,0.805,0.820$). The advantage is therefore a property of the
method, not of a particular connectivity. On the \textbf{uniform} roll, where
there is no density challenge, \entropath is best at small $k_{\mathrm{NN}}$
($5,10$); at $k_{\mathrm{NN}}=15,20$ the spectral embedding of Diffusion Maps
edges ahead at the embedding level. The on-thesis effect nonetheless persists:
among the random-walk methods, \entropath is far more robust to the shortcut
edges that appear at large $k_{\mathrm{NN}}$ --- DTNE collapses from $0.954$ to
$0.687$ and HeatGeo from $0.898$ to $0.683$ between $k_{\mathrm{NN}}=10$ and
$20$, whereas \entropath retains $0.813$, because MERW down-weights the
low-degree bridge edges responsible for the shortcuts
(cf.\ Remark~\ref{rem:bottleneck}).

At the distance-matrix level (before the embedding step, where the walk and
distance act directly) \entropath has the best matrix of all diffusion methods at
$k_{\mathrm{NN}}\ge 10$ on \emph{both} datasets; Diffusion Maps' strong uniform
embedding does not come from a better distance matrix (its distance-level Spearman is
only $0.13$--$0.57$) but from its spectral coordinates unrolling the easy
manifold. The method ranking on the non-uniform roll is unchanged across
$k_{\mathrm{NN}}\in\{5,\dots,20\}$, so the fixed choice $k_{\mathrm{NN}}=15$ used
in the main experiments is not a tuned parameter; it sits at the non-uniform
performance plateau, matches the single-cell setting, and --- since \entropath is
only second on the uniform roll at that value --- does not favour our method.

\begin{figure}[t]
  \centering
  \includegraphics[width=0.7\linewidth]{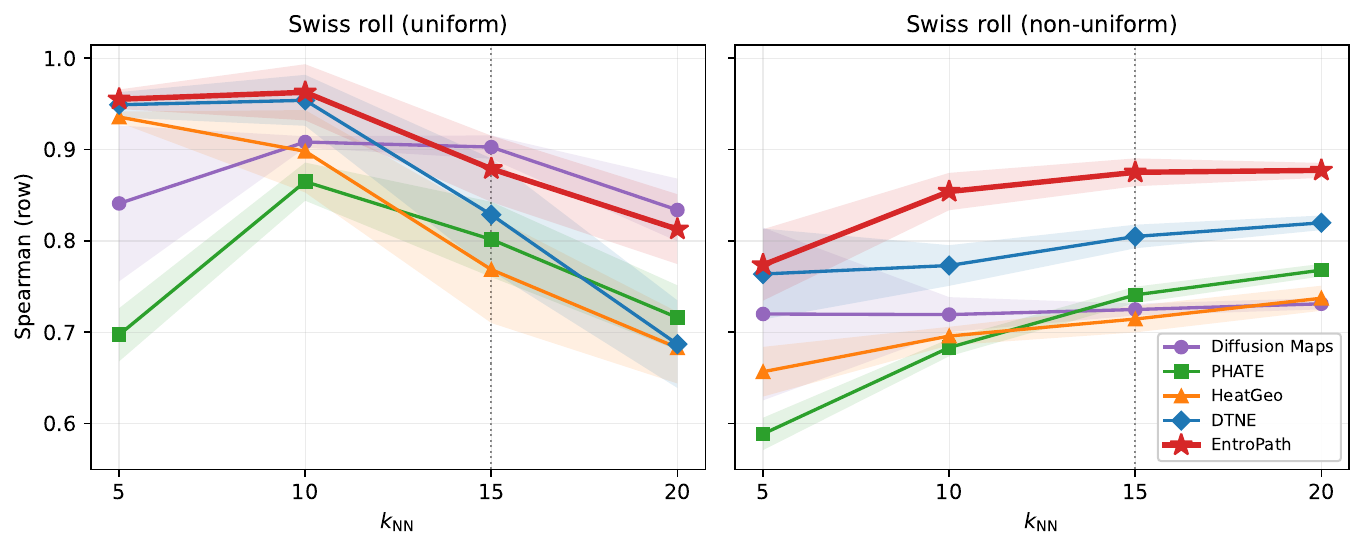}
  \caption{Row-wise Spearman correlation with the analytic geodesic vs.\ graph
    connectivity $k_{\mathrm{NN}}$ (mean over 30 seeds, shaded $\pm$\,1 std), for
    the uniform (left) and non-uniform (right) swiss roll. The dotted line marks
    the fixed value $k_{\mathrm{NN}}=15$. On the non-uniform roll \entropath leads
    at every connectivity; on the uniform roll it leads at small $k_{\mathrm{NN}}$
    and, among the random-walk methods, degrades least as shortcut edges appear
    at $k_{\mathrm{NN}}=20$.}
  \label{fig:ablation-knn}
\end{figure}

\begin{table}[t]
  \centering
  \caption{Embedding-level row-wise Spearman correlation with the analytic
    geodesic vs.\ $k_{\mathrm{NN}}$ (mean over 30 seeds). \textbf{Bold} = best
    diffusion method per column. \entropath leads the non-uniform regime at every
    $k_{\mathrm{NN}}$.}
  \label{tab:ablation-knn}
  \begin{tabular}{lcccccccc}
    \toprule
     & \multicolumn{4}{c}{Uniform} & \multicolumn{4}{c}{Non-uniform} \\
    \cmidrule(lr){2-5}\cmidrule(lr){6-9}
    Method & $k{=}5$ & $k{=}10$ & $k{=}15$ & $k{=}20$
           & $k{=}5$ & $k{=}10$ & $k{=}15$ & $k{=}20$ \\
    \midrule
    Diffusion Maps & 0.841 & 0.908 & \textbf{0.903} & \textbf{0.834} & 0.720 & 0.719 & 0.725 & 0.731 \\
    PHATE          & 0.697 & 0.865 & 0.802 & 0.716 & 0.589 & 0.683 & 0.741 & 0.768 \\
    HeatGeo        & 0.936 & 0.898 & 0.769 & 0.683 & 0.657 & 0.696 & 0.714 & 0.737 \\
    DTNE           & 0.949 & 0.954 & 0.829 & 0.687 & 0.764 & 0.773 & 0.805 & 0.820 \\
    \entropath      & \textbf{0.955} & \textbf{0.963} & 0.879 & 0.813 & \textbf{0.774} & \textbf{0.854} & \textbf{0.875} & \textbf{0.877} \\
    \bottomrule
  \end{tabular}
\end{table}

\subsubsection{Kernel profile (Gaussian vs.\ \texorpdfstring{$\alpha$}{alpha}-decay)}
\label{app:ablation-kernel}
\entropath supports two affinity kernels: an adaptive Gaussian kernel
$A_{ij}=\exp(-\|x_i-x_j\|^2/\sigma_i\sigma_j)$ and an $\alpha$-decay kernel
$A_{ij}=\exp(-(d_{ij}/\sigma_i)^{\alpha})$ with a large exponent ($\alpha=40$),
which gives a sharper, near-box profile within the adaptive bandwidth. On the
synthetic manifolds the two perform similarly. Where they differ, the $\alpha$-decay
kernel yields marginally higher global rank correlation and marginally lower
trustworthiness---consistent with its flatter in-bandwidth weighting placing
relatively more weight on the longer-range edges \emph{within} the neighbourhood
graph than the peaked Gaussian, nudging global coherence upward at a small cost to
the tightest local neighbourhoods (\Cref{tab:ablation-kernel}).

Both kernels are evaluated under the same embedding-level protocol as the main synthetic
benchmarks (\Cref{sec:exp-synthetic}): on three manifolds with
heterogeneous geometry---non-uniform Swiss roll, sphere, and dense tree---we
report the row-wise Spearman correlation between embedding distances and the
$k$-nearest-neighbour shortest-path geodesic (global fidelity), together with
trustworthiness (local neighbourhood preservation), at $k_{\mathrm{NN}}=15$ over
30 seeds. Only the \entropath kernel is varied; every other pipeline setting is
held fixed.

\begin{table}[t]
  \centering
  \caption{\entropath with the Gaussian vs.\ $\alpha$-decay kernel: row-wise
    Spearman against the shortest-path geodesic and trustworthiness
    ($k_{\mathrm{NN}}=15$, mean over 30 seeds). The kernel trades local fidelity
    (trustworthiness) against global coherence (Spearman); the net effect across
    datasets is a wash.}
  \label{tab:ablation-kernel}
  \begin{tabular}{llcc}
    \toprule
    Dataset & Kernel & Spearman (row) & Trustworthiness \\
    \midrule
    \multirow{2}{*}{Swiss roll (non-uniform)}
      & Gaussian          & 0.890 & \textbf{0.981} \\
      & $\alpha$-decay     & \textbf{0.896} & 0.971 \\
    \addlinespace
    \multirow{2}{*}{Sphere}
      & Gaussian          & 0.762 & \textbf{0.904} \\
      & $\alpha$-decay     & \textbf{0.764} & 0.882 \\
    \addlinespace
    \multirow{2}{*}{Tree (dense)}
      & Gaussian          & \textbf{0.785} & 0.879 \\
      & $\alpha$-decay     & 0.784 & \textbf{0.899} \\
    \bottomrule
  \end{tabular}
\end{table}

The observed differences are small and do not change \entropath's relative
standing on any dataset, so the kernel is not a performance-critical choice. We
therefore choose the kernel on theoretical grounds rather than by empirical
tuning: the \textbf{Gaussian} kernel is used for all synthetic experiments and
for the geodesic theory, since the graph-Laplacian consistency underlying the
geodesic recovery result (\Cref{thm:geodesic}) is established for the Gaussian
profile and its finite moments; the adaptive \textbf{$\alpha$-decay} kernel is
used for the single-cell data (\Cref{sec:singlecell}), where heavy-tailed
connectivity is needed for robustness across scales and no ground-truth geodesic
is recovered. Thus each application regime uses a single fixed kernel, with no
per-dataset hyperparameter tuning.

\subsubsection{Robustness to input noise}
\label{app:ablation-noise}

A positioning claim in \Cref{sec:related} is that \entropath's path ensemble
smooths shortcut-inducing perturbations while still approximating the same
underlying geodesic geometry in the short-time regime of \Cref{thm:geodesic}:
because the free-energy dissimilarity aggregates over many paths (a log-sum-exp),
local perturbations to individual edges contribute through many alternative paths
rather than through a single optimal path, whereas methods built on a single
shortest path (Isomap, Shortest Path) are sensitive to the spurious short-circuit
edges that noise introduces. We test this directly on the uniform Swiss roll,
sweeping the input noise level $\sigma\in\{0,0.05,0.1,0.25,0.5,1.0\}$ over 30
seeds; $\sigma=1.0$ is $\sim\!10\%$ of the manifold's arc-length scale, beyond
which the local neighbourhood structure is destroyed for every method.

The evaluation uses the \emph{analytic} geodesic, computed in closed form from the
true $(t,h)$ parameters and therefore completely immune to the input noise. This
isolates robustness of the distance estimator itself: unlike a graph
shortest-path ground truth, the reference does not degrade as the input graph
becomes noisier, so any drop in fidelity as $\sigma$ rises is purely method
degradation, not ground-truth degradation. The uniform sampling further isolates
the noise axis from density heterogeneity. We report the distance-level (Level-1)
row-wise Spearman correlation between each method's dissimilarity matrix and the
analytic geodesic, which probes the distance representation directly, without the
embedding step (\Cref{fig:ablation-noise}).

\entropath attains the highest distance-level fidelity at \emph{every} noise level
and degrades the most gracefully. From the noiseless baseline to the heavy-noise
point $\sigma=0.5$ it retains $88\%$ of its correlation ($0.878\to0.774$), against
$78\%$ for the single-path methods ($0.834\to0.650$), with the remaining
diffusion-distance methods (DTNE, PHATE, HeatGeo) in between ($83$--$86\%$). Its
margin over the single-path methods \emph{widens} through the mid-noise range
where shortcut edges proliferate---from $+0.04$ at $\sigma=0$ to $+0.12$ at
$\sigma=0.5$---before all methods converge toward the Euclidean floor
($\rho\!\approx\!0.53$) at $\sigma=1.0$, where the manifold is no longer locally
recoverable by anyone. This is consistent with the path-ensemble argument: because
the free-energy dissimilarity aggregates contributions from many paths rather than
relying on a single shortest path, it is less sensitive to perturbations that
alter any individual path, so the advantage is largest in the regime where
shortcut edges---but not yet total structural collapse---dominate. This experiment
isolates robustness to observation noise from robustness to density heterogeneity;
the latter is addressed separately in the non-uniform Swiss roll experiments
(\Cref{app:swiss-roll-full}).

\begin{figure}[t]
  \centering
  \includegraphics[width=0.4\linewidth]{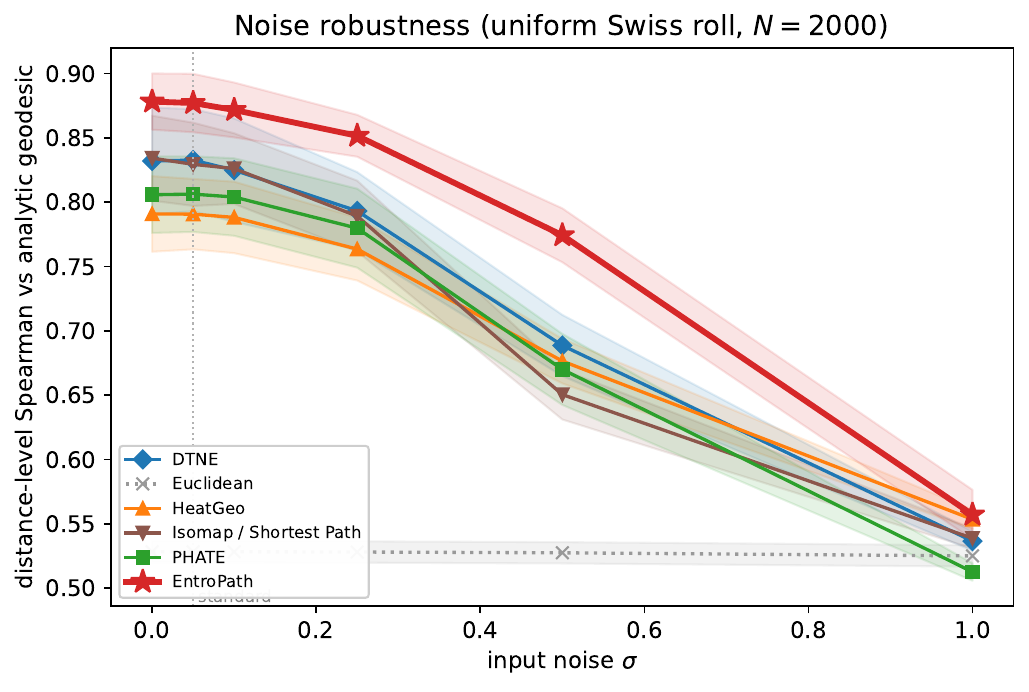}
  \caption{\textbf{Robustness to input noise} (uniform Swiss roll, $N=2000$, 30 seeds;
    bands $\pm1$ s.d.). Distance-level row-wise Spearman between each method's distance
    matrix and the noise-immune analytic geodesic, versus input noise $\sigma$. \entropath
    is most accurate at every $\sigma$ and degrades most gracefully; its lead over the
    single-path methods (Isomap and Shortest Path, which share one distance matrix at this
    level and so appear as a single curve) widens through the mid-noise range before all
    methods converge toward the Euclidean floor as local structure is destroyed. Diffusion
    Maps is omitted: as a spectral-embedding method its diffusion-distance matrix does not
    track the geodesic, so its Level-1 curve sits near the Euclidean floor and is
    uninformative.}
  \label{fig:ablation-noise}
\end{figure}

At the embedding level (Level-2) the ranking is broadly the same, with one exception:
the spectral embedding of Diffusion Maps carries its own noise robustness in the
mid-noise range---its low-frequency eigenvectors act as a smoother---so it leads on
Level-2 for $\sigma\le0.5$ before collapsing below all methods at $\sigma=1.0$. This is
a property of the embedding, not of geodesic recovery; among methods that reconstruct a
geodesic \emph{distance}, which is the quantity \entropath is designed to produce,
\entropath remains the most robust.

\subsection{Single-cell additional experiments}
\label{app:singlecell}

\subsubsection{Pseudotime analysis}
\label{app:pseudotime}


To infer a continuous cell progression along a developmental trajectory, we define
a pseudotime $\tau:V\to[0,1]$ based on MERW diffusion potentials from
user-specified root cells.

\paragraph{Root-based diffusion potential.}
Let $\mathcal{R}\subset V$ be a set of root cells.
Define
\begin{equation}\label{eq:droot}
  d_{\mathrm{root}}(x)
  = -\log\!\left(\sum_{r\in\mathcal{R}}T^k_{rx}\right),
\end{equation}

where $d_{\mathrm{root}}(x)$ is the free energy of the MERW path ensemble from
$\mathcal{R}$ to $x$ (\Cref{eq:free-energy}): the negative log partition function
over all $k$-step paths from the root set to $x$. We call this scalar field the
\emph{diffusion potential}; cells connected to the root by high-probability MERW
paths have lower potential and are assigned earlier pseudotime.

\paragraph{Normalisation.}
Two strategies are supported:
\begin{itemize}
  \item \textit{Min-max:} $\;\tau(x)=
    \bigl(d_{\mathrm{root}}(x)-\min d_{\mathrm{root}}\bigr)/
    \bigl(\max d_{\mathrm{root}}-\min d_{\mathrm{root}}\bigr)$.
  \item \textit{Rank:} $\;\tau(x)=\mathrm{rank}\bigl(d_{\mathrm{root}}(x)\bigr)/(N-1)$,
    which is more robust to density variations and outliers.
\end{itemize}

\paragraph{Landmark approximation.}

For large $N$, let $\mathcal{L}$ denote a set of $M \ll N$ landmark points. Compute diffusion on the landmark graph yielding $T^k_{\mathcal{L}}\in\R^{M\times M}$, and use the projection operator $P \in \mathbb{R}^{N \times M}$ from
\Cref{sec:landmarks}.
Root cell $r$ is represented \emph{softly} over landmarks via $P(r,:)$, and the
landmark diffusion potential is
\begin{equation}\label{eq:droot-lm}
  d^{\mathcal{L}}_{\mathrm{root}}(y)
  = -\log\!\left(
      \sum_{r\in\mathcal{R}}\sum_{z=1}^M P(r,z)\,T^k_{\mathcal{L}}(z,y)
    \right), \quad y\in\{1,\ldots,M\}.
\end{equation}
After normalising to landmark pseudotimes $\tau_{\mathcal{L}}(y)$, the full-data
pseudotime is recovered by
\[
  \tau(x) = \sum_{y=1}^M P(x,y)\,\tau_{\mathcal{L}}(y).
\]
The soft root representation avoids artefacts from hard landmark assignment and
ensures consistency between full-data and approximate computations.

The landmark approximation reduces runtime on large datasets ($n>2000$) while
preserving embedding quality, as required for the scalability results of
\Cref{sec:experiments}.


\paragraph{Terminal correction (optional).}
Given terminal cells $\mathcal{T}\subset V$, compute analogously
$d_{\mathrm{term}}(x)=-\log\bigl(\sum_{s\in\mathcal{T}}T^k_{sx}\bigr)$,
normalised to $[0,1]$.
The corrected pseudotime
\[
  \tau'(x) = \frac{\tau(x)}{\tau(x)+d_{\mathrm{term}}(x)},
\]
re-normalised to $[0,1]$, pulls late cells towards the terminal set while
preserving the root-anchored ordering.

\begin{remark}[Pseudotime as geodesic distance]
By \Cref{thm:geodesic}, in the short-time regime $t=k/\lmax\to 0$ the diffusion
potential \eqref{eq:droot} satisfies
$d_{\mathrm{root}}(x)\propto\min_{r\in\mathcal{R}}\dM(r,x)^2/4t$,
so pseudotime is proportional to the squared geodesic distance from the nearest
root on the underlying manifold.
\end{remark}

\paragraph{DEMaP protocol.}
We follow the DEMaP protocol of \citet{Moon2019} as configured by DTNE
\citep{Wei2025}. For each dataset we draw random subsamples and, on each,
compute the Spearman correlation between shortest-path geodesic distances on a
$k$-NN graph of the high-dimensional reference and Euclidean distances in the
two-dimensional embedding, reporting the mean over repetitions. Following
DTNE's per-dataset settings, the subsample size is $500$ for the smaller
datasets (Paul15, Nestorowa, Pancreas) and $2{,}000$ for the larger ones
(Lymphoid, Embryoid Body, root atlas), with $50$ repetitions throughout; the
geodesic $k$-NN matches DTNE's correlation notebook per dataset. Two
high-dimensional references are used, by design: trustworthiness and
continuity are computed against the representation each method received as
input (the PCA projection), measuring whether a method preserved what it saw;
DEMaP is computed against the strictest available reference, which for
Embryoid Body and Lymphoid is the raw sqrt-transformed expression (full LSI),
aligning our DEMaP with DTNE's published appendix, and coincides with the
input representation on the remaining datasets. All methods on a given dataset
are scored identically, so within-dataset comparisons are exact; the
per-dataset reference differences mean DEMaP values are not directly
comparable \emph{across} datasets.

\subsubsection{Pseudotime: empirical comparison}
\label{app:pseudotime_comparison}

The pseudotime construction of Section~\ref{app:pseudotime} induces a cell
ordering as the free-energy dissimilarity from a designated root cell. Here we assess
that ordering empirically against reference orderings and against DTNE, the
strongest manifold-distance baseline. For the developmental datasets that lack a
temporal annotation (Paul15, Nestorowa, Pancreas) the reference is the normalized
shortest-path (NSP) distance to the root on the high-dimensional $k$NN graph; for
Embryoid Body it is the annotated developmental stage; for the root atlas it is
the consensus developmental ordering of \citet{Shahan2022}. We report Spearman
and Kendall rank correlations (Pearson in Table~\ref{tab:app-pseudotime}).

\paragraph{\entropath and DTNE are concordant.} Across datasets the two orderings
are highly correlated with each other (pairwise Spearman $0.947$ on Paul15) and
similar against the reference, with DTNE ahead on every dataset --- marginally on
Paul15, Nestorowa, Lymphoid, and Embryoid Body ($\le 0.012$), and by a larger
margin on Pancreas ($0.041$) and the root atlas ($0.056$)
(Table~\ref{tab:app-pseudotime}). Because \entropath and DTNE share the same
free-energy dissimilarity and differ only in the underlying random walk---maximum
entropy versus standard---this concordance indicates that the maximum entropy
random walk recovers essentially the same developmental ordering while improving
the geometric fidelity of the embedding (Table~\ref{tab:sc-geometry}).
Figure~\ref{fig:paul15-pseudotime} illustrates the two orderings on Paul15.

\begin{figure}[t]
\centering
\includegraphics[width=0.7\textwidth]{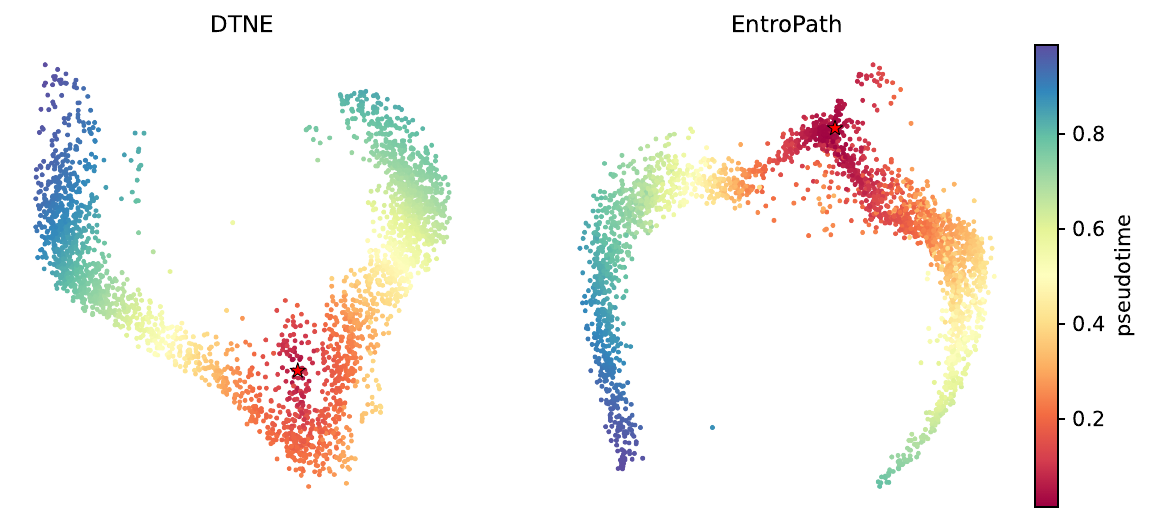}
\caption{Root-anchored pseudotime on Paul15 for DTNE (left) and \entropath (right), coloured by inferred pseudotime (early to late); the red star marks the
root cell. The two orderings are highly concordant (Spearman $0.947$), consistent
with the shared free-energy dissimilarity.}
\label{fig:paul15-pseudotime}
\end{figure}

\paragraph{Limitation.} \entropath does not lead this metric, and the reason is
specific: the pseudotime is read from the same diffusion operator $T^k$ used for
the embedding, but the diffusion power $k$ that is optimal for embedding quality
over-diffuses for \emph{ordering}---on the root atlas, ordering quality peaks at
a markedly lower power than the embedding-selected one and degrades beyond it.
A diffusion power selected specifically for ordering---decoupled from the
embedding and validated independently of the reference---may reduce this gap; we
leave such decoupled optimisation to future work and report all results at the
embedding-selected power, without tuning against the reference.

\begin{table}[t]
\centering
\caption{Pseudotime rank correlations with reference orderings (normalized
shortest-path distance to the root for Paul15, Nestorowa, Pancreas, Lymphoid, and
Embryoid Body; root geodesic for the root atlas); higher is better.}
\label{tab:app-pseudotime}
\begin{tabular}{lcccc}
\toprule
 & \multicolumn{2}{c}{DTNE} & \multicolumn{2}{c}{EntroPath} \\
\cmidrule(lr){2-3}\cmidrule(lr){4-5}
Dataset    & Spearman & Kendall & Spearman & Kendall \\
\midrule
Paul15     & 0.963 & 0.827 & 0.951 & 0.825 \\
Nestorowa  & 0.969 & 0.859 & 0.959 & 0.834 \\
Pancreas   & 0.992 & 0.924 & 0.951 & 0.832 \\
Lymphoid   & 0.651 & 0.492 & 0.645 & 0.461 \\
Embryoid   & 0.955 & 0.825 & 0.951 & 0.812 \\
Root atlas\footnotemark & 0.932 & 0.788 & 0.876 & 0.697 \\
\bottomrule
\end{tabular}
\end{table}
\footnotetext{Against the consensus developmental ordering of
\citet{Shahan2022}, \entropath attains Spearman $0.740$ / Kendall $0.545$ and DTNE
$0.811$ / $0.613$; we report the root-geodesic reference here for consistency
with the other datasets.}

\section{Spectral basis of time selection}
\label{app:appendix-spectral-basis}

Our entropy-based selection of the diffusion power $k$ follows the
protocol introduced by PHATE~\citep{Moon2019}: compute the von
Neumann entropy of the diffusion operator at varying powers, and
identify the knee point of the resulting entropy curve. A subtle
question is which spectral object the entropy should be computed from.
We compute it from the singular values of the row-stochastic MERW
transition matrix
\begin{equation}
  T \;=\; D_\psi^{-1} \tilde{A} D_\psi,
  \qquad D_\psi = \mathrm{diag}(\psi),
\end{equation}
where $\psi$ is the MERW stationary measure satisfying
$\tilde{A}\psi = \lmax\psi$. This differs from computing the
entropy directly from the eigenvalues of the symmetric kernel
$\tilde{A}$, in a way worth making explicit.

\paragraph{Implicit symmetrization.}
Computing entropy from the singular values of $T^k$ admits an
equivalent formulation as eigenvalues of a symmetric operator. For
any matrix $T$ with singular value decomposition $T = U\Sigma V^\top$,
\begin{equation}
  T T^\top \;=\; U \Sigma^2 U^\top,
  \qquad
  \sqrt{T T^\top} \;=\; U \Sigma U^\top,
\end{equation}
where $\sqrt{\cdot}$ denotes the principal matrix square root. The
matrix $\sqrt{T T^\top}$ is symmetric positive semidefinite, and its
eigenvalues are precisely the singular values of $T$. Computing
von Neumann entropy from the singular values of $T^k$ is therefore
equivalent to computing it from the eigenvalues of
$\sqrt{T^k (T^k)^\top}$. We use SVD in the implementation to avoid the
explicit matrix square root, but the symmetric reformulation clarifies
what spectral object is actually being analyzed.

\paragraph{Why $T$ rather than $\tilde{A}$.}
The matrices $T$ and $\tilde{A}$ are similar, related by the
diagonal conjugation $T = D_\psi^{-1} \tilde{A} D_\psi$, and therefore
share the same eigenvalue spectrum $\{\mu_m\}$. However, they have
different \emph{singular value} spectra: the singular values of $T$
are determined by $T T^\top$, which is not unitarily equivalent to
$\tilde{A}^2$ when $\psi$ is non-uniform. Specifically, the
conjugation by the non-orthogonal $D_\psi$ stretches the singular
spectrum of $T$ away from the eigenvalue spectrum of $\tilde{A}$, with
the magnitude of the stretch governed by
$\kappa(D_\psi) = \max_i \psi_i / \min_i \psi_i$.

This stretching has a clear geometric interpretation. The singular
values of $T$, equivalently the eigenvalues of $\sqrt{T T^\top}$,
encode not only the spectral decay structure shared with $\tilde{A}$,
but also the heterogeneity of the stationary measure $\psi$, which
itself reflects the local geometry of the graph (high-density regions
correspond to large $\psi_i$). The entropy curve computed from
$\{\sigma_m^k\}$ thus reflects both the multiscale structure of the
kernel and the geometry of the sampling distribution. On uniformly
sampled data ($\psi$ approximately constant), $T$ and $\tilde{A}$ have
nearly identical singular spectra and the two entropy criteria
coincide. On non-uniformly sampled data such as branching trees, the
$T$-based criterion produces more stable knees and yields better
downstream embedding quality.

We retain $\tilde{A}^k$ as the kernel for distance computation, where
the symmetry of $\tilde{A}$ and the direct interpretation via
Varadhan's formula are essential. The asymmetric matrix $T$ is used
only for time selection, where its richer singular spectrum is
informative.

\subsection{Relation to DTNE}
\label{sec:dtne-compare}

DTNE \citep{Wei2025} defines the distance
$D^{\mathrm{DTNE}}_{ij}=-2\log(\hat p_i\cdot\hat p_j)$ using cosine-normalised
modified personalised-PageRank vectors. At the level of the log-kernel
construction, \entropath and DTNE coincide; the essential differences lie in the
underlying diffusion process (MERW versus SRW) and the diffusion horizon (see
\Cref{tab:dtne}). A cosine-normalised variant of our distance,
$-2\log(\hat v_i\cdot\hat v_j)$, matches the DTNE form exactly. We use the
unnormalised version $D_{ij}=-\log(v_i\cdot v_j)$ by default, as it avoids the
diagonal rescaling that can distort distances to peripheral nodes on heterogeneous
graphs. Details on both variants, the kernel-trick derivation of the factor $-2$,
and the Bhattacharyya kernel connection appear in \Cref{app:dtne}.

\begin{table}[t]
\centering
\caption{\entropath and DTNE share the log-kernel distance construction but differ
in the diffusion operator and horizon.}
\label{tab:dtne}
\begin{tabular}{lll}
\toprule
& DTNE & \entropath \\
\midrule
Base walk & SRW: $A_{ij}/d_i$ & MERW: $A_{ij}\psi_j/(\lmax\psi_i)$ \\
Stationary distribution & $\pi_i\propto d_i$ & $\pi_i\propto\psi_i^2$ \\
Diffusion profile & mPPR (infinite-horizon mixture) & $T^k$, finite horizon $k$ \\
Kernel & cosine in $\sqrt{\,\cdot\,}$-space & inner product in spectral space \\
Default distance & $-2\log(\hat p_i\cdot\hat p_j)$ & $-\log(v_i\cdot v_j)$ \\
Geodesic approximation & not established & \Cref{thm:geodesic} \\
\bottomrule
\end{tabular}
\end{table}

\section{Distance variants and the DTNE correspondence}
\label{app:dtne}

\subsection{Computational remarks}
\label{comp:remarks}

In practice there is no need to form the matrix power $T^k$ explicitly. After
computing the top $d_{\mathrm{emb}}$ eigenpairs $(\tilde\mu_\ell,\phi^{(\ell)})$ of
$\tilde A$ (equivalently of $A$, with $\tilde\mu_\ell=\lambda_\ell/\lmax$), the
dissimilarity is
\[
  D_{ij}\approx-\log\sum_{\ell=1}^{d_{\mathrm{emb}}}
    \tilde\mu_\ell^{\,k}\,\phi^{(\ell)}_i\phi^{(\ell)}_j,
\]
which costs $O(n d_{\mathrm{emb}}^2)$ per evaluation after the initial
eigendecomposition. The dominant cost is therefore the sparse eigendecomposition
of $A$, achievable in $O(n\,k_{\mathrm{NN}}\,d_{\mathrm{emb}})$ time via Lanczos or
randomised SVD.

\paragraph{Numerical floor.}
Whichever route is used, entries of $\tilde A^{\,k}$ (or the truncated spectral
sum above) are floored at a small $\epsilon$ (default $10^{-21}$) before the
logarithm, to absorb the negative round-off that the matrix power can produce.
The value is immaterial and plays no modelling role.

\subsection{Two distance variants}
\label{app:variants}

Two natural distances arise from the Gram kernel $S_{ij}=v_i\cdot v_j$
(\Cref{prop:gram}), differing in whether the self-similarity $S_{ii}=\|v_i\|^2$ is
divided out. We work throughout with the positive-mode representation of
\Cref{prop:gram}: the feature vectors $v_i$ of \eqref{eq:vi} range over the
positive modes $\tilde\mu_\ell>0$, so they are real for any $k$ and no global
positive-semidefiniteness of $\tilde A^k$ is required. This is the part of the
spectrum retained by classical MDS; for even $k$, $\tilde A^k\succeq0$ and the
representation holds over the full spectrum.

\begin{definition}[Unnormalised distance]\label{def:Dunn}
\[
  D^{\mathrm{unn}}_{ij}=-\log S_{ij}=-\log(v_i\cdot v_j)\quad(i\ne j),
  \qquad D^{\mathrm{unn}}_{ii}:=0,
\]
the diagonal zeroed before classical MDS as in \Cref{rem:metric}.
\end{definition}

\begin{definition}[Normalised (canonical) distance]\label{def:Dcan}
\[
  D^{\mathrm{norm}}_{ij}
  =-2\log\frac{S_{ij}}{\sqrt{S_{ii}S_{jj}}}
  =-2\log(\hat v_i\cdot\hat v_j),
  \qquad D^{\mathrm{norm}}_{ii}=0\ \text{exactly},
\]
with $\hat v_i=v_i/\|v_i\|$. Here $D^{\mathrm{norm}}_{ii}=0$ holds exactly (not by
convention), since $\hat v_i\cdot\hat v_i=1$.
\end{definition}

\subsection{Kernel-trick interpretation of the normalised distance}
\label{app:kernel-trick}

The factor $-2$ and the cosine normalisation in \Cref{def:Dcan} arise naturally
from viewing the distance as a squared Euclidean distance in a feature space.
Following DTNE \citep{Wei2025}, let $G_{ij}=\hat v_i\cdot\hat v_j$ (so $G_{ii}=1$).
Whenever $[\log G_{ij}]$ is positive semidefinite, there exist feature vectors
$z_i$ with $\langle z_i,z_j\rangle=\log G_{ij}$, whence $\|z_i\|^2=\log G_{ii}=0$
and
\[
  \|z_i-z_j\|^2=-2\log(\hat v_i\cdot\hat v_j)=D^{\mathrm{norm}}_{ij}.
\]
This PSD condition, on the log-kernel $[\log G_{ij}]$, is distinct from the
positive-mode condition on $\tilde A^k$ used in \Cref{app:variants}: the latter
makes the vectors $v_i$ real, the former makes $D^{\mathrm{norm}}$ a squared
Euclidean distance. In general $[\log G_{ij}]$ is not PSD (\Cref{rem:metric}), so
the construction motivates the canonical form rather than certifying a global
Euclidean embedding.

\subsection{Normalised vs.\ unnormalised: when to use which}
\label{app:comparison}

The two variants are related by
\[
  D^{\mathrm{norm}}_{ij}=-2\log(v_i\cdot v_j)+\log\|v_i\|^2+\log\|v_j\|^2.
\]
On heterogeneous graphs the self-similarity
$S_{ii}=\|v_i\|^2=\sum_\ell\tilde\mu_\ell^{\,k}[\phi^{(\ell)}_i]^2$ is dominated by
the Perron mode at large $k$ ($S_{ii}\approx\psi_i^2$) and varies significantly
across nodes: peripheral nodes have small $S_{ii}$, so normalising by
$\sqrt{S_{ii}S_{jj}}$ inflates distances involving them. The unnormalised distance
avoids this distortion, remains numerically stable, and encodes absolute
diffusivity in addition to geometry; we therefore adopt it by default. The
normalised variant is preferable on homogeneous graphs or when the exact
kernel-trick Euclidean interpretation is required.

\begin{center}
\begin{tabular}{p{4.3cm}p{2.4cm}p{2.6cm}p{3.4cm}}
\toprule
Formula & $D_{ii}=0$? & $=\|z_i-z_j\|^2$? & Recommended use \\
\midrule
$-\log S_{ij}$, $D_{ii}:=0$ & before MDS & no &
  heterogeneous graphs (default) \\[4pt]
$-2\log(\hat v_i\cdot\hat v_j)$ & exactly & where log-kernel is PSD &
  homogeneous graphs; kernel interpretation \\
\bottomrule
\end{tabular}
\end{center}

\subsection{Kernel equivalence with DTNE}
\label{app:dtne-kernel}

DTNE \citep{Wei2025} employs modified personalised-PageRank vectors $p_i$ and the
distance $D^{\mathrm{DTNE}}_{ij}=-2\log(\hat p_i\cdot\hat p_j)$. It interprets its
kernel as the Bhattacharyya coefficient
$G^{\mathrm{DTNE}}_{ij}=\sum_u\sqrt{p_i(u)p_j(u)}$, the inner product of the
square-root-probability vectors. Cosine-normalising these recovers the same
algebraic form as our $\hat v_i\cdot\hat v_j$.

Both methods thus rely on cosine similarity in an embedded feature space
(square-root probability for DTNE, spectral for \entropath), differing only in the
underlying walk (SRW vs.\ MERW) and the diffusion horizon (an infinite-horizon
PageRank mixture vs.\ a finite $k$-step ensemble), as summarised in
\Cref{tab:dtne}. The geodesic-recovery guarantee (\Cref{thm:geodesic}) holds for
the MERW-based construction; no analogous result is established for the SRW-based
DTNE distance.

\section{Proof of \texorpdfstring{\Cref{thm:geodesic}}{Theorem~\ref*{thm:geodesic}}}
\label{app:proof}
We use the notation of \Cref{thm:geodesic}: graphs $\{G_n\}$ converging spectrally
(\Cref{ass:spectral-conv}) to $(\calM,g)$, vertices $i\in V_n$ identified with
points $x_i\in\calM$, and the MERW matrix
$T_n=\Psimat_n^{-1}(A_n/\lmax^{(n)})\Psimat_n$, where
$\Psimat_n=\operatorname{diag}(\psi^{(n)})$ and $\psi^{(n)}$ is the
($\ell^2$-normalised) Perron vector of $A_n$. We drop the superscript $(n)$ when
unambiguous, and write $H_n=\lmax^{(n)}I-A_n$, whose spectral convergence
(\Cref{ass:spectral-conv}) identifies the limiting operator $H=-\Delta_\calM+V$.
The proof combines three standard ingredients.

\subsection{Step 1 --- Discrete-to-continuum limit \texorpdfstring{($n\to\infty$, $t$ fixed)}{(n to infinity, t fixed)}}

\paragraph{Idea of Step 1.}
The matrix power $(A_n/\lmax^{(n)})^{k_n}$ is a discrete heat kernel: in the
eigenbasis, each mode is reweighted by $(1-\mu_\ell/\lmax^{(n)})^{k_n}$, which for
the smooth low-lying modes approaches the continuum weight $e^{-t\mu_\ell}$.
Summing the modes turns the discrete operator into the eigenfunction expansion of
the continuum heat kernel
$K^H_t(x,y)=\sum_\ell e^{-t\mu_\ell}u_\ell(x)u_\ell(y)$. Two bookkeeping points
complete the argument: discrete and continuum eigenvectors use different
normalisations, contributing a single global factor $Z_n=\Theta(n)$; and the
rough high-frequency modes, which the continuum picture discards, must not survive
in the discrete sum.

\begin{lemma}[Discrete-to-continuum limit]\label{lem:dtoc}
Fix $t>0$, $k_n=\lfloor t\,\lmax^{(n)}\rfloor$. Under \Cref{ass:spectral-conv}
there is a pair-independent factor $Z_n=\Theta(n)$ (bounded above and below by
constant multiples of $n$; it depends only on the global sampling density, not on
$i,j$) such that, for every pair $x_i\neq x_j$,
\[
  Z_n\,\bigl[(A_n/\lmax^{(n)})^{k_n}\bigr]_{ij}\;\longrightarrow\;K^H_t(x_i,x_j),
\]
the integral kernel of the heat semigroup $e^{-tH}$ ($H=-\Delta_\calM+V$) against
the sampling measure.
\end{lemma}

\begin{proof}
Expanding in the $\ell^2(V_n)$-orthonormal eigenbasis of $A_n$ (equivalently
$H_n$), with eigenvalues $\mu_\ell^{(n)}$ of $H_n$ in increasing order,
\[
  \bigl[(A_n/\lmax^{(n)})^{k_n}\bigr]_{ij}
  =\sum_\ell\Bigl(1-\tfrac{\mu_\ell^{(n)}}{\lmax^{(n)}}\Bigr)^{k_n}
   \phi^{(n)}_\ell(i)\,\phi^{(n)}_\ell(j).
\]

\emph{Low modes.} For each fixed $\ell$, $\mu_\ell^{(n)}=O(1)$ (it converges to
the $\ell$-th eigenvalue of $H$), and
\[
  k_n\log\Bigl(1-\tfrac{\mu_\ell^{(n)}}{\lmax^{(n)}}\Bigr)
   =-t\mu_\ell+O\!\bigl(t\,\mu_\ell^2/\lmax^{(n)}\bigr)\to-t\mu_\ell,
\]
so the discrete weight $(1-\mu_\ell^{(n)}/\lmax^{(n)})^{k_n}\to e^{-t\mu_\ell}$.
The interpolated eigenvectors converge in $L^2(\mu)$ to the continuum
eigenfunctions (\Cref{ass:spectral-conv}); since the discrete vectors are
$\ell^2(V_n)$-normalised ($\sum_i\phi_\ell(i)^2=1$) and the continuum ones are
$L^2(\mu)$-normalised, they differ by the pair-independent factor
$Z_n^{1/2}=\Theta(\sqrt n)$, i.e.\
$\phi^{(n)}_\ell(i)=Z_n^{-1/2}u_\ell(x_i)\,(1+o(1))$. The eigenfunctions $u_\ell$
are continuous on the compact $\calM$, so this holds pointwise at $x_i,x_j$, and
\[
  Z_n\,\phi^{(n)}_\ell(i)\phi^{(n)}_\ell(j)\to u_\ell(x_i)u_\ell(x_j).
\]

\emph{High modes.} Split the spectrum at a fixed cutoff $M$. The finitely many
modes with $\mu_\ell^{(n)}\le M$ converge term-by-term as above. For the
remainder we use two facts. (i) The kNN affinity graphs are non-bipartite and
connected---overlapping neighbourhoods create odd cycles---so the bottom of
$\operatorname{spec}(A_n)$ stays bounded away from $-\lmax^{(n)}$, uniformly in
$n$; we take this bottom gap as part of the regime. The corresponding oscillating
modes (those with $A_n$-eigenvalue near $-\lmax^{(n)}$) then satisfy
$|1-\mu_\ell^{(n)}/\lmax^{(n)}|\le 1-\delta$ and contribute
$O((1-\delta)^{k_n})\to0$. (ii) The remaining modes satisfy
$|(1-\mu_\ell^{(n)}/\lmax^{(n)})^{k_n}|\le e^{-t\mu_\ell^{(n)}}$, and with the
uniform bound on $\|u_\ell\|_\infty$ their total contribution is dominated by the
trace-class tail $\sum_{\mu_\ell>M}e^{-t\mu_\ell}$, which $\to0$ as $M\to\infty$.
Letting $n\to\infty$ then $M\to\infty$ gives
\[
  Z_n\bigl[(A_n/\lmax^{(n)})^{k_n}\bigr]_{ij}
  \to\sum_\ell e^{-t\mu_\ell}u_\ell(x_i)u_\ell(x_j)=K^H_t(x_i,x_j).\qedhere
\]
\end{proof}

\subsection{Step 2 --- Symmetrisation cancels the Perron ratio (exactly, at finite \texorpdfstring{$n$}{n})}

Since $T_n^{k}=\Psimat_n^{-1}(A_n/\lmax^{(n)})^{k}\Psimat_n$,
\[
  T_n^{k}(i,j)=\frac{\psi_j}{\psi_i}\bigl[(A_n/\lmax^{(n)})^{k}\bigr]_{ij},
  \qquad
  T_n^{k}(j,i)=\frac{\psi_i}{\psi_j}\bigl[(A_n/\lmax^{(n)})^{k}\bigr]_{ji}.
\]
As $A_n$ is symmetric, so is $(A_n/\lmax^{(n)})^{k}$; the Perron ratios cancel
with \emph{no residual term}:
\begin{equation}\label{eq:cancel}
  \sqrt{T_n^{k}(i,j)\,T_n^{k}(j,i)}
  =\sqrt{\tfrac{\psi_j}{\psi_i}\tfrac{\psi_i}{\psi_j}}\;
   \bigl[(A_n/\lmax^{(n)})^{k}\bigr]_{ij}
  =\bigl[(A_n/\lmax^{(n)})^{k}\bigr]_{ij}.
\end{equation}
This identity is \emph{exact at finite $n$}---in contrast to the continuum
approximation of Step 1---and is the key structural feature of the symmetrised
free-energy dissimilarity. Combining \eqref{eq:cancel} with \Cref{lem:dtoc},
\begin{equation}\label{eq:Dexpand}
  D^{(n)}_{ij}(t)
   = -\log\sqrt{T_n^{k}(i,j)\,T_n^{k}(j,i)}
   = \log Z_n - \log K^H_t(x_i,x_j) + o(1),
\end{equation}
where $\log Z_n$ is a pair-independent level (the log-partition / mixing offset;
cf.\ \Cref{prop:largek}), not a geometric distortion.
\begin{remark}
The exact cancellation is the purpose of the geometric mean: the one-sided
quantity $-\log T_n^{k}(i,j)$ carries the asymmetric bias
$\log(\psi_i/\psi_j)=O(1)$, which \eqref{eq:cancel} eliminates rather than merely
controlling asymptotically.
\end{remark}

\subsection{Step 3 --- Short-time (Varadhan) asymptotics (\texorpdfstring{$t\to0^+$}{t to 0+})}

\begin{lemma}[Varadhan short-time asymptotics]\label{lem:varadhan}
Let $H=-\Delta_\calM+V$ on a closed Riemannian manifold $(\calM,g)$ with
$V\in C(\calM)$, and let $K^H_t$ be its heat kernel. Then for each $x\neq y$,
\[
  \lim_{t\to0^+} -4t\,\log K^H_t(x,y) \;=\; d_\calM(x,y)^2 .
\]
\end{lemma}

\begin{proof}
For $-\Delta_\calM$ this is Varadhan's formula \citep{Varadhan1967}:
$-4t\log K^{-\Delta}_t(x,y)\to d_\calM(x,y)^2$. By the Feynman--Kac
representation,
\[
  K^H_t(x,y)
  = K^{-\Delta}_t(x,y)\,
    \mathbb{E}^{x\to y,\,t}\!\Bigl[e^{-\int_0^t V(\gamma_s)\,ds}\Bigr],
\]
the bridge pinned at $\gamma_0=x$, $\gamma_t=y$. Since $\calM$ is compact and $V$
continuous, $\|V\|_\infty<\infty$, so $\bigl|\int_0^t V(\gamma_s)\,ds\bigr|\le
t\|V\|_\infty$ for every path; the Feynman--Kac factor lies in $e^{O(t)}$,
uniformly. Hence $-\log K^H_t(x,y)=-\log K^{-\Delta}_t(x,y)+O(t)$, and multiplying
by $4t$ gives the claim.
\end{proof}

\subsection{Step 4 --- Combining along a diagonal sequence}

By Step 2, for each fixed $t>0$,
\[
  D^{(n)}_{ij}(t) = \log Z_n - \log K^H_t(x_i,x_j) + \epsilon_n(t),
  \qquad \epsilon_n(t)\xrightarrow[n\to\infty]{}0,
\]
with $Z_n=\Theta(n)$ pair-independent (\Cref{lem:dtoc}). A diagonal extraction
over $t\to0^+$ yields a sequence $t_n\to0^+$, with $n$ taken large enough that
both $\epsilon_n(t_n)\to0$ and $t_n\log Z_n=t_n\bigl(\log n+O(1)\bigr)\to0$, along
which $k_n=\lfloor t_n\lmax^{(n)}\rfloor\to\infty$ and
\[
  4t_n\,D^{(n)}_{ij}(t_n)
   = \underbrace{4t_n\log Z_n}_{\to\,0}
     \;\underbrace{-\,4t_n\log K^H_{t_n}(x_i,x_j)}_{\to\,d_\calM(x_i,x_j)^2}
     + \underbrace{4t_n\,\epsilon_n(t_n)}_{\to\,0}
   \;\longrightarrow\; d_\calM(x_i,x_j)^2,
\]
the middle term by \Cref{lem:varadhan}. This is the joint (diagonal) extraction
over $t\to0^+$ and $n\to\infty$ that yields \eqref{eq:varadhan}. \qed

\begin{remark}[Kernel-agnosticity]\label{rem:kernel}
The theorem uses the affinity kernel only through \Cref{ass:spectral-conv}: the
proof relies on symmetry and non-negativity of $A_n$ (Step~2) and on the spectral
convergence of $H_n=\lmax^{(n)}I-A_n$ to a continuum operator $H=-\Delta_\calM+V$
(Step~1), but not on the explicit form of the kernel. The argument therefore
applies to any affinity construction satisfying \Cref{ass:spectral-conv},
including the adaptive-bandwidth Gaussian used in our synthetic experiments and
the $\alpha$-decay kernel of \citet{Moon2019} used for single-cell data, provided
each meets that assumption. Since \Cref{ass:spectral-conv} fixes the principal
part to be the Laplace--Beltrami operator $-\Delta_\calM$, the kernel choice can
alter only the lower-order potential $V$ and the constants in the estimates; the
recovered geodesic, determined by the principal part, is unchanged.
\end{remark}

\end{document}